\documentclass{article}

\PassOptionsToPackage{numbers, compress}{natbib}

\usepackage[main,final]{neurips_2026}

\usepackage[utf8]{inputenc} %
\usepackage[T1]{fontenc}    %
\usepackage{hyperref}       %
\usepackage{url}            %
\usepackage{amsfonts}       %
\usepackage{nicefrac}       %
\usepackage{microtype}      %
\usepackage[table]{xcolor}         %

\usepackage{graphicx}
\usepackage{subcaption}
\usepackage{caption}
\usepackage{booktabs} %

\usepackage{amsmath}
\usepackage{amssymb}
\usepackage{mathtools}
\usepackage{amsthm}
\usepackage{wrapfig}

\usepackage{enumitem}
\usepackage{algorithm}
\usepackage{algpseudocode}
\usepackage{array}
\usepackage{multirow}
\usepackage{titletoc}
\usepackage{makecell}
\usepackage{colortbl}
\usepackage{threeparttable}
\usepackage{pifont}

\usepackage[capitalize,noabbrev]{cleveref}
\definecolor{colorAdded}{named}{black}
\graphicspath{{fig_repository/}}

\theoremstyle{plain}
\newtheorem{theorem}{Theorem}[section]
\newtheorem{proposition}[theorem]{Proposition}
\newtheorem{lemma}[theorem]{Lemma}
\newtheorem{corollary}[theorem]{Corollary}
\theoremstyle{definition}
\newtheorem{definition}[theorem]{Definition}
\newtheorem{assumption}[theorem]{Assumption}
\theoremstyle{remark}

\definecolor{YaleBlue}{rgb}{0.059,0.302,0.573}
\definecolor{forestgreen}{rgb}{0.133,0.549,0.133}
\definecolor{lightblue}{rgb}{0.796, 0.894, 0.9808}
\definecolor{crimson}{rgb}{0.863,0.078,0.235}
\definecolor{naturegreen}{RGB}{0,102,85} %
\definecolor{naturegray}{RGB}{248,248,248} %

\usepackage[disable,textsize=tiny]{todonotes}

\definecolor{citeblue}{RGB}{30,80,160}
\hypersetup{
  colorlinks=true,
  linkcolor=red,
  citecolor=citeblue,
  urlcolor=citeblue,
  pdftitle={Stochastic Interpolants via Conditional Dependent Coupling},
  pdfauthor={Chenrui Ma, Xi Xiao, Tianyang Wang, Xiao Wang, Yanning Shen},
}

\title{Stochastic Interpolants via Conditional Dependent Coupling}

\author{%
  Chenrui Ma$^{1,2}$ \quad Xi Xiao$^{3,4}$ \quad Tianyang Wang$^{3}$ \\
  \textbf{Xiao Wang}$^{4}$ \quad \textbf{Yanning Shen}$^{2}$\thanks{Corresponding author: \texttt{yannings@uci.edu}.} \\
  $^{1}$University of Virginia \quad $^{2}$University of California, Irvine \\
  $^{3}$University of Alabama at Birmingham \\
  $^{4}$Oak Ridge National Laboratory
}

\begin{document}

\maketitle

\begin{abstract}
Existing image generation models face critical challenges regarding the trade-off between computation and fidelity. Specifically, models relying on a pretrained Variational Autoencoder (VAE) suffer from information loss, limited detail, and the inability to support end-to-end training. In contrast, models operating directly in the pixel space incur prohibitive computational cost. Although cascade models can mitigate computational cost, stage-wise separation prevents effective end-to-end optimization, hampers knowledge sharing, and often results in inaccurate distribution learning within each stage. To address these challenges,
we introduce a unified multi-stage generative framework formulated under the \emph{stochastic interpolant} formalism with our \textbf{Conditional Dependent Coupling} strategy as a Flow-Matching variant. It decomposes the generative process into interpolant trajectories at multiple stages, ensuring accurate distribution learning while enabling end-to-end optimization. Importantly, the entire process is modeled as a single unified Diffusion Transformer, eliminating the need for disjoint modules and also enabling knowledge sharing.
Under explicitly stated assumptions, we provably reduce both transport cost and asymptotic inference time; we empirically validate the underlying NFE--transport-cost and per-evaluation-cost relations.
Extensive experiments demonstrate that our method achieves competitive fidelity and substantially lower wall-clock cost in the low-NFE regime across multiple resolutions.
\end{abstract}

\section{Introduction} \label{sec:intro}

Generative models have achieved remarkable progress in recent years, driving advances in diverse domains such as natural language processing~\cite{openai2024gpt4technicalreport}, computer vision~\cite{geng2025mean}, and scientific modeling~\cite{fotiadis2024stochastic}. Unlike discrete data generation (e.g., large language models for text), high-dimensional continuous data generation, such as image synthesis, suffers from high complexity and computational demands~\cite{rombach2022ldm, peebles2023dit, ma2024sit}. The primary challenge lies in balancing fidelity, efficiency, and scalability when learning to approximate intricate data distributions~\cite{sun2024llamagen, tian2024var}.

There are two main paradigms for generative modeling: pixel space generation and latent space generation. Pixel space generation operates directly in the original data domain, preserving fine-grained details without compression~\cite{bao2023uvit, hoogeboom2025sid2}. However, this approach suffers from extreme inefficiency due to the high dimensionality of images, resulting in expensive training and inference~\cite{rombach2022ldm, kingma2023vdmpp}. In contrast, latent space generation leverages compact representations, typically through Variational Autoencoders (VAEs), to reduce dimensionality and improve efficiency~\cite{li2024autoregressive, esser2021taming}. While effective, these methods inevitably incur information loss during encoding and decoding, and often cannot be trained in a fully end-to-end manner with the generative backbone~\cite{jin2024pyramidal, jiao2025flexvar}.

To mitigate these limitations, multi-stage generation methods have been proposed. By decomposing the synthesis process into a sequence of stages, they allow early stages to operate in lower-dimensional spaces and later stages to refine the outputs at higher resolutions~\cite{tian2024var, chen2025pixelflow}. This hierarchical design improves efficiency and progressively captures data complexity~\cite{li2025fractalmar, ren2024flowar}. Nonetheless, existing multi-stage frameworks rely on disentangled stage-specific models, preventing unified parameterization and end-to-end optimization~\cite{ho2022cdm, kim2024pagoda}. Furthermore, the decoupled design may lead to inaccurate distribution modeling, with errors compounding across stages~\cite{chen2025pixelflow, jin2024pyramidal}.

In this work, we propose a novel multi-stage generation framework formulated under the \emph{stochastic interpolant}~\cite{albergo2023stochastic, albergo2022building} formalism with \textbf{conditional dependent coupling}. The framework is compatible with both deterministic (Flow/ODE) and stochastic (Diffusion/SDE) instantiations of the interpolant; in this paper we instantiate and empirically validate the deterministic Flow-Matching variant, hence we refer to the resulting model as CDC-FM, and leave the SDE instantiation to future work. Our approach generates high-resolution images in a coarse-to-fine manner, essentially addressing the notorious limitations of the multi-stage generation paradigm. Specifically, the proposed method enables efficient and high-performance generation directly in pixel space, thus retaining detailed information without the bottleneck of latent compression. Moreover, by proposing \textbf{conditional dependent coupling}, we unify the multi-stage generation process into one coherent framework, ensuring accurate distribution learning at each stage.
Notably, the entire multi-stage framework can be parameterized by a single DiT~\cite{peebles2023scalable}, 
facilitating knowledge sharing across stages and enabling full end-to-end optimization. 
We further provide formal proof that the proposed framework significantly reduces the transport cost and inference time. 
See the Related Works in the Appendix. The project repository is \url{https://github.com/merry7cherry/Stochastic-Interpolants-via-Conditional-Dependent-Coupling}; the implementation is not yet released.

\section{Preliminaries} \label{sec:preliminary}
The \emph{stochastic interpolant}~\cite{albergo2023stochastic, albergo2022building} unifies the theory of Ordinary Differential Equations (ODEs) and Stochastic Differential Equations (SDEs).
Our method, which was developed based on this theory, is therefore comparable to both the Flow/ODEs and Diffusion/SDEs generation.
For notation simplicity and efficient information delivery, in the rest of the paper, we focus on the Flow/ODEs. However, the complete \emph{stochastic interpolant} theory that introduces an additional noise term and recovers Diffusions/SDEs is presented in Appendix~\S\ref{appendix:SI}, where we show that the deterministic interpolant definitions below can be generalized readily.

\begin{definition}[Deterministic interpolant for Flow Matching] \label{def:Deterministic_interpolant_and_Flow_Matching}
Given two probability densities $\rho_0,\rho_1:\mathbb{R}^d\!\to\!\mathbb{R}_{\ge 0}$ and a coupling $\rho(x_0,x_1)$ with marginals $\rho_0$ and $\rho_1$, we define the deterministic interpolant~\cite{lipman2024flow, lipman2022flow}:
\begin{equation}
\label{eq:flow_interpolant}
I_t \;=\; \alpha_t x_0 + \beta_t x_1,\qquad t\in[0,1],
\end{equation}
where $\alpha_t,\beta_t$ are differentiable in $t$ and satisfy the boundary conditions:
\[
\alpha_0=\beta_1=1,\qquad \alpha_1=\beta_0=0,\qquad \alpha_t^2+\beta_t^2>0\ \ \forall t\in[0,1].
\]
This yields a time–dependent density $\rho_t$ for $I_t$ with $\rho_{t=0}=\rho_0$ and $\rho_{t=1}=\rho_1$.
\end{definition}

\begin{theorem}[Transport continuity equation for Flow model]
Let $b_t:\mathbb{R}^d\to\mathbb{R}^d$ denote the \emph{conditional velocity field}~\cite{lipman2024flow, lipman2022flow}:
\begin{equation}
\label{eq:b}
b_t(x) \;=\; \mathbb{E}\!\left[\,\dot I_t \,\middle|\, I_t=x\,\right],
\end{equation}
where $\dot f= \frac{\mathrm{d}f}{\mathrm{d}t}$ and the expectation is over $(x_0,x_1)\!\sim\!\rho(x_0,x_1)$. The interpolant density $\rho_t$ solves the transport continuity equation~\cite{villani2008optimal, villani2021topics}:
\begin{equation}
\label{eq:transportEquation}
\partial_t \rho_t(x) \;+\; \nabla\!\cdot\!\big(b_t(x)\,\rho_t(x)\big) \;=\; 0.
\end{equation}
In practice, we learn a model velocity $b_t$ to approximate $b_t$ by minimizing:
\begin{equation}
\label{eq:b_objective}
L_b(\hat b)\;=\;\int_0^1 \mathbb{E}\!\left[\ \big\lvert \hat b_t(I_t)\big\rvert^2 \;-\; 2\,\dot I_t\!\cdot\!\hat b_t(I_t)\ \right]\mathrm{d}t,
\end{equation}
which is estimable from samples $(x_0,x_1)\!\sim\!\rho(x_0,x_1)$.
\end{theorem}

\begin{corollary}[Probability flow ODE]
The transport continuity equation implies that the solution $X_t$ to the probability flow ODE~\cite{song2023consistency, ma2024sit}:
\begin{equation}
\label{eq:pf_ode}
\dot X_t \;=\; b_t\!\left(X_t\right)
\end{equation}
matches the interpolant law: if $X_{t=0}\!\sim\!\rho_0$, then $X_{t=1}\!\sim\!\rho_1$. Hence generative sampling is obtained by drawing $x_0\!\sim\!\rho_0$ and integrating \eqref{eq:pf_ode} from $t{=}0$ to $t{=}1$.
\end{corollary}

\paragraph{Remark (Generalization to \emph{stochastic interpolant}).}
The Appendix~\S\ref{appendix:SI} introduces the full \emph{stochastic interpolant}
$I_t=\alpha_t x_0+\beta_t x_1+\gamma_t z$ with $z\!\sim\!\mathcal N(0,I_d)$, \cite{albergo2023stochastic, albergo2022building}
which recovers diffusion-type models, the score field $s_t(x)=\nabla \log \rho_t(x)$, and the companion objective for $s_t$.
All transport-equation statements above remain valid, with \eqref{eq:b}–\eqref{eq:pf_ode} appearing as the $\gamma_t\equiv 0$ case used throughout the main text.

\begin{definition}[Transport cost] \label{definition:Transport cost}
Let $X_t(x_0)$ be the solution to the probability flow ODE \eqref{eq:pf_ode} for the initial condition $X_{t=0}(x_0) = x_0 \sim \rho_0$. Then the following inequality holds:
\begin{equation}
\label{eq:transportation}
\begin{aligned}
    \mathbb{E}_{x_0 \sim \rho_0} \left[ |X_{t=1}(x_0) - x_0|^2 \right] \leq \int_0^1 \mathbb{E}[|\dot{I}_t|^2] dt < \infty.
\end{aligned}
\end{equation}
Minimizing the left-hand side of \eqref{eq:transportation} would achieve the optimal transport in the sense of Benamou-Brenier~\cite{benamou2000computational}, and the minimum would give the Wasserstein-2 distance between $\rho_0$ and $\rho_1$~\cite{albergo2023stochasticcouplings}.
\end{definition}

While for a common \emph{stochastic interpolant}, this transport cost construction was made using the choice $\rho(x_0, x_1) = \rho_0(x_0) \rho_1(x_1)$, so that $x_0$ and $x_1$ were drawn independently from the base and the target, \textbf{data-dependent coupling} constructs the joint distribution by $\rho(x_0, x_1) = \rho_0(x_0 \mid x_1) \rho_1(x_1)$ to reduce the transport cost~\cite{albergo2023stochasticcouplings}, such that:
\begin{equation}
\label{eq:datacoupling_reduce_transportation}
\begin{split}
\int_{\mathbb{R}^{3d}} |\dot{I}_t|^2 & \rho(x_0, x_1) \rho_z(z) \, dx_0 dx_1 dz \
\leq 
\int_{\mathbb{R}^{3d}} |\dot{I}_t|^2 \rho_0(x_0) \rho_1(x_1) \rho_z(z) \, dx_0 dx_1 dz
\end{split}
\end{equation}
The bound on the transportation cost in \eqref{eq:transportation} is more tightly controlled by the construction of data-dependent couplings~\cite{albergo2023stochasticcouplings}.

\section{Methodology}

In this section, we introduce a unified multi-stage generative algorithm designed to efficiently produce high-resolution images without relying on a pretrained VAE, thereby mitigating information and detail loss while enabling end-to-end training.
The algorithm proceeds through a sequence of stages, where the resolution is progressively refined in a coarse-to-fine manner.
At each stage, the output of the previous stage is further enhanced toward a finer resolution target using \emph{stochastic interpolant} with \textbf{conditionally dependent coupling}, thus forming a cascaded architecture and guaranteeing accurate data distribution modeling at each stage. Crucially, all stages are parameterized by a single unified DiT~\cite{peebles2023dit}, which facilitates knowledge sharing and supports end-to-end optimization.
Under the \emph{stochastic interpolant} framework, the proposed method is naturally compatible with Flow models (ODEs) as well as Diffusion models (SDEs).
For clarity, we focus here on deterministic interpolants (Flow/ODEs), while noting that the method extends readily to stochastic interpolants (includes Diffusion/SDEs), see \hyperref[sec:preliminary]{Preliminary} for details.

\subsection{Notations} \label{sec:notation}

\textbf{Data and Resolution Hierarchy.}
Let the target high-resolution image distribution be defined on $\mathbb{R}^{d_K}$. We define a hierarchy of resolutions, where $k \in \{1, \dots, K\}$ represents the generation stage:
\ding{182} $x^{(K)} \sim p_{\text{data}}$: A sample from the true high-resolution data distribution.
\ding{183} $x_1^{(k)} \in \mathbb{R}^{d_k}$: A ground-truth image sample at stage $k$, which is generated by downsampling the original high-resolution image.

\textbf{Upsampling and Downsampling Operators.}
We define operators for changing image resolutions. Let the scaling factor between stage $k-1$ and $k$ be a power of 2: $d_k = 2^{k-1} d_1$, for simplicity.
\begin{itemize}[left=1em]
    \item \textit{Downsampling Operator $D_k$}: This operator maps an image from resolution $d_k$ to $d_{k-1}$ ($D_k: \mathbb{R}^{d_k} \to \mathbb{R}^{d_{k-1}}$). It operates via neighborhood averaging, where a block of pixels in the higher-resolution image is averaged to a single pixel in the lower-resolution image.
    \item \textit{Upsampling Operator $U_k$}: This operator maps an image from resolution $d_{k-1}$ to $d_k$ ($U_k: \mathbb{R}^{d_{k-1}} \to \mathbb{R}^{d_k}$). It operates via neighborhood replication, where each pixel from the lower-resolution image is copied to form a block in the higher-resolution image.
\end{itemize}
These operators are defined such that they are transitive: $D_{k+1 \to k-1} = D_k \circ D_{k+1}$. We can define a composite downsampling operator $D_{K \to k}$ that maps from the highest resolution $d_K$ directly to $d_k$. The ground-truth sample for any stage $k$ is thus defined as $x_1^{(k)} = D_{K \to k}(x^{(K)}) \sim \rho_1^{(k)}(x_1^{(k)} | x^{(K)})$.

\subsection{Multi-Stage Flow Matching Model}\label{sec:multistage}

To implement the unified multi-stage generation process, we define a single unified generative model(Flow/ODEs or Diffusion/SDEs) to transform from the source distribution to the target at each stage $k$ with \textbf{conditional dependent coupling}. 
The $K$ stages are designed as below:
\ding{182} \textbf{Stage $k = 1$: Noise-to-Image Generation.} The first stage generates a low-resolution image from a simple source distribution (e.g., Gaussian noise).
\ding{183} \textbf{Stage $k > 1$: Image-to-Image Refinement.}
In subsequent stages, the resolution is progressively refined by leveraging \textbf{conditionally dependent coupling}, where the probability path is explicitly conditioned on the relationship between the low-resolution source and the high-resolution target images. This cascaded approach breaks down the complex task of high-resolution image generation into a series of more manageable sub-problems. We will elaborate on the \textbf{motivation} and \textbf{benefits} of the design at the end of this section.

\paragraph{Stage $k=1$: Noise-to-Image Generation}
The first stage ($k=1$) of our model is responsible for generating the lowest-resolution image from a simple prior distribution. The model is defined by its source and target distributions and the \textbf{conditional dependent coupling} between them. 
For this foundational stage, the joint distribution is the product of the marginals: 
\begin{equation}
\label{eq:indepednet_coupling}
\begin{aligned}
\rho^{(1)}(x_0^{(1)}, x_1^{(1)} | x^{(K)})
&= \rho_0^{(1)}(x_0^{(1)} | x_1^{(1)}) \rho_1^{(1)}(x_1^{(1)} | x^{(K)}) \\
&= \rho_0^{(1)}(x_0^{(1)}) \rho_1^{(1)}(x_1^{(1)} | x^{(K)})
\end{aligned}
\end{equation}
The \textbf{target distribution} is the empirical distribution of the lowest-resolution ground-truth images, where $x_1^{(1)} \sim \rho_1^{(1)}(x_1^{(1)} | x^{(K)})$.
The \textbf{source distribution} is a standard Gaussian distribution denoted as $x_0^{(1)} \sim \rho_0^{(1)}(x_0^{(1)}) = \mathcal{N}(0, \sigma^2 I_{d_1})$, which is independent to $\rho_1^{(1)}(x_1^{(1)} | x^{(K)})$, where can also drive that $\rho_0^{(1)}(x_0^{(1)} | x_1^{(1)}) = \rho_0^{(1)}(x_0^{(1)})$ from \eqref{eq:indepednet_coupling}.

\paragraph{Stage $k > 1$: Image-to-Image Refinement}
For all subsequent stages ($k > 1$), the model learns to perform a coarse-to-fine resolution enhancement. The core of these stages is the \textbf{conditional dependent coupling} that leverages the structural similarity between the lower-resolution source image and the higher-resolution target image.
For each stage $k$, the \textbf{target distribution} is the distribution of ground-truth images $x_1^{(k)} \sim \rho_1^{(k)}(x_1^{(k)} | x^{(K)})$ at resolution $d_k$. Instead of drawing from a simple prior, the \textbf{source sample} $x_0^{(k)}$ is constructed directly from the corresponding target sample $x_1^{(k)}$: $x_0^{(k)} \sim \rho_0^{(k)}(x_0^{(k)} | x_1^{(k)})$. This is achieved by first applying a deterministic mapping $m_k(x_1^{(k)}) = U_k(D_k(x_1^{(k)}))$ and then adding a small amount of Gaussian noise to ensure the conditional probability is well-defined:
\begin{equation}
\label{eq:data_dependent_coupling_compute}
x_0^{(k)} = m_k(x_1^{(k)}) + \sigma_k \zeta_k = U_k(D_k(x_1^{(k)})) + \sigma_k \zeta_k
\end{equation}
where $\zeta_k \sim \mathcal{N}(0, I_{d_k})$ and $\sigma_k$ is a small, stage-dependent noise level, which can be tuned. This construction induces the conditional probability $\rho_0^{(k)}(x_0^{(k)} | x_1^{(k)})$ {\color{colorAdded}as so called data-dependent coupling~\cite{albergo2023stochasticcouplings}.} Thus, the joint probability density $\rho^{(k)}(x_0^{(k)}, x_1^{(k)} | x^{(K)})$ is shown as:
\begin{equation}
\label{eq:data_dependent_coupling}
\begin{aligned}
\rho_0^{(k)}(x_0^{(k)} | x_1^{(k)}) &= \mathcal{N}(x_0^{(k)}; m_k(x_1^{(k)}), \sigma_k^2 I_{d_k})
\end{aligned}
\end{equation}
\begin{equation}
\label{eq:data_dependent_coupling_joint_distribution}
\begin{aligned}
\rho^{(k)}(x_0^{(k)}, x_1^{(k)} | x^{(K)}) &= \rho_0^{(k)}(x_0^{(k)} | x_1^{(k)}) \rho_1^{(k)}(x_1^{(k)} | x^{(K)})
\end{aligned}
\end{equation}

\paragraph{Model Parameterization}
We design a unified DiT model that operates consistently across multiple resolutions by sharing a single set of parameters. To enable the model to recognize and differentiate between these resolutions, we introduce an additional resolution embedding. Specifically, we treat the absolute resolution of the feature map $r_k$ obtained after patch embedding~\cite{dosovitskiy2020image} as a conditional signal. This signal is then encoded using sinusoidal positional embeddings~\cite{vaswani2017attention} $e_k = E(r_k)$ and fused with the timestep embedding by the cross-attention mechanism~\cite{vaswani2017attention} before being fed into the model.
{\color{colorAdded} See Appendix~\S\ref{appendix:resolution_embedding} for details about resolution embedding.}

To unify the multi-stage generation, the rescaled timestep strategy~\cite{chen2025pixelflow, jin2024pyramidal, yan2024perflow} is employed to align the multi-stage generation time definition with the standard Flow model~\cite{lipman2022flow}, which generates across the time scope $t \in [0, 1]$.
For each stage $k$, the start time point and end time point are defined as $t^k_0$ and $t^k_1$, where $0 \leq t^k_0<t^k_1 \leq 1$, $t^k_0 = t^{k-1}_1$, and meet boundary condition $t^1_0=0$, $t^K_1=1$. For simplicity, we construct a linear interpolant by specifying $\alpha_t = 1-\tau$ and $\beta_t = \tau$ in Definition~\ref{def:Deterministic_interpolant_and_Flow_Matching}:
\begin{equation}
I^{(k)}_\tau \;=\; (1-\tau) \cdot x^{(k)}_0 + \tau \cdot x^{(k)}_1,
\end{equation}
where $\tau = \frac{t - t^k_0}{t^k_1 - t^k_0}$, so that timestep $t \in [0, 1]$ of the entire multi-stage generation process located in each stage $k$: $t \in [t^k_0, t^k_1]$ is rescaled to $\tau \in [0, 1]$ within each stage $k$.

The single unified DiT is parameterized by $b^{\theta}$ to model all generation stages. For the deterministic interpolant (Flow/ODEs), derived from \eqref{eq:b_objective}, the DiT is optimized by:
\begin{equation}
\label{eq:loss_function}
\begin{split}
\mathcal{L}(\theta) = \mathbb{E}_{\substack{t \sim [0,1], k, e_k, \\ (x_0^{(k)}, x_1^{(k)}) \sim \rho^{(k)}(\cdot | x^{(K)})}} \Big[ & \left| b^{\theta}(I_\tau^{(k)}, \tau, e_k) \right|^2 - 2\dot{I}_\tau^{(k)} \cdot b^{\theta}(I_\tau^{(k)}, \tau, e_k) \Big]
\end{split}
\end{equation}
where $k = \operatorname*{arg}_{\,k}\{\, t \in [t^k_0, t^k_1] \}$, and $\dot{I}_\tau^{(k)} = x_1^{(k)} - x_0^{(k)}$.

\subsection{Conditional Dependent Coupling%
} \label{sec:conditionalDepCoupling}

For the multi-stage generation model proposed in Sec.~\ref{sec:multistage}, we employed the \textbf{conditional dependent coupling} strategy to unify them. 
Instead of mapping unstructured noise to an image, our approach maps a structured prior---the upsampled low-resolution image from the previous stage with added stage-dependent noise---to the target high-resolution distribution. 
The joint probability distribution of the entire multi-stage interpolant with \textbf{conditional dependent coupling} is expressed as:
\begin{align}
\label{eq:conditional_data-dependent_coupling_joint_distribution_training}
\rho(x_1^{(K)}, x_0^{(K)}, x_1^{(K-1)}, \dots, x_0^{(2)}, x_1^{(1)}, x_0^{(1)})
& = \rho(x^{(K)}) \prod_{k=1}^K \rho^{(k)}(x_0^{(k)}, x_1^{(k)} | x^{(K)}) \\
& = \rho(x^{(K)}) \prod_{k=1}^K \rho_0^{(k)}(x_0^{(k)} | x_1^{(k)}) \rho_1^{(k)}(x_1^{(k)} | x^{(K)}),
\end{align}
This formulation reveals a key property of our model. Conditioned on the final high-resolution image $x^{(K)}$, the data coupling at any given stage $k$, denoted as $\rho^{(k)}(x_0^{(k)}, x_1^{(k)} | x^{(K)})$, is independent of that at any other stage $k^* \neq k$:
\begin{equation}
\begin{aligned}
\rho^{(k)}(x_0^{(k)}, x_1^{(k)} | x^{(K)}) \perp\!\!\!\perp \rho^{(k^*)}(x_0^{(k^*)}, x_1^{(k^*)} | x^{(K)}),
\end{aligned}
\end{equation}

\paragraph{Sequential dependency as Markov chain}
This conditional independence establishes a foundational \textbf{Markov chain}~\cite{norris1998markov} structure across the stages. Consequently, the generative process at inference time proceeds sequentially as follows:
\begin{equation}
\label{eq:inference_markov}
\begin{split}
\rho(\hat{x}_0^{(1)}, \hat{x}_1^{(1)}, &\dots, \hat{x}_0^{(K)}, \hat{x}_1^{(K)}) 
= \rho^{(1)}(\hat{x}_0^{(1)}, \hat{x}_1^{(1)}) \prod_{k=2}^K \rho^{(k)}(\hat{x}_0^{(k)}, \hat{x}_1^{(k)} \mid \hat{x}_1^{(k-1)}) \\
&= \rho_0^{(1)}(\hat{x}_0^{(1)}) \rho_1^{(1)}(\hat{x}_1^{(1)} \mid \hat{x}_0^{(1)}) \times \prod_{k=2}^K \rho_1^{(k)}(\hat{x}_1^{(k)} \mid \hat{x}_0^{(k)}) \rho_0^{(k)}(\hat{x}_0^{(k)} \mid \hat{x}_1^{(k-1)}),
\end{split}
\end{equation}
Here, the initial prior $\rho_0^{(1)}(\hat{x}_0^{(1)})$ is a standard normal distribution, $\mathcal{N}(\hat{x}_0^{(1)}; 0, I_{d_1})$. The conditional distributions $\rho_1^{(k)}(\hat{x}_1^{(k)} \mid \hat{x}_0^{(k)})$ for all stages $k$ are parameterized by a single unified generative model $b^\theta$. For the refinement stages ($k \geq 2$), the conditional prior $\rho_0^{(k)}(\hat{x}_0^{(k)} \mid \hat{x}_1^{(k-1)})$ is defined by the deterministic upsampling step with added Gaussian noise:
\begin{equation}
\label{eq:inf_prior}
\begin{aligned}
\rho_0^{(k)}(\hat{x}_0^{(k)} \mid \hat{x}_1^{(k-1)}) = \mathcal{N}(\hat{x}_0^{(k)}; U_k(\hat{x}_1^{(k-1)}), \sigma_k^2 I_{d_k}),
\end{aligned}
\end{equation}
which is implemented by sampling $\hat{x}_0^{(k)} = U_k(\hat{x}_1^{(k-1)}) + \sigma_k \zeta_k$, akin with the distribution formed by \eqref{eq:data_dependent_coupling_compute}.
{\color{colorAdded} We emphasize that \eqref{eq:inference_markov} should be understood as a \emph{modeling construction} of the reverse generative process motivated by the Flow-Matching reverse-time interpretation, rather than a direct probabilistic consequence of the training joint in \eqref{eq:conditional_data-dependent_coupling_joint_distribution_training}; see Appendix~\S\ref{appendix:cond-dep-markov} for the explicit derivation linking the two, and Appendix~\S\ref{appendix:viz} for the Markov Chain visualization.}
In contrast to previous autoregressive modeling methods that severely limited by the computational complexity arising from the full-resolution long-history condition~\cite{tian2024var, renflowar}, and potential complex token/feature pyramid construction~\cite{jiao2025flexvar}, the proposed generative process begins with a sample from the simple prior: $\hat{x}_0^{(1)} \sim \rho_0^{(1)}(\hat{x}_0^{(1)})$, subsequently, the output of each stage $\hat{x}_1^{(k-1)}$ serves as the basis for the input to the next stage $\hat{x}_0^{(k)}$, forming the sequential dependency characteristic of a Markov chain, as shown in \eqref{eq:inference_markov}.

\paragraph{Reducing transport cost}
\textbf{Conditional dependent coupling} strategy integrating each stage and thus decreasing transport cost $L=\int_0^1 \mathbb{E}[|\dot{I}_t|^2] dt < \infty$ as defined in Definition~\ref{definition:Transport cost}, see in Theorem~\ref{theorem:lowerTransportCost}.
\begin{theorem} \label{theorem:lowerTransportCost}
\textit{(Under Assumption on the noise-decay schedule, formalized in \eqref{eq:sigma_pyramid} and Appendix~\S\ref{pf:transportcost} (Assumption G.3).)}
The transport cost of the naive single-stage model,
which generates a high-resolution image $\hat{x}_1 \sim x^{(K)} \in \mathbb{R}^{d_K}$ directly from Gaussian noise $x_0 \sim \mathcal{N}(0, \sigma^2 I_{d_K})$ (denoted as $L_A$), is greater than that of the proposed multi-stage model with conditional dependent coupling strategy, which generates a high-resolution image $\hat{x}_1^{(K)} \sim x^{(K)} \in \mathbb{R}^{d_K}$ from Gaussian noise $x_0^{(1)} \sim \mathcal{N}(0, \sigma^2 I_{d_1})$ through $K$ stages (denoted as $L_B$). Specifically, we have $L_A > L_B$.
\begin{equation}
    L_A = \mathbb{E}[|x_1|^2] + \mathbb{E}[|x_0|^2] \; > \; L_B = \sum_{k=1}^{K} L_k,
\end{equation}
where $L_k$ denotes the transport cost at the $k$-th stage. See proof in Appendix~\S\ref{pf:transportcost}.
\end{theorem}

\paragraph{Accurate distribution learning}
In contrast to existing unified multi-stage generation methods~\cite{chen2025pixelflow, jin2024pyramidal}, the proposed method guarantees that the model $b^\theta$ learns the accurate data distribution $\rho_1^{(k)}(x_1^{(k)} \mid x^{(K)})$ at each stage. This, in turn, enables the unified generative model to capture the target high-resolution image distribution $x^{(K)} \sim p_{\text{data}}$ by progressively matching $\rho_1^{(k)}(x_1^{(k)} \mid x^{(K)})$ at every stage $k$. As an example, consider conditional generation: $b^{\theta}(I_\tau^{(k)}, \tau, e_k)$ models stage-$k$ generation by solving the ODE $\dot X^{k}_\tau = b^{\theta}(I_\tau^{(k)}, \tau, e_k)$, where $b^{\theta}(I_\tau^{(k)}, \tau, e_k) = \mathbb{E}_\textbf{c} \left[ b^{\theta}(I_\tau^{(k)}, \tau, e_k, \textbf{c}) \right]$ and $\textbf{c}$ denotes the conditioning signal (e.g., class label, text prompt). Conditional generation toward conditions $\textbf{c}_1$ and $\textbf{c}_2$ within stage $k$ is given by
\begin{equation}
\label{eq:conditional_generation_int}
\begin{split}
I_{1,\textbf{c}_1}^{(k)} &= I_\tau^{(k)} + \int_{\tau}^{1} b^{\theta}(I_t^{(k)}, t, e_k, \textbf{c}_1) \, \mathrm{d}t, \\
I_{1,\textbf{c}_2}^{(k)} &= I_\tau^{(k)} + \int_{\tau}^{1} b^{\theta}(I_t^{(k)}, t, e_k, \textbf{c}_2) \, \mathrm{d}t.
\end{split}
\end{equation}
where $I_\tau^{(k)}$ denotes the current sample at the rescaled timestep $\tau \in [0,1]$ within stage $k$, and specifically $I_0^{(k)}$ denotes the starting point of stage $k$.

\begin{figure*}[t]
    \centering
    \includegraphics[width=0.5\textwidth]{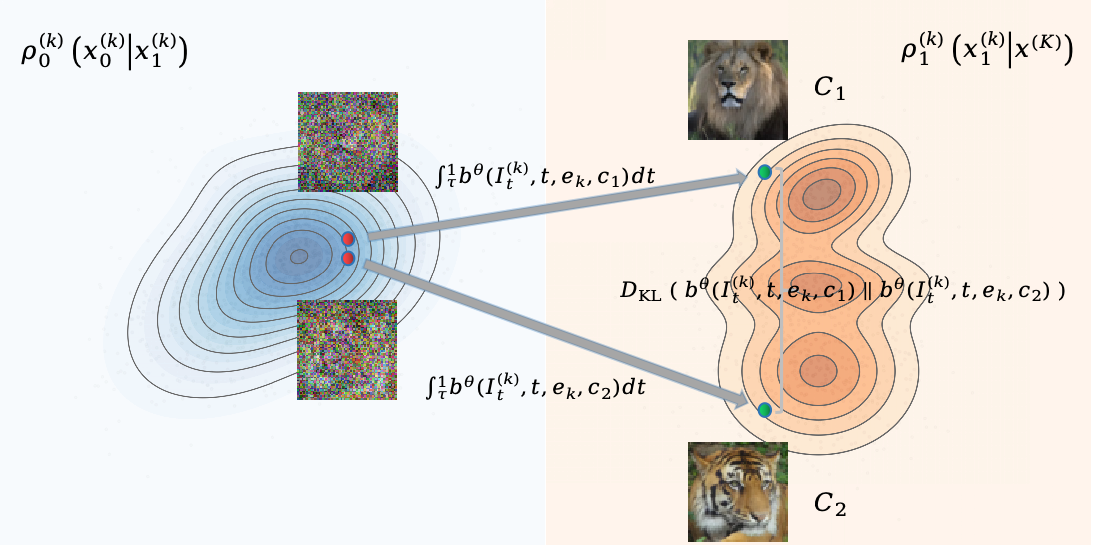}%
    \hfill
    \includegraphics[width=0.5\textwidth]{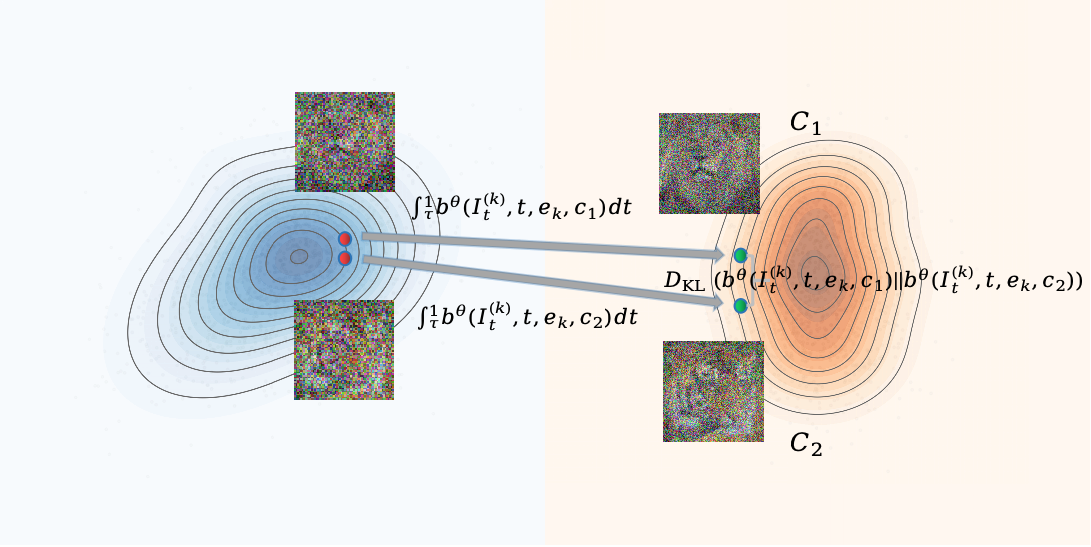}

    \caption{Data Distribution Learning and Transition.}
    \label{fig:distribution}

\end{figure*}

In the initial and intermediate stages, $I_\tau^{(k)}$ approximately follows a standard Gaussian—structureless and semantically vague, due to low resolution and high noise levels, and is therefore not characterizable. Consequently, we assume that the current sample $I_\tau^{(k)}$ used to generate $I_{1,\textbf{c}_1}^{(k)}$ and $I_{1,\textbf{c}_2}^{(k)}$ is identical within an initial or intermediate stage $k$; see the red point in Figure~\ref{fig:distribution}.

For previous multi-stage methods~\cite{chen2025pixelflow, jin2024pyramidal} that maintain noise in the target across stage $k<K$, the initial and intermediate stages necessarily target structureless, semantically vague data (close to standard Gaussian noise) due to low resolution and high noise levels. As a result, the model $b^\theta$ cannot learn accurate condition-specific distributions for $\textbf{c}_1$ and $\textbf{c}_2$, as indicated by a small divergence $D_{\mathrm{KL}}\Bigl( b^{\theta}(I_t^{(k)}, t, e_k, \textbf{c}_1)\big\| b^{\theta}(I_t^{(k)}, t, e_k, \textbf{c}_2) \Bigr)$; see the right part of Figure~\ref{fig:distribution}.

However, under the proposed training strategy with \textbf{conditional dependent coupling}, which treats the accurate data distribution $\rho_1^{(k)}(x_1^{(k)} \mid x^{(K)}) = \mathbb{E}_\textbf{c} \left[ \rho_1^{(k)}(x_1^{(k)} \mid x^{(K)}, \textbf{c}) \right]$ as the target at each stage—where $\rho_1^{(k)}(x_1^{(k)} \mid x^{(K)}, \textbf{c})$ denotes the condition-specific distribution of condition $\textbf{c}$—the multi-stage generator $b^\theta$ is guaranteed to learn the information and details pertinent to stage $k$ by transforming the source distribution $\rho_0^{(k)}(x_0^{(k)} \mid x_1^{(k)})$ into the accurate target distribution $\rho_1^{(k)}(x_1^{(k)} \mid x^{(K)})$ at stage $k$ for every condition $\textbf{c}$. This is reflected by an adequate divergence $D_{\mathrm{KL}}\Bigl( b^{\theta}(I_t^{(k)}, t, e_k, \textbf{c}_1)\big\| b^{\theta}(I_t^{(k)}, t, e_k, \textbf{c}_2) \Bigr)$; see the left part of Figure~\ref{fig:distribution}.

\paragraph{Classifier-Free Guidance for Conditional Dependent Coupling} \label{sec:cfg}

Thanks to the accurate data distribution modeling at each stage $k$, Classifier-Free Guidance (CFG)~\cite{ho2022classifier} can be seamlessly integrated with \textbf{conditional dependent coupling}. This eliminates the need for a stage-wise CFG schedule~\cite{chen2025pixelflow,  kynkaanniemi2024applying}, which otherwise requires carefully assigning the CFG strength at every stage, thereby ensuring both elegance and ease of implementation. 
In the proposed method, a single CFG strength $S_\text{cfg}$ is applied uniformly across all stages $k$. Further details are provided in Appendix~\S\ref{ap:cfg}.

\paragraph{Algorithm Workflows}

Here we present the detailed unconditional training and inference procedures for the unified multi-stage generative model $b^\theta$ with \textbf{conditional dependent coupling} in Appendix~\S\ref{ap:alg_uncondition}. The inference steps are described using the forward Euler method~\cite{lipman2024flow, lipman2022flow} for solving the ODE, as a simple example of a numerical solver. 
Moreover, the detailed training and inference procedures for modeling $b^\theta$ in conditional generation with CFG are provided in Appendix~\S\ref{ap:cfg}.

\section{Experiment}
To comprehensively evaluate the effectiveness of the proposed unified multi-stage generation model with \textbf{conditional dependent coupling} (CDC-FM), we conduct extensive experiments. 
\textbf{Class-conditional image generation.}
In Table~\ref{tab:classcondition_imagenet_256}, we compare CDC-FM with various existing methods, including GANs, Autoregressive(AR) models, Diffusion/Flow Matching models, and Multi-stage models in the ImageNet-1k/256 benchmark~\cite{krizhevsky2012imagenet}, which represent current state-of-the-art generative performance.
Higher-resolution results on the ImageNet-1k/512 benchmark~\cite{krizhevsky2012imagenet} are deferred to Appendix~\S\ref{appendix:imagenet512} (Table~\ref{tab:classcondition_imagenet_512}); they show the robustness of our method at higher resolution.
We visualize the class-conditional image generation result of our method in Figure~\ref{fig:imagenet256}, demonstrating high-fidelity, accuracy, and diversity across all the classes.

\begin{table*}[htbp]
\centering
\caption{Comprehensive comparison of generative models on ImageNet-1k/256.  Missing values are denoted ``-''. In the NFE field, “×2" indicates that CFG incurs an NFE of 2 per sampling step.}
\label{tab:classcondition_imagenet_256}
\footnotesize
\setlength{\tabcolsep}{2pt}
\begin{tabular}{@{} c l r r r r r c c @{}}
\hline
 & Method & \texttt{Params} & \texttt{FID} $\downarrow$ & \texttt{IS} $\uparrow$ & \texttt{Prec.} $\uparrow$ & \texttt{Rec.} $\uparrow$ & \texttt{NFE} & \texttt{Speed} \\
\hline
\multirow{3}{*}{\rotatebox{90}{{GAN}}}
& BigGAN~\cite{brock2018biggan} & 112M & 6.95 & 224.5 & 0.89 & 0.38 & 1 & 1.46 \\
& GigaGAN~\cite{kang2023gigagan} & 569M & 3.45 & 225.5 & 0.84 & 0.61 & 1 & 1.32 \\
& StyleGAN~\cite{karras2018stylegan} & 166M & 2.30 & 265.1 & 0.78 & 0.53 & 1 & 0.96 \\
\hline
\multirow{14}{*}{\rotatebox{90}{{Diffusion / Flow}}}
& ADM~\cite{dhariwal2021adm} & 554M & 10.94 & 101.0 & 0.69 & 0.63 & $250 \times 2$ & 9.46 \\
& U\text{-}ViT~\cite{bao2023uvit} & 287M & 3.40 & 219.9 & 0.83 & 0.52 & - & - \\
& LDM\text{-}4\text{-}G~\cite{rombach2022ldm} & 400M & 3.60 & 247.7 & - & - & $250 \times 2$ & 4.18 \\
& DiT\text{-}XL/2~\cite{peebles2023dit} & 675M & 2.27 & 278.2 & 0.83 & 0.57 & $250 \times 2$ & 3.66 \\
& SiT\text{-}XL/2~\cite{ma2024sit} & 675M & 2.06 & 270.3 & 0.82 & 0.59 & $250 \times 2$ & 3.53 \\
& VDM++~\cite{kingma2023vdmpp} & 2.0B & 2.12 & 267.7 & - & - & - & - \\
& SiD~\cite{hoogeboom2023simple} & 2.0B & 2.77 & 211.8 & - & - & $512 \times 2$ & 10.12 \\
& SiD2~\cite{hoogeboom2025sid2} & 2.0B & 1.72 & - & - & - & - & - \\
& JetFormer~\cite{tschannen2024jetformer} & 2.8B & 6.64 & - & 0.69 & 0.56 & - & - \\
& REPA~\cite{yu2025representation} & 675M & 1.42 & 305.7 & 0.80 & 0.65 & $250 \times 2$ & 3.56 \\
& LightningDiT~\cite{yao2025reconstruction} & 675M & 1.35 & 295.3 & 0.79 & 0.65 & $250 \times 2$ & 3.54 \\
& MeanFlow~\cite{geng2025mean} & 675M & 3.43 & 234.0 & 0.78 & 0.55 & $1$ & 0.18 \\
& REPA-E~\cite{leng2025repae} & 675M & 1.83 & 296.4 & 0.81 & 0.62 & $250 \times 2$ & 3.55 \\
& REG~\cite{wu2025reg} & 675M & 1.32 & 308.9 & 0.81 & 0.66 & $250 \times 2$ & 3.55 \\
\hline
\multirow{10}{*}{\rotatebox{90}{{Autoregressive}}} 
& VQ\text{VAE-}2~\cite{razavi2019vqvae2} & 13.5B & 31.11 & 45.0 & 0.36 & 0.57 & - & - \\
& ViT\text{-}VQGAN~\cite{yu2021vitvqgan} & 1.7B & 4.17 & 175.1 & - & - & - & - \\
& VQGAN~\cite{esser2021vqgan} & 1.4B & 15.78 & 74.3 & - & - & 1024 & 12.76 \\
& RQTransformer~\cite{lee2022rqtransformer} & 3.8B & 7.55 & 134.0 & - & - & - & - \\
& LlamaGen\text{-}XL~\cite{sun2024llamagen} & 775M & 2.62 & 244.1 & 0.80 & 0.57 & - & - \\
& MaskGIT~\cite{chang2022maskgit} & 227M & 6.18 & 182.1 & 0.80 & 0.51 & 8 & 0.97 \\
& RCG~\cite{li2023self} & 502M & 3.49 & 215.5 & - & - & - & - \\
& MAR\text{-}H~\cite{li2024autoregressive} & 943M & 1.55 & 303.7 & 0.81 & 0.62 & $256 \times 2$ & 1.23 \\
& FlowAR\text{-}L~\cite{ren2024flowar} & 589M & 1.90 & 281.4 & 0.83 & 0.57 & - & - \\
& VAR\text{-}d30~\cite{tian2024var} & 2.0B & 1.92 & 323.1 & 0.82 & 0.59 & $10 \times 2$ & 1.12 \\
\hline
\multirow{6}{*}{\rotatebox{90}{{Multi-stage}}}
& CDM~\cite{ho2022cdm} & - & 4.88 & 158.7 & - & - & - & - \\
& PixelFlow~\cite{chen2025pixelflow} & 677M & 1.98 & 282.1 & 0.81 & 0.60 & $30\!\times\!4\!\times\!2$ & 7.62 \\
& RIN~\cite{jabri2022rin} & 410M & 3.42 & 182.0 & - & - & - & - \\
& PaGoDA~\cite{kim2024pagoda} & 0.9B & 1.56 & 259.6 & - & 0.59 & 1 & 0.72 \\
& FractalMAR\text{-}H~\cite{li2025fractalmar} & 848M & 6.15 & 348.9 & 0.81 & 0.46 & $256\!\times\!2$ & 4.46 \\
& \makecell[l]{Single-stage Pixel \\ ($K{=}1$)} & 677M & 1.82 & 305.8 & 0.81 & 0.60 & $250 \times 2$ & 8.73 \\
\hline
\rowcolor{lightblue!30} & \textbf{CDC-FM (ours)} & 677M & 1.97 & 298.2 & 0.80 & 0.58 & $4 \times 4 \times 2$ & 1.31 \\
\rowcolor{lightblue!30} & \textbf{CDC-FM (ours)} & 677M & 1.87 & 301.8 & 0.82 & 0.60 & $8 \times 4 \times 2$ & 2.66 \\
\rowcolor{lightblue!30} & \textbf{CDC-FM (ours)} & 677M & 1.81 & 304.3 & 0.83 & 0.61 & $16 \times 4 \times 2$ & 5.18 \\
\hline
\end{tabular}

\end{table*}

\textbf{Comparison with Recent Low-NFE and Latent-Space Methods.}
We additionally compare against recent extreme low-NFE and latent-space frontiers, including MeanFlow~\cite{geng2025mean} (1-NFE), REPA-E~\cite{leng2025repae}, and REG~\cite{wu2025reg}, as well as a single-stage pixel baseline ($K{=}1$) that shares our backbone, training schedule, and inference solver but removes the conditional dependent coupling. As shown in Table~\ref{tab:classcondition_imagenet_256}, REG and LightningDiT attain the lowest absolute FID, but they operate in latent space and incur the full $250\!\times\!2$ NFE wall-clock cost. MeanFlow is faster but trades $\sim$2 FID points; the single-stage pixel baseline matches our FID but is roughly $1.7\times$ slower than CDC-FM at the $16{\times}4{\times}2$ configuration. Across all multi-stage and pixel-space competitors with reported runtimes, CDC-FM is the only method that simultaneously achieves sub-2.0 FID and sub-6\,sec/img on ImageNet-1k/256, and it Pareto-dominates the single-stage pixel baseline at matched FID under the same backbone and pipeline (see ablation in Appendix~\S\ref{appendix:efficiency_validation}).

\textbf{Inference Efficiency.}
Thanks to the disentangled design of each generation stage in the proposed multi-stage generation model with \textbf{conditional dependent coupling}, we explicitly separate coarse image generation from fine image refinement. This disentanglement enables both efficient inference and training while maintaining strong performance. Table~\ref{tab:classcondition_imagenet_256} reports a comparison of inference time costs between representative methods with CDC-FM under varying numbers of function evaluations (NFE) on the ImageNet-1k/256 benchmark. Furthermore, we measure the inference time of CDC-FM at varying image resolutions: \{64, 128, 256, 512\} on the ImageNet-1k benchmark. Figure~\ref{fig:InfT_FID_resolution} showcases the trend that the inference time of the proposed multi-stage generative model with \textbf{conditional dependent coupling} is linear to the generated image size, which coincides with Theorem~\ref{theorem:lowerInfer} and its proof in Appendix~\S~\ref{pf:fasterInfer}.

\begin{figure}[htbp]
  \centering
  \begin{minipage}{0.49\linewidth}
    \centering
    \includegraphics[width=\linewidth]{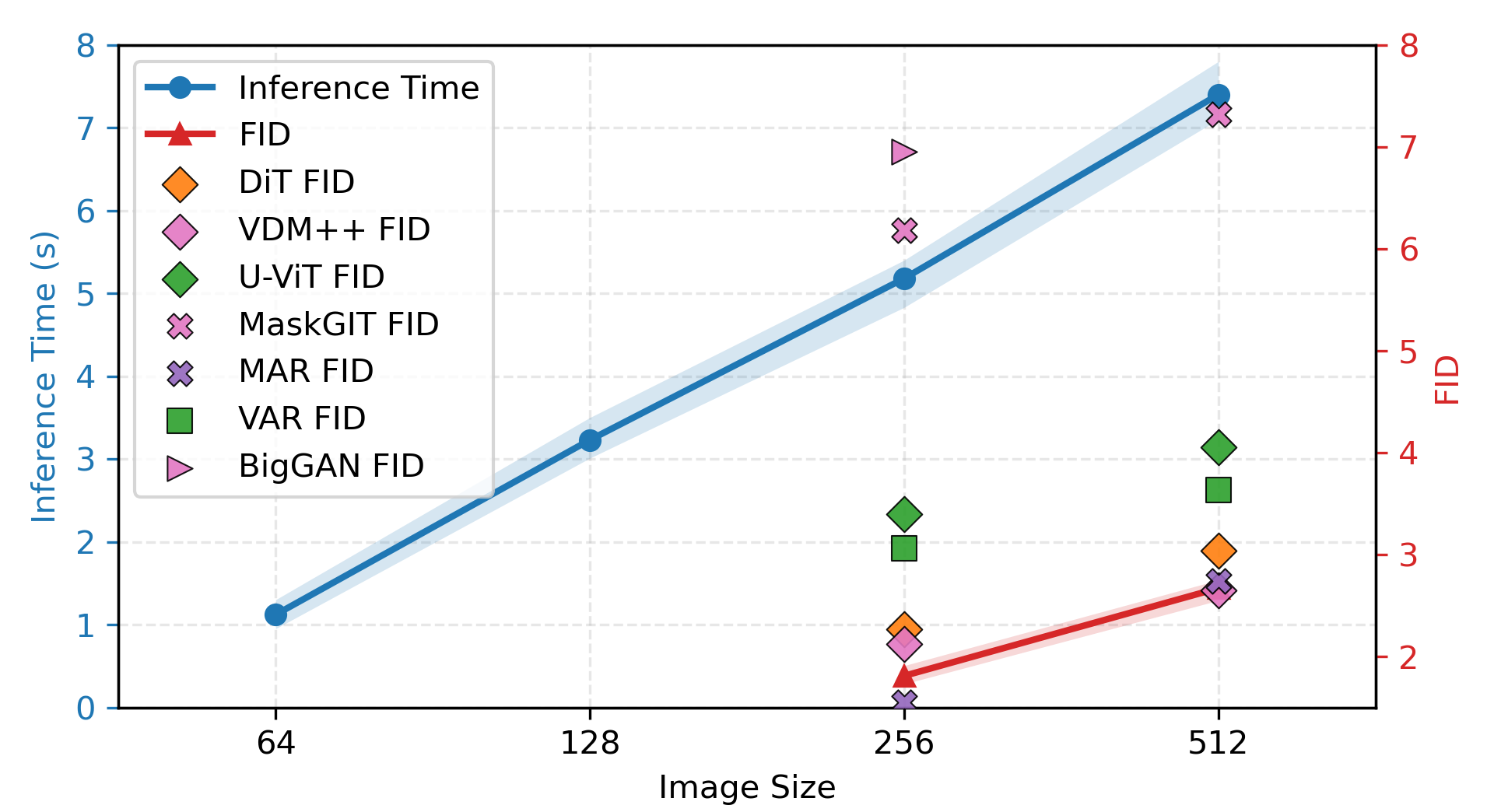}

    \caption{Inference time / FID across image sizes.}
    \label{fig:InfT_FID_resolution}
  \end{minipage}
  \hfill
  \begin{minipage}{0.49\linewidth}
    \centering
    \includegraphics[width=\linewidth]{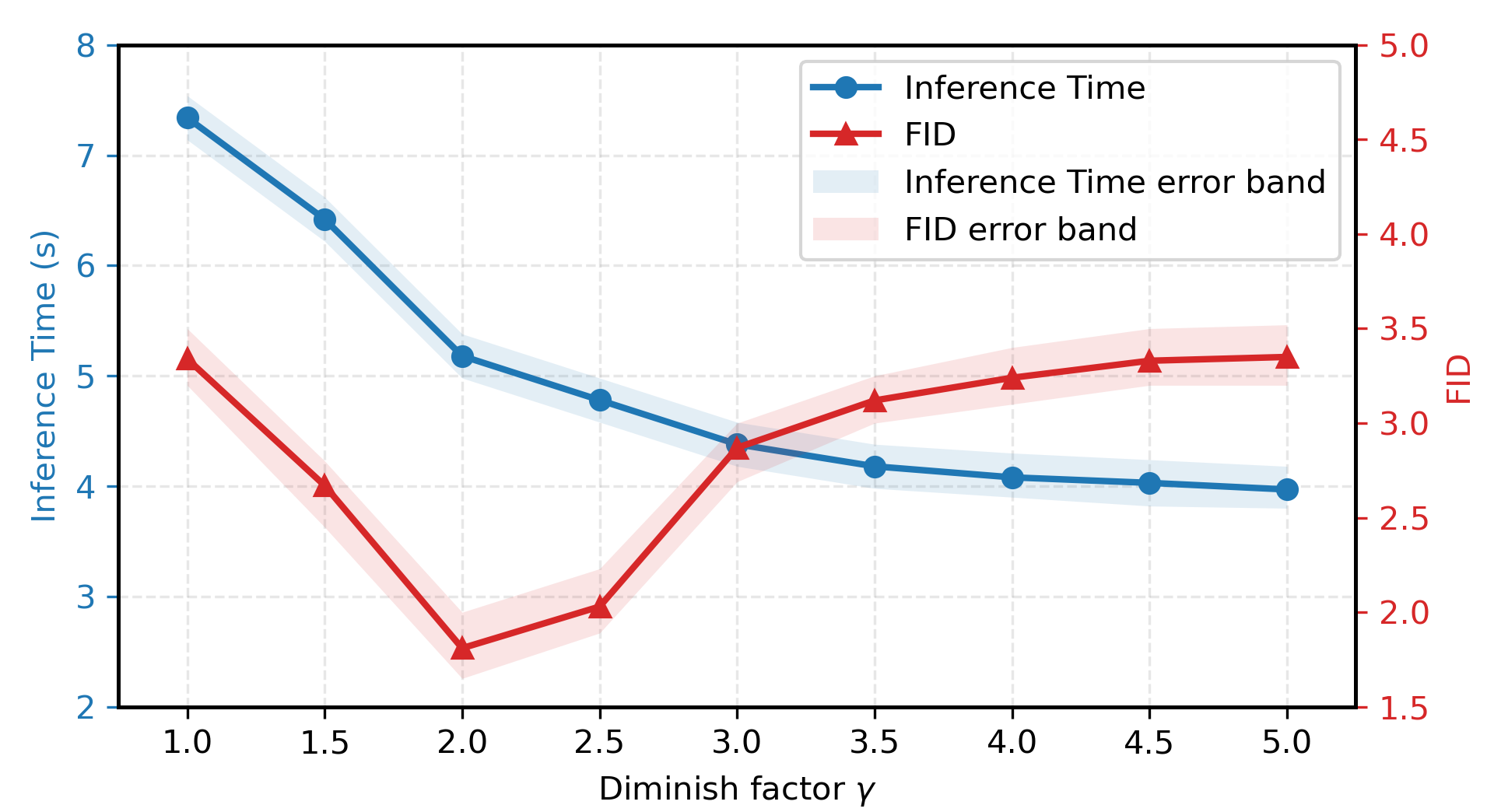}

    \caption{Effect of diminish factor.}
    \label{fig:DiminishFactor}
  \end{minipage}
\end{figure}

\begin{figure}[htbp]
  \centering
  \begin{minipage}{0.49\linewidth}
    \centering
    \includegraphics[width=\linewidth]{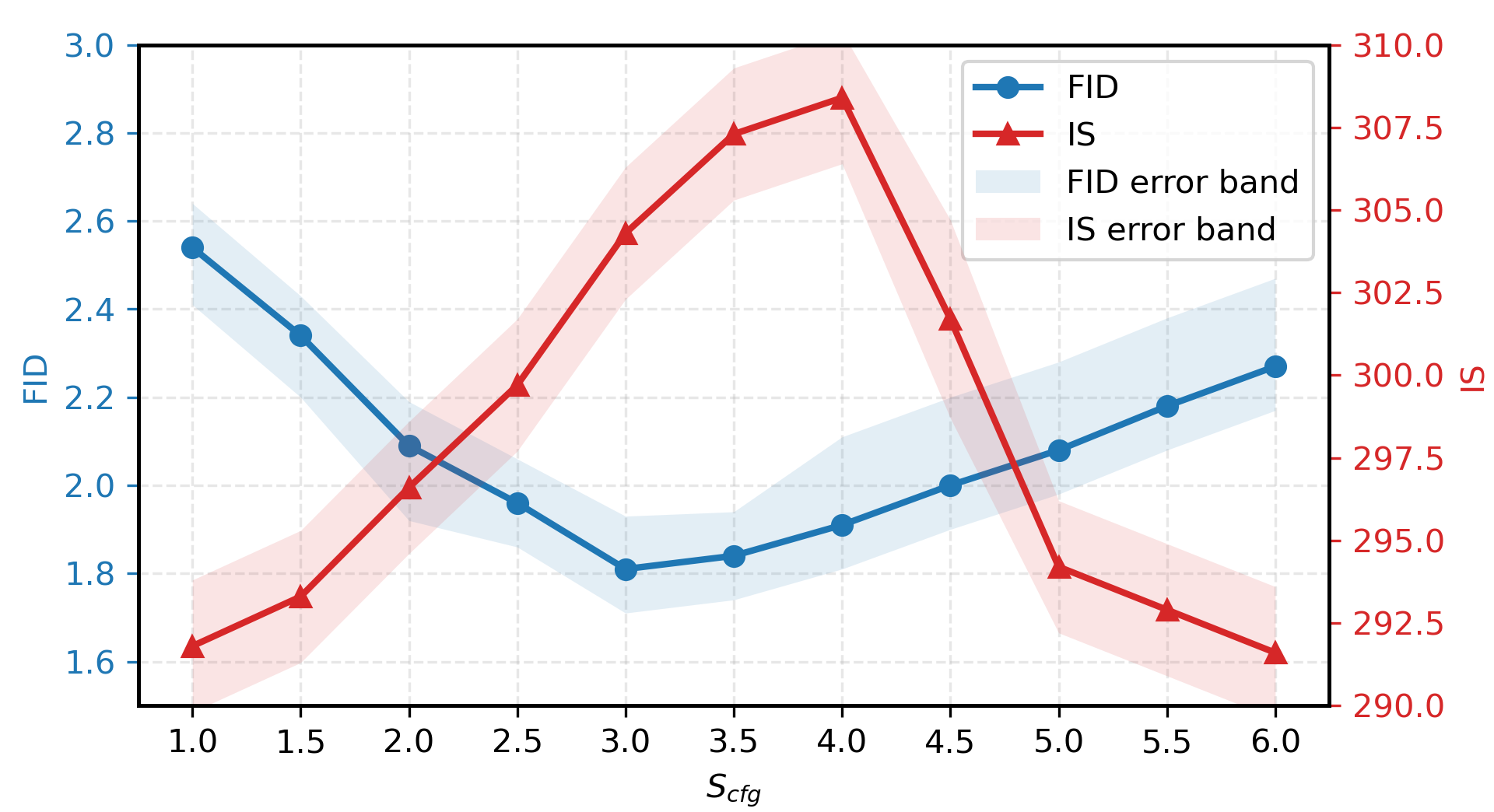}

    \caption{FID and IS under different $S_\text{cfg}$.}
    \label{fig:cfg_FID_IS}
  \end{minipage}
  \hfill
  \begin{minipage}{0.49\linewidth}
    \centering
    \includegraphics[width=\linewidth]{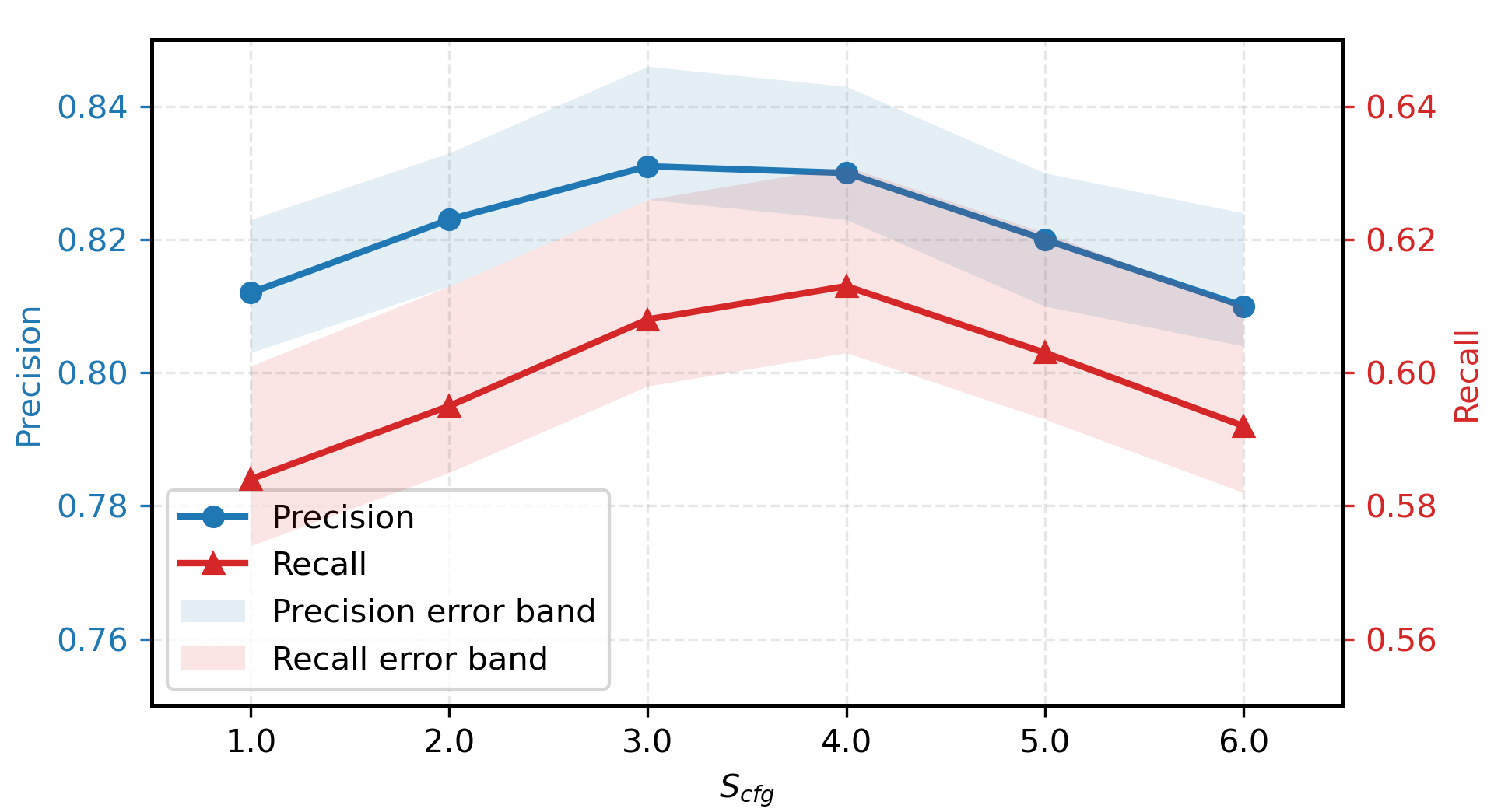}

    \caption{Precision and recall under different $S_\text{cfg}$.}
    \label{fig:cfg_pre_recall}
  \end{minipage}
\end{figure}

\textbf{Classifier-free Guidance.}
As discussed in Sec.~\ref{sec:cfg}, we employ a constant CFG strength $S_\text{cfg}$ across all stages. This strategy is both simple and elegant to implement, while remaining robust in practice. The performance of CDC-FM on the ImageNet-1k/256 benchmark under varying $S_\text{cfg} \in [1.0, 6.0]$ is reported in Figure~\ref{fig:cfg_FID_IS} and Figure~\ref{fig:cfg_pre_recall}.
To validate the effectiveness of using a constant CFG strength $S_{\text{cfg}}$ across all stages, we report the per-stage FID under varying $S_{\text{cfg}}$. 
As shown in Figure~\ref{fig:stageFID}, each stage attains its lowest FID at a similar $S_{\text{cfg}}$ when using \textbf{conditional dependent coupling}. 
This demonstrates that employing a constant CFG strength $S_{\text{cfg}}$ across all stages is feasible and effective.

\textbf{Impact of the Diminish Factor.}
We further investigate the effect of the diminish factor $\gamma$ in CDC-FM on the ImageNet-1k/256 benchmark, sweeping $\gamma$ from $1.0$ to $5.0$ in increments of $0.5$ and keeping $\sigma = 1$. 
The results, summarized in Figure~\ref{fig:DiminishFactor}, show that the FID initially decreases as $\gamma$ increases, up to a certain tipping point, after which the FID begins to rise. 
At the same time, the inference time steadily decreases with larger $\gamma$ and eventually plateaus.
To test the interaction between $\gamma$ and $\sigma$, we further evaluate the FID under varying $\gamma$ and $\sigma$ in Figure~\ref{fig:interact_gamma_sigma}.
When the stage-$k$ noise level $\sigma_k$ becomes too small (which occurs when $\gamma$ is too large or the base noise level~$\sigma$ is too small), each $x_0^{(k)}$ collapses toward a lower-dimensional subset embedded in a higher-dimensional space, causing the induced distribution to concentrate on a lower-dimensional manifold~\cite{albergo2023stochasticcouplings}. 
\begin{wrapfigure}{r}{0.5\linewidth} %
    \centering

    \includegraphics[width=0.95\linewidth]{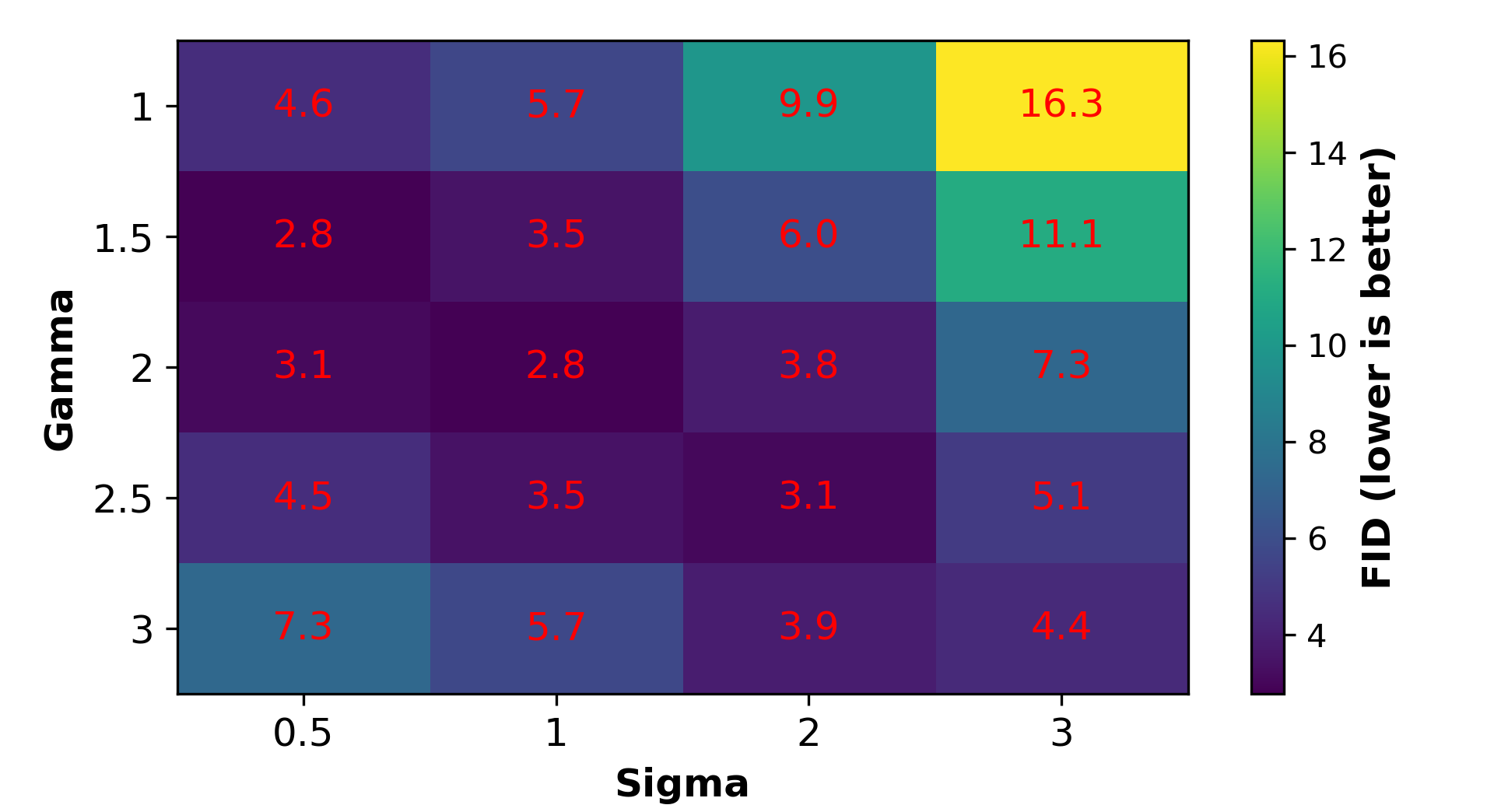}

    \caption{FID under varying $\gamma$ and $\sigma$.}
    \label{fig:interact_gamma_sigma}

\end{wrapfigure}
Conversely, when $\sigma_k$ becomes too large (resulting from a too-small $\gamma$ or a too-large $\sigma$), $x_0^{(k)}$ retains insufficient information from the previous stage. 
Selecting an appropriate $\sigma_k$ for each stage mitigates both extremes. 
Introducing a small amount of Gaussian noise smooths the base density, ensuring it remains well-defined across the entire ambient space~\cite{albergo2023stochasticcouplings}, while still preserving adequate information from the preceding stage. 
This balance enables each stage to refine its input effectively using fewer simulation steps.
\begin{table}[tbp]
    \centering
    \caption{Ablation studies on Unified Architecture and Fusion Mechanism.}
    \label{tab:arch_fusion_combined}

    \small
    \renewcommand{\arraystretch}{1.1}
    \setlength{\tabcolsep}{4pt}
    
    \begin{tabular}{l c c c c l c}
        \toprule
        \multicolumn{4}{c}{\textbf{Architecture Comparison}} & \hspace{1pt} & \multicolumn{2}{c}{\textbf{Fusion Ablation}} \\
        \cmidrule(r){1-4} \cmidrule(l){6-7}
        Method & FID & Params & Time & & Variant & FID \\
        \midrule
        Unified & \textbf{1.81} & 677M & \textbf{76h} & & w/ Attn & \textbf{1.81} \\
        Stage-spec. & 2.68 & $130{\times}3{+}458$M & 78h & & w/o Attn & 1.84 \\
        \bottomrule
    \end{tabular}
    
\end{table}

To examine the effectiveness of the proposed method and the contribution of each individual component, we conduct a series of ablation studies. 
These include evaluating the unified model parameterization, the use of interpolants, the choice of resolution hierarchy (i.e., the stage count $K$), and the method by which the resolution embedding is incorporated. To evaluate whether the unified generative model parameterization enables effective knowledge sharing across stages, we compare the unified DiT (CDC-FM) with stage-specific DiTs on the ImageNet-1k/256 benchmark, reporting FID, model parameters, and training time. 
For all experiments, we adopt the resolution hierarchy $[32, 64, 128, 256]$ with stage count $K = 4$. 
For the stage-specific DiTs, we use SiT-B (130M) for the $32$, $64$, and $128$ resolutions, and SiT-L (458M) for the $256$ resolution. 
To ensure a fair comparison, we keep the \textbf{conditional dependent coupling} identical to that used in the unified DiT. 
The results are presented in Table~\ref{tab:arch_fusion_combined} (left).
In the unified DiT, the resolution embedding $e_k$ serves as a conditioning signal that informs the model, “you are now operating at resolution $d_k$.” 
\begin{table}[tbp]
    \centering
    
    \caption{Ablation studies on Interpolants (Left) and Hierarchy Stages (Right).}
    \label{tab:ablation_combined}

    \small
    \renewcommand{\arraystretch}{1.1}
    \setlength{\tabcolsep}{4pt}
    
    \begin{tabular}{l l c c c c c}
        \toprule
        \multicolumn{3}{c}{\textbf{Interpolant Variant}} & \hspace{1pt} & \multicolumn{3}{c}{\textbf{Hierarchy Setting}} \\
        \cmidrule(r){1-3} \cmidrule(l){5-7}
        Type & Definition & FID & & Stages & Res. Path & FID \\
        \midrule
        Linear & \makecell[l]{$\alpha_t{=}1{-}\tau;$ \\ $\beta_t{=}\tau$} & 1.81 & & 4 & 32\dots256 & 1.81 \\
        Trig. & \makecell[l]{$\alpha_t{=}\cos\frac{\pi\tau}{2};$ \\ $\beta_t{=}\sin\frac{\pi\tau}{2}$} & 2.11 & & 2 & 64, 256 & 1.97 \\
        VP-SDE & \makecell[l]{$\alpha_t{=}\sqrt{1{-}\tau^2};$ \\ $\beta_t{=}\tau$} & 2.03 & & - & - & - \\
        \bottomrule
    \end{tabular}
    
\end{table}

This allows all stages $k$ to share the same model parameters, in contrast to stage-specific DiTs, which require separate parameter sets for each resolution. 
Such unified parameterization encourages efficient knowledge sharing across stages and leads to more compact and effective models. To further examine the selection of stochastic interpolants, we additionally test alternative interpolants: trigonometric~\cite{albergo2023stochastic}, where $\alpha_t = \cos \frac{\pi}{2}\tau$ and $\beta_t = \sin \frac{\pi}{2}\tau$, and score-based diffusion model~\cite{song2021maximum}, where $\alpha_t = \sqrt{1-\tau^2}$ and $\beta_t = \tau$.
The results are provided in Table~\ref{tab:ablation_combined} (left). For the ImageNet-1k/256 experiments, we use the resolution hierarchy $[32, 64, 128, 256]$ with stage count $K=4$. 
We also test an alternative hierarchy $[64, 256]$ with stage count $K=2$. 
A comparison between the two settings can be found in Table~\ref{tab:ablation_combined} (right). Here, we additionally evaluate a variant that drops cross-attention and fuses $e_k$ with the timestep embedding via direct addition. 
The results of this comparison are shown in Table~\ref{tab:arch_fusion_combined}  (right).

\section{Conclusion}

We develop a unified multi-stage generation framework based on \emph{stochastic interpolant} with conditional dependent coupling, which enables accurate data distribution learning while operating efficiently in pixel space. The entire process can be parameterized by a single DiT, thereby achieving end-to-end optimization and facilitating knowledge sharing across stages. Under explicitly stated assumptions on the noise-decay schedule, per-evaluation cost, and NFE--transport-cost proportionality (Theorems~\ref{theorem:lowerTransportCost} and~\ref{theorem:lowerInfer}, with empirical validation of the latter two in Appendix~\S\ref{appendix:efficiency_validation}), CDC-FM provably reduces transport cost and asymptotic inference time. On ImageNet-1k/256 and 512, CDC-FM is Pareto-competitive with single- and multi-stage baselines, achieving comparable or better FID/IS at substantially lower wall-clock cost in the low-NFE regime. Limitations and future directions, including SDE instantiation and the choice of the first-stage prior, are discussed in Section~\ref{sec:limitations}.

\clearpage

\begin{ack}
This manuscript was co-authored by Oak Ridge National Laboratory (ORNL), operated by UT-Battelle, LLC under Contract No. DE-AC05-00OR22725 with the U.S. Department of Energy. Any subjective views or opinions expressed in this paper do not necessarily represent those of the U.S. Department of Energy or the United States Government.
The authors declare no competing interests.
\end{ack}

\bibliographystyle{unsrtnat}
\bibliography{references}

\newpage
\appendix

\section{Related Works}

\textbf{Latent Space Generation.}
Variational Autoencoders (VAEs) enable generative models to perform Flow/Diffusion generation, as well as autoregressive generation, within a lower-dimensional latent space~\cite{rombach2022ldm, peebles2023dit, ma2024sit, razavi2019vqvae2, esser2021vqgan, yu2021vitvqgan, lee2022rqtransformer, sun2024llamagen, ren2024flowar, tian2024var, jabri2022rin, li2026diverseditpp,li2026diffaid}. This leads to more efficient training and inference. However, the compact representations produced by the encoder and the subsequent decoding process inevitably introduce a degree of information and detail loss~\cite{li2024autoregressive, hoogeboom2025sid2, hoogeboom2023simple}. Despite this limitation, VAEs remain a crucial component in these generative models. Moreover, they typically cannot be trained jointly with the generative model in a fully end-to-end manner~\cite{chen2025pixelflow}.

\textbf{Pixel Space Generation.}
Although directly implementing the generation process avoids the need for VAEs and facilitates the preservation of fine-grained details, it is highly inefficient for both training and inference~\cite{rombach2022ldm, hoogeboom2025sid2, hoogeboom2023simple, bao2023uvit}. Furthermore, due to the high dimensionality and complexity of the data distribution, learning to generate samples that faithfully capture detailed information from the target distribution becomes particularly challenging, especially in the absence of VAEs to extract meaningful latent features~\cite{chen2025pixelflow, jin2024pyramidal, ren2024flowar, tian2024var, jabri2022rin, sun2024llamagen}.

\textbf{Multi-stage Generation.}
Multi-stage generation methods advance sample synthesis by decomposing the process into a sequence of stages, where early stages operate in a lower-dimensional space to improve efficiency~\cite{chen2025pixelflow, jin2024pyramidal, tian2024var, jiao2025flexvar, renflowar, ho2022cdm, jabri2022rin}. In addition, the overall generative task is divided into subtasks at each stage, enabling the model to gradually capture the complexity of the target data distribution~\cite{tian2024var, jiao2025flexvar, renflowar, ho2022cdm}. However, the use of disentangled stage designs hinders unified model parameterization and end-to-end optimization~\cite{chen2025pixelflow, jin2024pyramidal}. Furthermore, the decoupling inherent in multi-stage frameworks may lead to inaccurate modeling of the data distribution, thereby introducing errors that can accumulate across stages~\cite{chen2025pixelflow, jin2024pyramidal}.

\section{Qualitative results}
We visualize the class-conditional image generation result of our method in Figure~\ref{fig:imagenet256}, demonstrating high-fidelity, accuracy, and diversity across all the classes.

\newcommand{\CellW}{.25\linewidth}       %
\newcommand{\HalfCellW}{.125\linewidth}  %

\newcommand{\OneImg}[1]{%
  \includegraphics[width=\CellW,height=\CellW,keepaspectratio=false]{#1}%
}

\newcommand{\FourOfAClass}[4]{%
  \setlength{\tabcolsep}{0pt}\renewcommand{\arraystretch}{0}%
  \begin{tabular}{@{}cc@{}}
    \includegraphics[width=\HalfCellW,height=\HalfCellW,keepaspectratio=false]{#1} &
    \includegraphics[width=\HalfCellW,height=\HalfCellW,keepaspectratio=false]{#2} \\
    \includegraphics[width=\HalfCellW,height=\HalfCellW,keepaspectratio=false]{#3} &
    \includegraphics[width=\HalfCellW,height=\HalfCellW,keepaspectratio=false]{#4} \\
  \end{tabular}%
}

\begin{figure}[t]
\centering
\setlength{\tabcolsep}{0pt}\renewcommand{\arraystretch}{0}
\begin{tabular}{@{}cccc@{}}

\OneImg{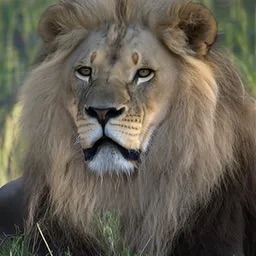} & \OneImg{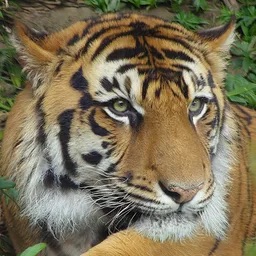} & \OneImg{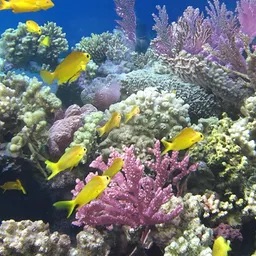} & \OneImg{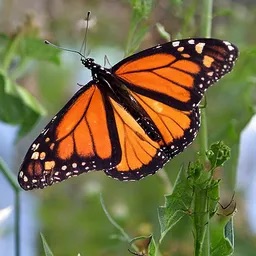} \\

\OneImg{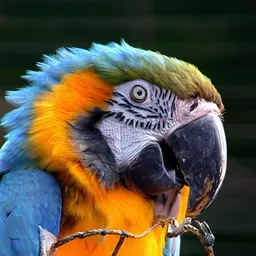} & \OneImg{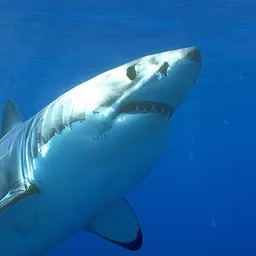} & \OneImg{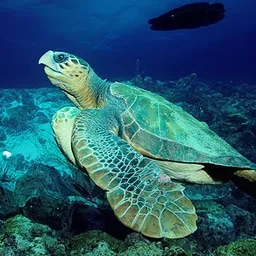} & \OneImg{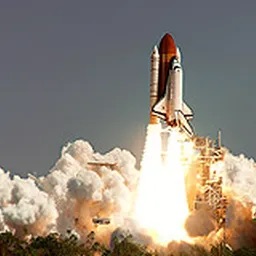} \\

\FourOfAClass{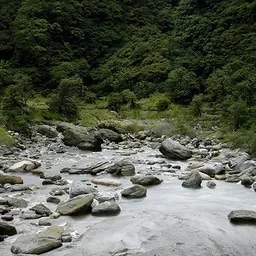}{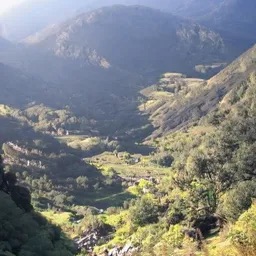}{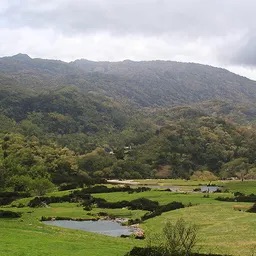}{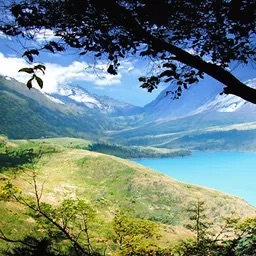} &
\FourOfAClass{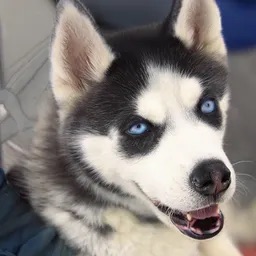}{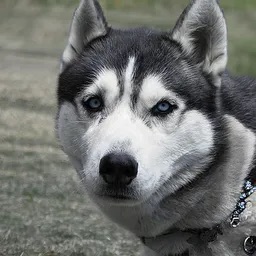}{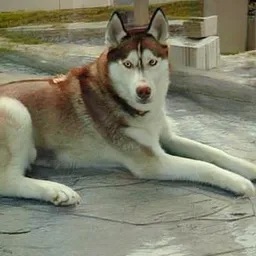}{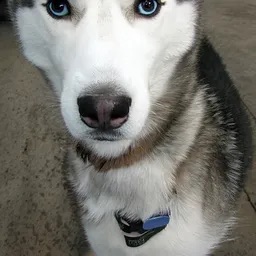} &
\FourOfAClass{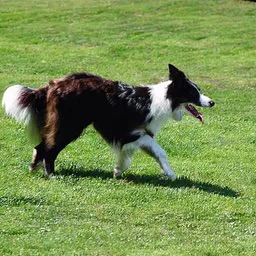}{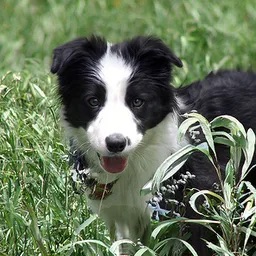}{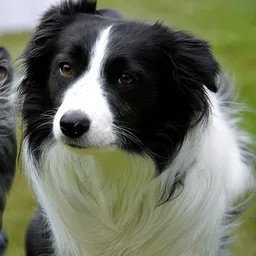}{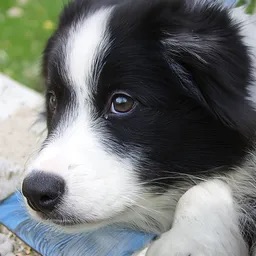} &
\FourOfAClass{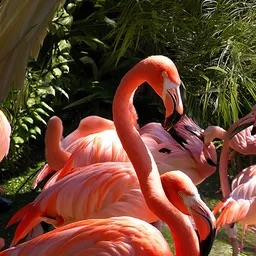}{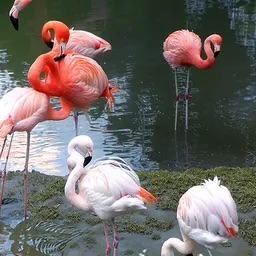}{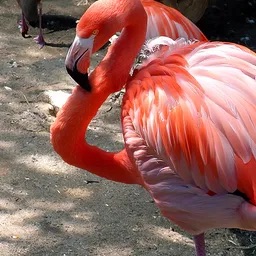}{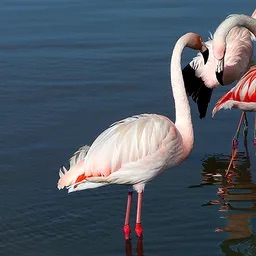} \\

\FourOfAClass{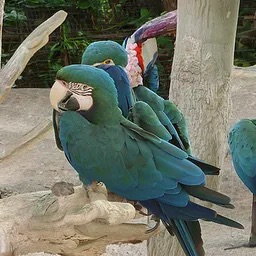}{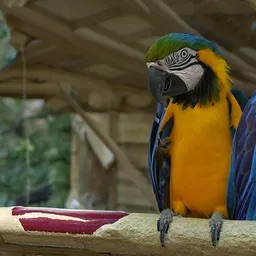}{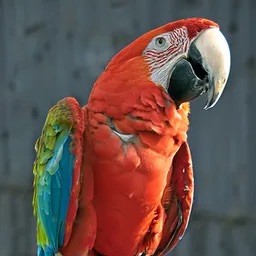}{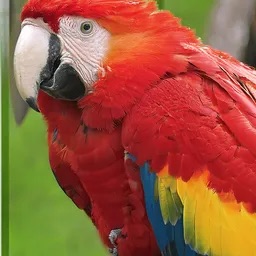} &
\FourOfAClass{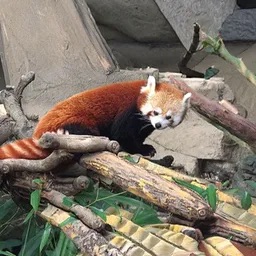}{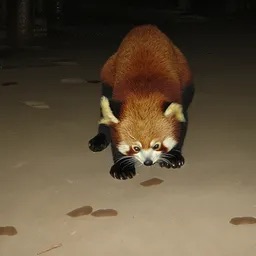}{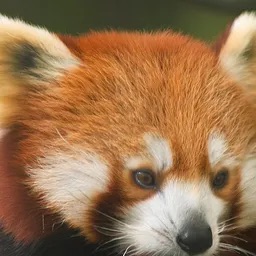}{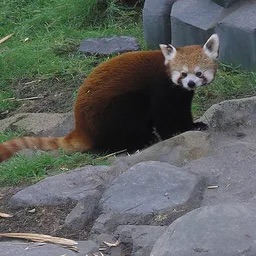} &
\FourOfAClass{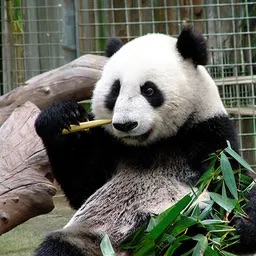}{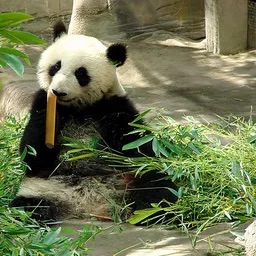}{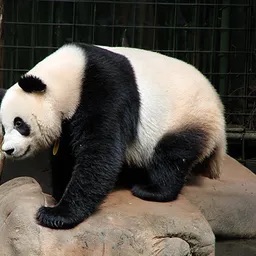}{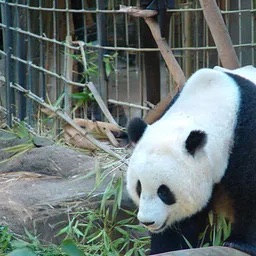} &
\FourOfAClass{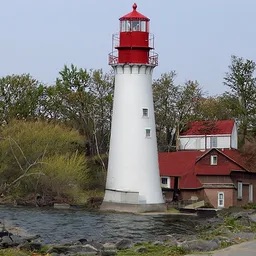}{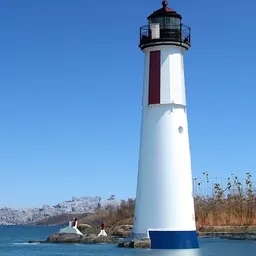}{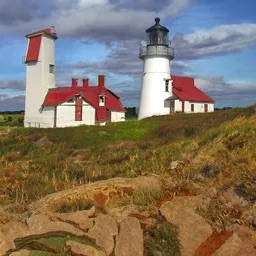}{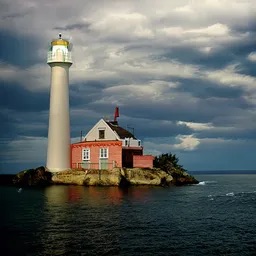} \\

\FourOfAClass{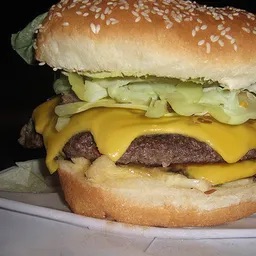}{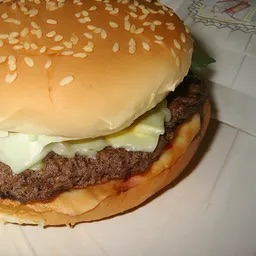}{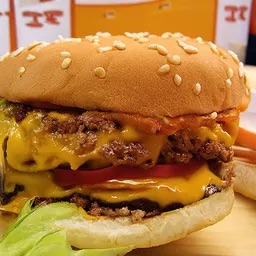}{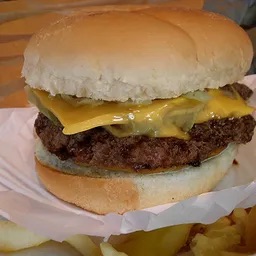} &
\FourOfAClass{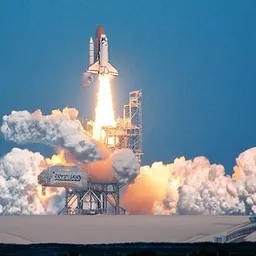}{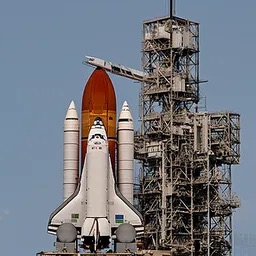}{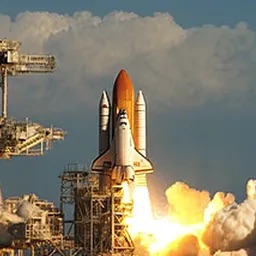}{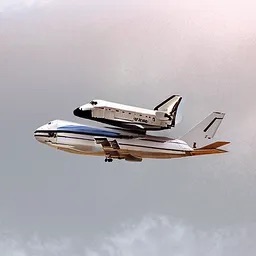} &
\FourOfAClass{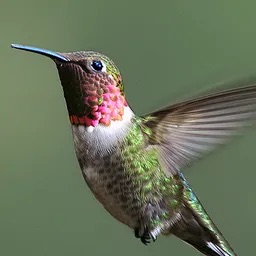}{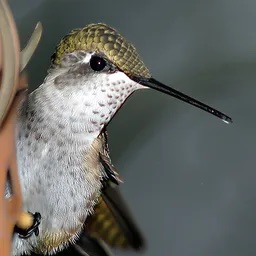}{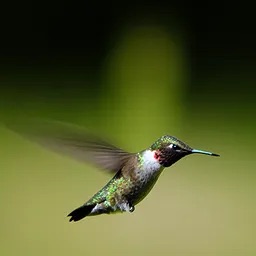}{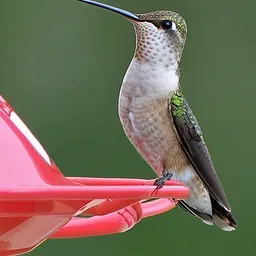} &
\FourOfAClass{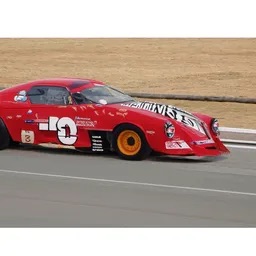}{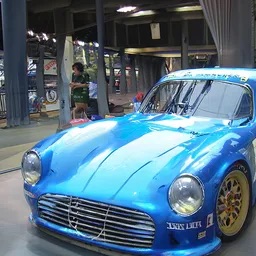}{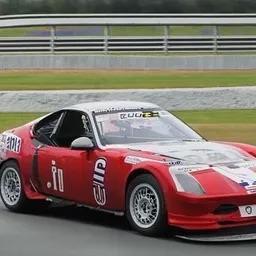}{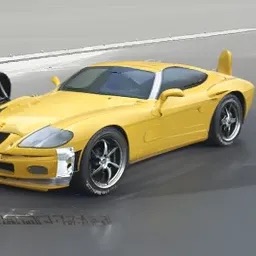} \\

\end{tabular}
\caption{ImageNet-1k/256 class-conditional generation results.  }
\label{fig:imagenet256}

\end{figure}

\section{Visualization of Unified Model with Conditional Dependent Coupling}
\label{appendix:viz}
In this section, we visualize representative examples of the output produced at each stage of the proposed method, together with the details added at each stage. 
The results shown in Figure~\ref{fig:interpolant} illustrate that the proposed framework indeed forms a Markov chain: conditioned on the output of stage $k-1$, the behavior of stage $k$ is independent of stage $k-2$. 
For a detailed explanation of how the joint distribution in \eqref{eq:conditional_data-dependent_coupling_joint_distribution_training} induces the Markovian inference process in \eqref{eq:inference_markov}, please refer to Appendix~\S\ref{appendix:cond-dep-markov}.

We also observe that the magnitude of the added details decreases as the stage index $k$ increases. 
This indicates that the method decomposes the generation process into multiple stages in a coarse-to-fine manner: earlier stages focus on generating diverse coarse structures, while later stages progressively refine fine-grained details. 
Such decoupling enables both strong generative performance and improved sampling efficiency.

\begin{figure*}[t] %
    \centering
    \includegraphics[width=0.68\linewidth]{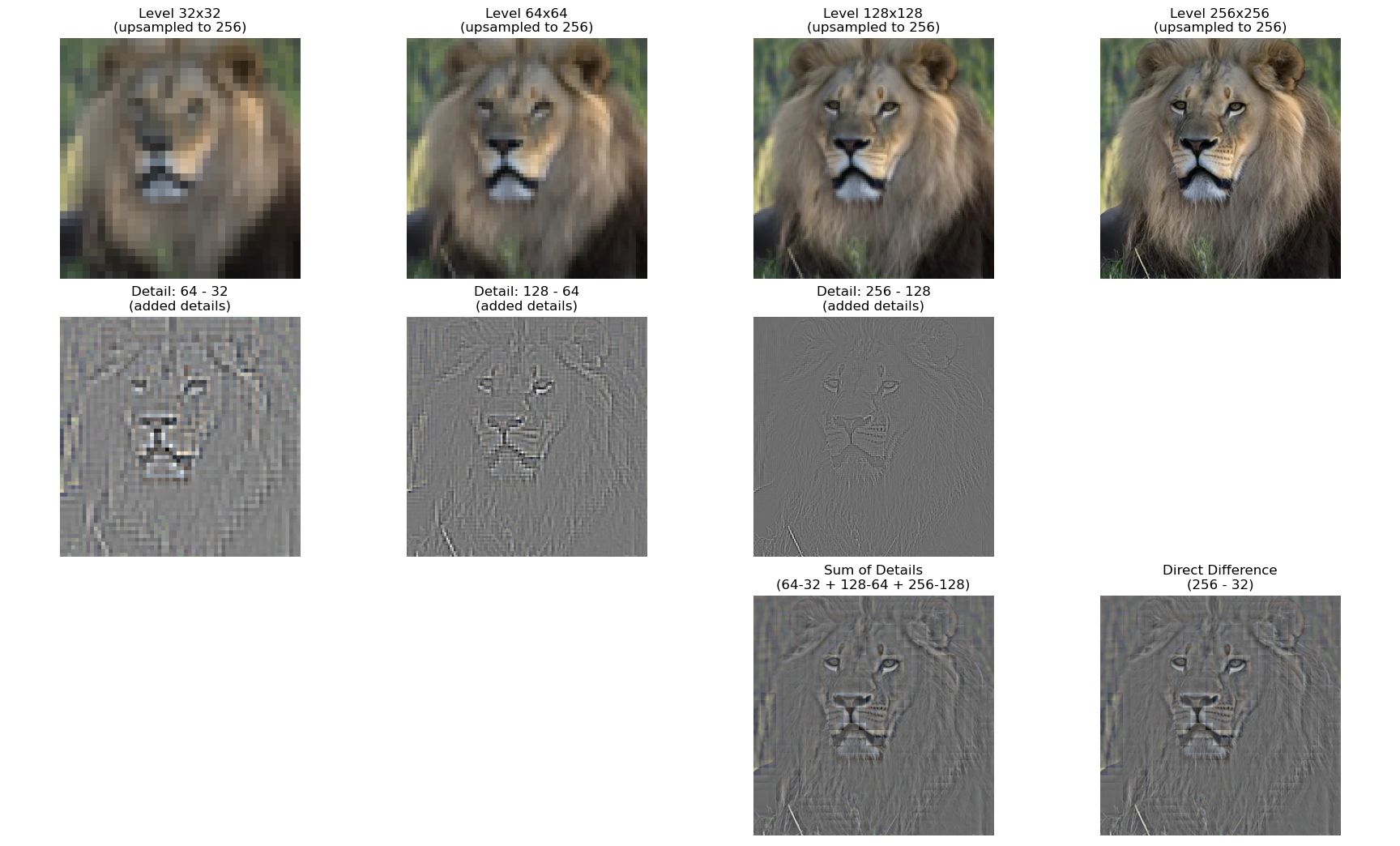}
    \vspace{0.2em} %

    \includegraphics[width=0.68\linewidth]{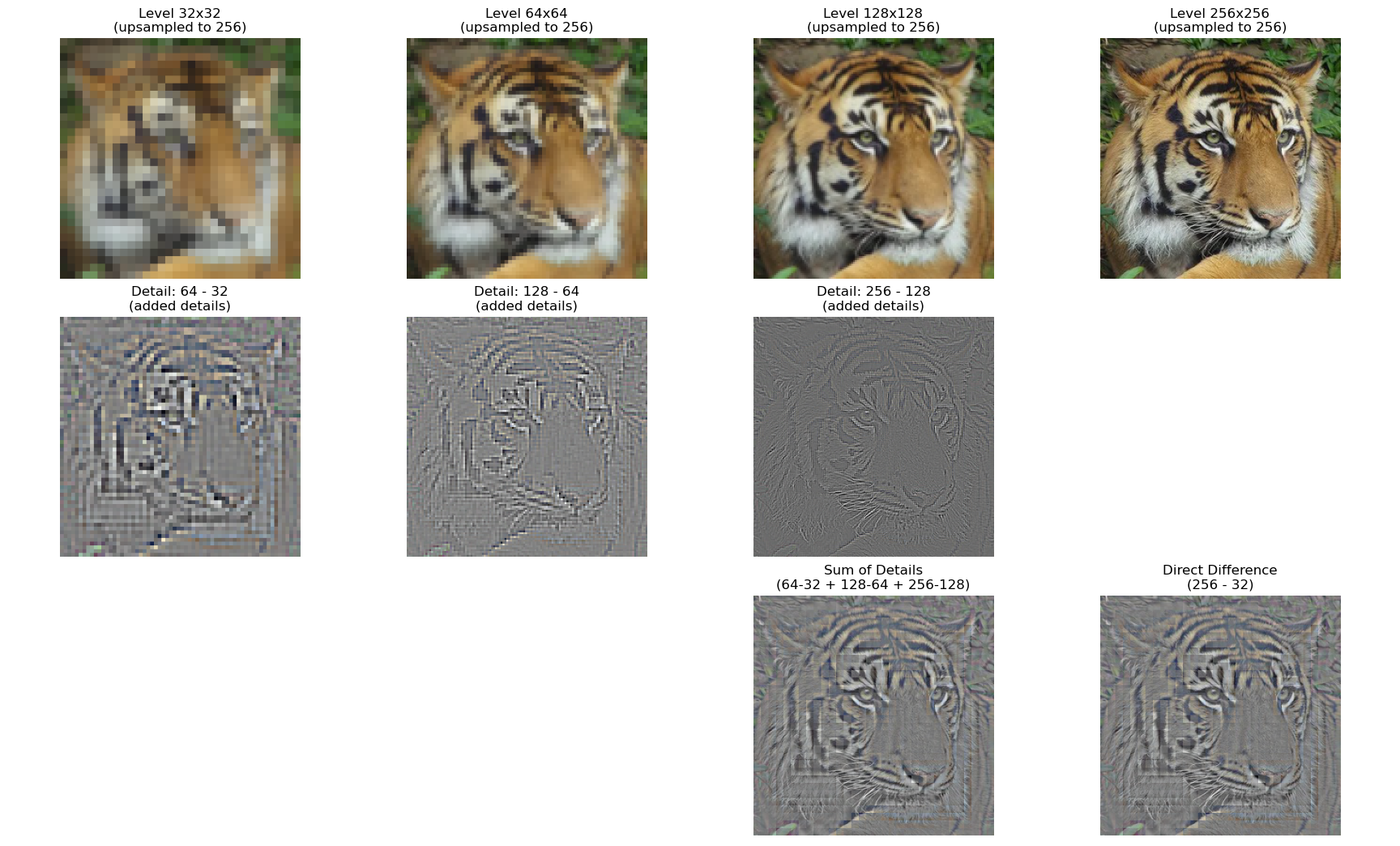}
    \vspace{0.2em}

    \includegraphics[width=0.68\linewidth]{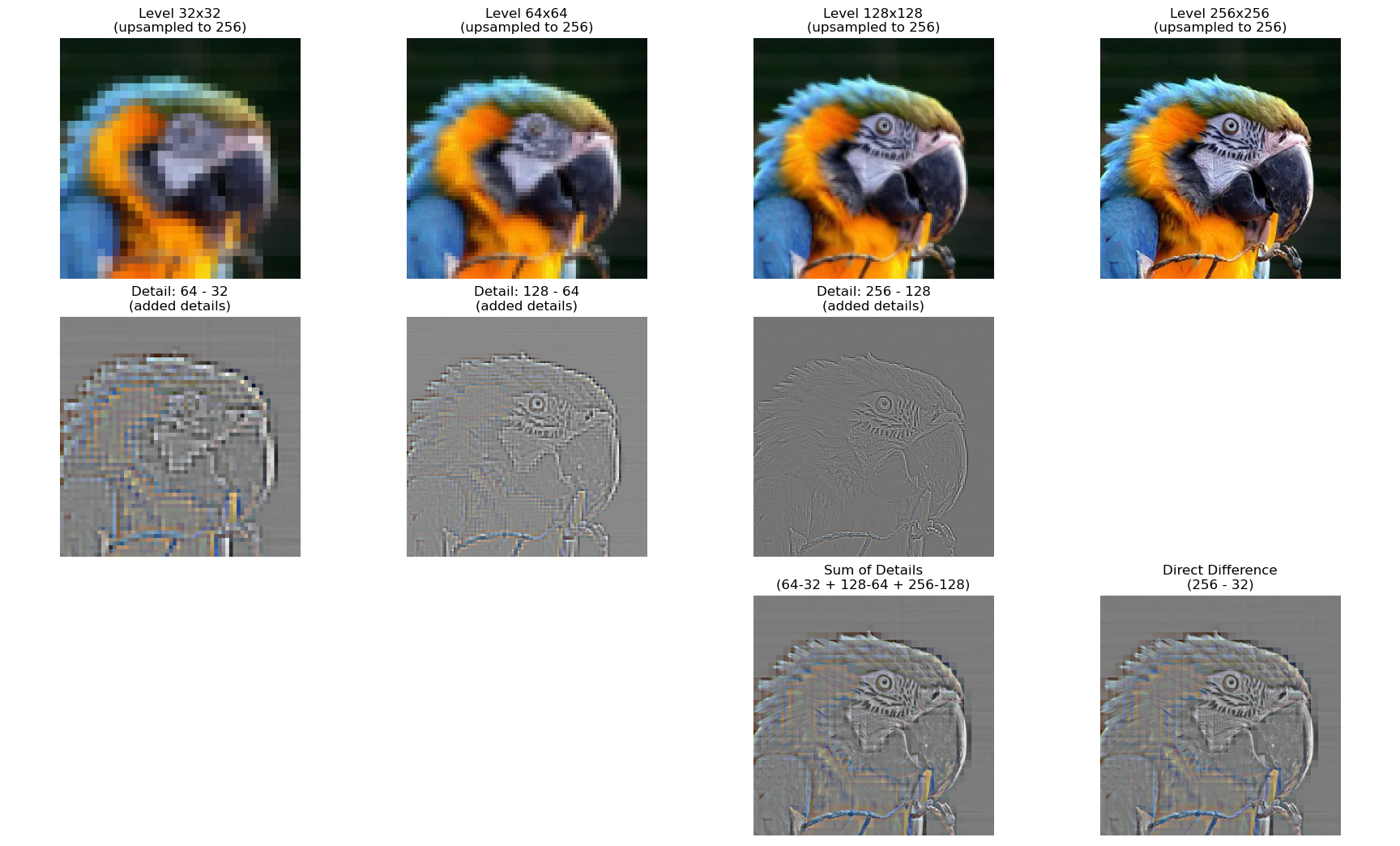}

    \caption{Visualization of each stage in CDC-FM, including the output at every stage, the details added at each stage, and the corresponding hierarchical structure of those details.}
    \label{fig:interpolant}

\end{figure*}

\textbf{Classifier-free Guidance.}
As discussed in Sec.~\ref{sec:cfg}, we employ a constant CFG strength $S_\text{cfg}$ across all stages. This strategy is both simple and elegant to implement, while remaining robust in practice. The performance of CDC-FM on the ImageNet-1k/256 benchmark under varying $S_\text{cfg} \in [1.0, 6.0]$ is reported in Figure~\ref{fig:cfg_FID_IS} and Figure~\ref{fig:cfg_pre_recall}.
{\color{colorAdded}
To validate the effectiveness of using a constant CFG strength $S_{\text{cfg}}$ across all stages, we report the per-stage FID under varying $S_{\text{cfg}}$. 
As shown in Figure~\ref{fig:stageFID}, each stage attains its lowest FID at a similar $S_{\text{cfg}}$ when using \textbf{conditional dependent coupling}. 
This demonstrates that employing a constant CFG strength $S_{\text{cfg}}$ across all stages is both feasible and effective.
}

\begin{figure}[t]
    \centering
    \includegraphics[width=0.5\linewidth]{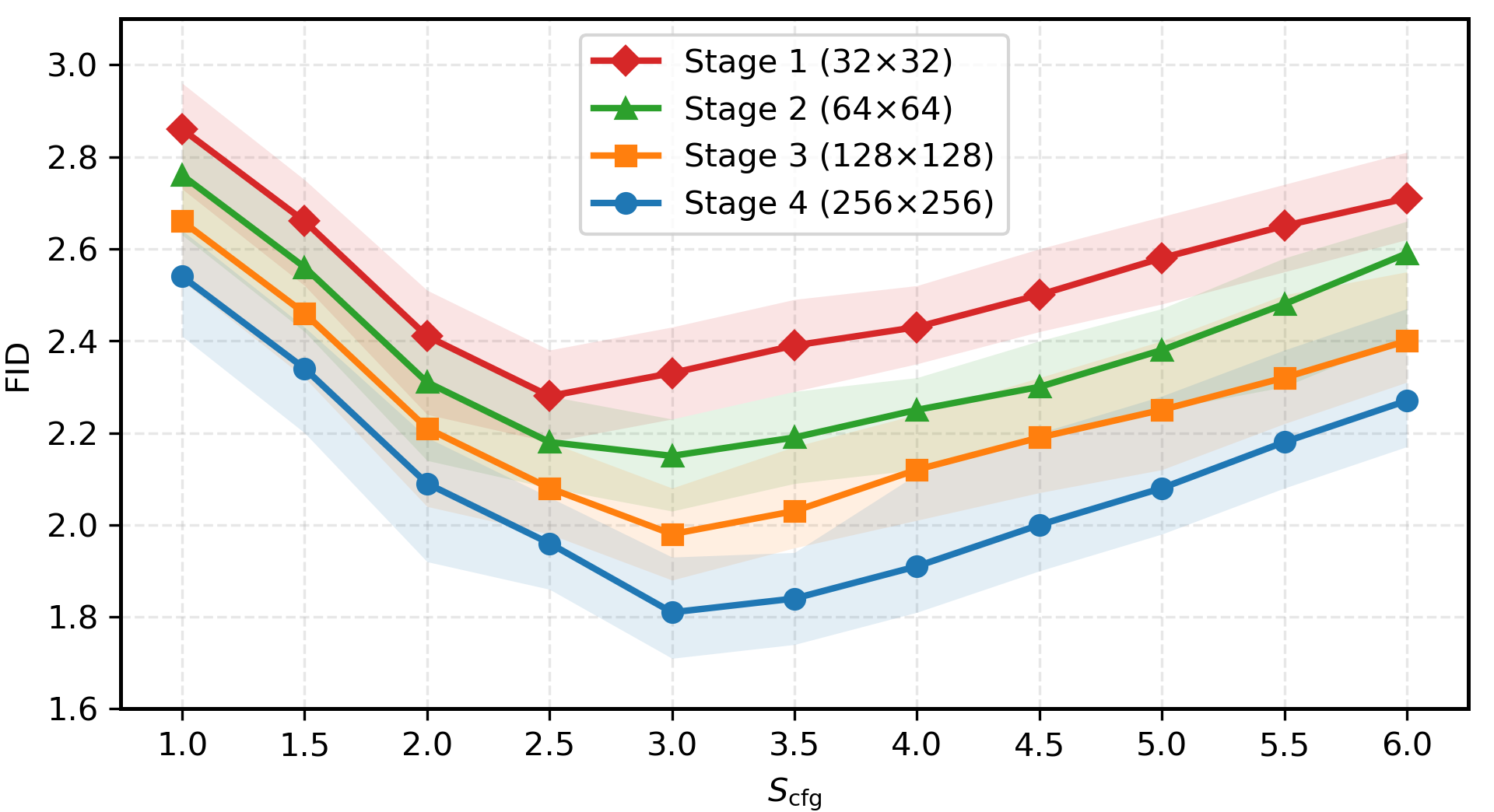}

    \caption{The per-stage FID under varying $S_{\text{cfg}}$.}
    \label{fig:stageFID}

\end{figure}

\section{Stochastic Interpolant} \label{appendix:SI}
The stochastic interpolant unifies the theory of Ordinary Differential Equations (ODEs) and Stochastic Differential Equations (SDEs).
\begin{definition}[Stochastic Interpolant] \label{ap:def:stochastic interpolant}
\cite{albergo2023stochastic, albergo2023stochasticcouplings, albergo2022building}
Given two probability density functions $\rho_0, \rho_1 : \mathbb{R}^d \rightarrow \mathbb{R}_{\geq 0}$, a stochastic interpolant between $\rho_0$ \textit{and} $\rho_1$ is a stochastic process $I_t$ defined as
\begin{equation}
\label{ap:eq:stochastic_interpolant}
\begin{aligned}
I_t = \alpha_t x_0 + \beta_t x_1 + \gamma_t z, \qquad t \in [0, 1]
\end{aligned}
\end{equation}
where $\alpha_t$, $\beta_t$, and $\gamma_t^2$ are differentiable functions of time satisfying the boundary conditions
$$
\alpha_0 = \beta_1 = 1, \quad \alpha_1 = \beta_0 = \gamma_0 = \gamma_1 = 0, \quad \text{and} \quad \alpha_t^2 + \beta_t^2 + \gamma_t^2 > 0 \quad \forall t \in [0, 1].
$$
The pair $(x_0, x_1)$ is drawn from a joint probability density $\rho(x_0, x_1)$ with finite second moments, whose marginals are $\rho_0$ and $\rho_1$:
\begin{align}
\int_{\mathbb{R}^d} \rho(x_0, x_1) dx_1 &= \rho_0(x_0), \\
\int_{\mathbb{R}^d} \rho(x_0, x_1) dx_0 &= \rho_1(x_1),
\end{align}
and $z \sim \mathcal{N}(0, \mathrm{Id})$  a Gaussian random variable independent of $(x_0, x_1)$, i.e., $z \perp (x_0, x_1)$.
\end{definition}

The stochastic interpolant framework uses information about the process $I_t$ to derive either an ODE or an SDE. The solutions $X_t$ to these equations are designed to push the initial law $\rho_0$ onto the law of the interpolant $I_t$ for all times $t \in [0, 1]$. Consequently, the process $I_t$ satisfies $I_{t=0} = x_0 \sim \rho_0(x_0)$ and $I_{t=1} = x_1 \sim \rho_1(x_1)$. This property allows for generative modeling: by drawing samples $x_0 \sim \rho_0(x_0)$ and using them as initial conditions $X_{t=0} = x_0$, one can generate samples $X_{t=1} \sim \rho_1(x_1)$ via numerical integration of the corresponding ODE or SDE.

\begin{theorem}[Transport equation] \label{ap:theorem:transport equation}
Let the time-dependent density of the stochastic interpolant $I_t$ be $\rho_t(x)$. We define the velocity field $b_t(x)$ and the score field $s_t(x)$ as:
\begin{equation}
\label{ap:eq:b}
\begin{aligned}
b_t(x) &= \mathbb{E}[\dot{I}_t \mid I_t = x]
\end{aligned}
\end{equation}
\begin{equation}
\label{ap:eq:s}
\begin{aligned}
s_t(x) &= \nabla \log \rho_t(x)
\end{aligned}
\end{equation}
where the dot denotes the time-derivative: $\dot f= \frac{\mathrm{d}f}{\mathrm{d}t}$ and the expectation is over $\rho(x_0, x_1)$ and $z$ conditional on $I_t = x$.
The probability density $\rho_t(x)$ satisfies the boundary conditions $\rho_{t=0}(x) = \rho_0(x)$ and $\rho_{t=1}(x) = \rho_1(x)$, and solves the transport equation:
\begin{equation}
\label{ap:eq:transportEquation}
\begin{aligned}
\partial_t \rho_t(x) + \nabla \cdot \left( b_t(x) \rho_t(x) \right) = 0. 
\end{aligned}
\end{equation}
Moreover, for every $t$ such that $\gamma_t \neq 0$, the score is given by:
\begin{equation}
\label{ap:eq:s_emphirical}
\begin{aligned}
s_t(x) = -\gamma_t^{-1} \mathbb{E}(z \mid I_t = x).
\end{aligned}
\end{equation}
Finally, the fields $b_t$ and $s_t$ are the unique minimizers of the respective objective functions:
\begin{align}
\label{ap:eq:b_objective}
L_b(\hat{b}) &= \int_0^1 \mathbb{E} \left[ \left|\hat{b}_t(I_t)\right|^2 - 2\dot{I}_t \cdot \hat{b}_t(I_t) \right] dt,
\end{align}
\begin{align}
\label{ap:eq:s_objective}
L_s(\hat{s}) &= \int_0^1 \mathbb{E} \left[ \left|\hat{s}_t(I_t)\right|^2 + 2\gamma_t^{-1} z \cdot \hat{s}_t(I_t) \right] dt,
\end{align}
where $\mathbb{E}$ denotes an expectation over $(x_0, x_1) \sim \rho(x_0, x_1)$ and $z \sim \mathcal{N}(0, \mathrm{Id})$.
\end{theorem}
The objective functions \eqref{ap:eq:b_objective} and \eqref{ap:eq:s_objective} can be readily estimated in practice from samples $(x_0, x_1) \sim \rho(x_0, x_1)$ and $z \sim  \mathcal{N}(0, 1)$, which will enable us to learn approximations for use in a generative model. The transport equation \ref{ap:eq:transportEquation} can be used to derive generative models, as we now show.

\begin{corollary}[Probability flow (ODE) and diffusions (SDE)] \label{ap:corollary:ODE and SDE}
The transport equation \ref{ap:eq:transportEquation} implies that the density of the solutions $X_t$ to the probability flow ODE matches the interpolant density $\rho_t$. The solutions to the probability flow equation:
\begin{equation}
\label{ap:eq:pf_ode}
\dot{X}_t = b_t(X_t)
\end{equation}
satisfy the properties that
\begin{align}
\label{ap:eq:ode_forward}
X_{t=1} \sim \rho_1(x_1) \quad &\text{if} \quad X_{t=0} \sim \rho_0(x_0), \\
\label{ap:eq:ode_backward}
X_{t=0} \sim \rho_0(x_0) \quad &\text{if} \quad X_{t=1} \sim \rho_1(x_1).
\end{align}
In addition, for any choice of a time-dependent diffusion coefficient $\epsilon_t \ge 0$, solutions to the forward SDE
\begin{equation}
\label{ap:eq:sde_forward}
dX_t^F = \left( b_t(X_t^F) + \epsilon_t s_t(X_t^F) \right) dt + \sqrt{2\epsilon_t}dW_t,
\end{equation}
satisfy the property that
\begin{equation}
\label{ap:eq:sde_forward_prop}
X_{t=1}^F \sim \rho_1(x_1) \quad \text{if} \quad X_{t=0}^F \sim \rho_0(x_0).
\end{equation}
And solutions to the backward SDE
\begin{equation}
\label{ap:eq:sde_backward}
dX_t^R = \left( b_t(X_t^R) - \epsilon_t s_t(X_t^R) \right) dt + \sqrt{2\epsilon_t}dW_t,
\end{equation}
satisfy the property that
\begin{equation}
\label{ap:eq:sde_backward_prop}
X_{t=0}^R \sim \rho_0(x_0) \quad \text{if} \quad X_{t=1}^R \sim \rho_1(x_1).
\end{equation}
\end{corollary}
Both deterministic and stochastic generative models were derived within the stochastic interpolant framework.

\section{Justification of Conditional Dependent Coupling and Markovian Inference}
\label{appendix:cond-dep-markov}

In this appendix, we make explicit how the joint distribution in \eqref{eq:conditional_data-dependent_coupling_joint_distribution_training} leads to the Markovian inference process in \eqref{eq:inference_markov}. We also clarify the equivalence between the training-time coupling and the inference-time generative prior.

\subsection{From Full Conditioning to Markovian Inference}

During training, the couplings $\rho^{(k)}(x_0^{(k)}, x_1^{(k)} \mid x^{(K)})$ are defined with explicit conditioning on the full-resolution image $x^{(K)}$ through $x_1^{(k)} \sim \rho_1^{(k)}(x_1^{(k)} | x^{(K)})$, which is implemented by $x_1^{(k)} = D_{K \to k}(x^{(K)})$, and \eqref{eq:data_dependent_coupling}, which is implemented by \eqref{eq:data_dependent_coupling_compute}.
At inference time, however, we no longer have access to $x^{(K)}$ and instead conduct a \emph{generative process} that reverses the data-construction direction, analogous to forward/backward procedures in flow or diffusion models:
\begin{itemize}
    \item In diffusion, one constructs noisy images by repeatedly \emph{adding} noise to a clean image during training, and then learns to \emph{remove} noise step by step at inference.
    \item In flow matching, one defines an interpolant from source to target and trains a vector field so that the backward integration maps from source to data.
\end{itemize}

In our setting, the ``forward'' direction (data-pair construction for training) at stage $k$ is
\[
x_1^{(k)} \sim \rho_1^{(k)}(x_1^{(k)} \mid x^{(K)}),
\quad
x_0^{(k)} \sim \rho_0^{(k)}(x_0^{(k)} \mid x_1^{(k)}),
\]
and the ``backward'' (inference) direction uses the reversed conditional
\[
\hat{x}_1^{(k)} \sim \rho_1^{(k)}(\hat{x}_1^{(k)} \mid \hat{x}_0^{(k)}),
\]
together with the prior
\[
\hat{x}_0^{(k)} \sim \rho_0^{(k)}(\hat{x}_0^{(k)} \mid \hat{x}_1^{(k-1)}),
\]
for $k \ge 2$, and $\hat{x}_0^{(1)} \sim \rho_0^{(1)}(\hat{x}_0^{(1)})$ at the first stage.

By construction, the inference process is \emph{sequential}:
\begin{enumerate}
    \item Sample the initial latent from the simple prior
    \[
    \hat{x}_0^{(1)} \sim \rho_0^{(1)}(\hat{x}_0^{(1)}) = \mathcal{N}(\hat{x}_0^{(1)}; 0, I_{d_1}).
    \]
    \item For stage $k=1$, sample
    \[
    \hat{x}_1^{(1)} \sim \rho_1^{(1)}(\hat{x}_1^{(1)} \mid \hat{x}_0^{(1)}).
    \]
    \item For each refinement stage $k \ge 2$:
    \[
    \hat{x}_0^{(k)} \sim \rho_0^{(k)}(\hat{x}_0^{(k)} \mid \hat{x}_1^{(k-1)}), \qquad
    \hat{x}_1^{(k)} \sim \rho_1^{(k)}(\hat{x}_1^{(k)} \mid \hat{x}_0^{(k)}).
    \]
\end{enumerate}

Since at step $k\ge2$ the pair $(\hat{x}_0^{(k)}, \hat{x}_1^{(k)})$ depends only on $\hat{x}_1^{(k-1)}$, we obtain the Markov factorization:
\begin{equation}
\label{eq:appendix_inference_markov}
\begin{aligned}
&\rho(\hat{x}_0^{(1)}, \hat{x}_1^{(1)}, \hat{x}_0^{(2)}, \dots, \hat{x}_1^{(K-1)}, \hat{x}_0^{(K)}, \hat{x}_1^{(K)}) \\
&\quad= \rho^{(1)}(\hat{x}_0^{(1)}, \hat{x}_1^{(1)}) 
       \prod_{k=2}^K \rho^{(k)}(\hat{x}_0^{(k)}, \hat{x}_1^{(k)} \mid \hat{x}_1^{(k-1)}) \\
&\quad= \rho_0^{(1)}(\hat{x}_0^{(1)}) \rho_1^{(1)}(\hat{x}_1^{(1)} \mid \hat{x}_0^{(1)})
       \prod_{k=2}^K \rho_1^{(k)}(\hat{x}_1^{(k)} \mid \hat{x}_0^{(k)}) \,
                       \rho_0^{(k)}(\hat{x}_0^{(k)} \mid \hat{x}_1^{(k-1)}),
\end{aligned}
\end{equation}
which is exactly \eqref{eq:inference_markov}.

Thus, \eqref{eq:inference_markov} is the natural \emph{reverse} generative process induced by the training-time couplings:
\begin{itemize}
    \item $\rho_0^{(k)}(x_0^{(k)} \mid x_1^{(k)})$ is used in the forward (data-pair construction) direction.
    \item $\rho_1^{(k)}(\hat{x}_1^{(k)} \mid \hat{x}_0^{(k)})$ is its reverse counterpart at inference.
\end{itemize}
The unified DiT $b^\theta$ parameterizes $\rho_1^{(k)}(\hat{x}_1^{(k)} \mid \hat{x}_0^{(k)})$ for all $k$, analogous to how a single score network or vector field is shared across timesteps in diffusion and flow matching models.

\subsection{Equivalence of Training and Inference Priors at Refinement Stages}

We now show that the refinement-stage prior \eqref{eq:inf_prior} used at inference, which is implemented by:
\[
\hat{x}_0^{(k)} = U_k(\hat{x}_1^{(k-1)}) + \sigma_k \zeta_k
\]
conducts the same conditional distribution as $\rho_0^{(k)}(x_0^{(k)} \mid x_1^{(k)})$ in \eqref{eq:data_dependent_coupling}, which is implemented by \eqref{eq:data_dependent_coupling_compute}, once we account for the downsampling hierarchy.

First, recall the \emph{transitivity} of the downsampling operators defined in Sec. \ref{sec:notation}:
\[
D_{k+1 \to k-1} = D_k \circ D_{k+1},
\qquad
D_{K \to k-1} = D_k \circ D_{K \to k}.
\]
By definition of the stage-$k$ ground-truths,
\[
x_1^{(k)} = D_{K \to k}(x^{(K)})
\quad\Rightarrow\quad
D_k(x_1^{(k)}) = D_k(D_{K \to k}(x^{(K)})) = D_{K \to k-1}(x^{(K)}) = x_1^{(k-1)}.
\]
Hence, for the \emph{data-construction process} used in training,
\[
m_k(x_1^{(k)}) 
= U_k(D_k(x_1^{(k)}))
= U_k(x_1^{(k-1)}).
\]
Plugging this into \eqref{eq:data_dependent_coupling_compute} gives
\[
x_0^{(k)} 
= U_k(x_1^{(k-1)}) + \sigma_k \zeta_k,
\quad \zeta_k \sim \mathcal{N}(0, I_{d_k}),
\]
so that
\[
x_0^{(k)} \sim \mathcal{N}\bigl(x_0^{(k)}; U_k(x_1^{(k-1)}), \sigma_k^2 I_{d_k}\bigr).
\]

At inference time, we define the refinement prior by:
\[
\hat{x}_0^{(k)} 
= U_k(\hat{x}_1^{(k-1)}) + \sigma_k \zeta_k,
\quad \zeta_k \sim \mathcal{N}(0, I_{d_k}),
\]
which induces to \eqref{eq:inf_prior}:
\[
\rho_0^{(k)}(\hat{x}_0^{(k)} \mid \hat{x}_1^{(k-1)}) 
= \mathcal{N}\bigl(\hat{x}_0^{(k)}; U_k(\hat{x}_1^{(k-1)}), \sigma_k^2 I_{d_k}\bigr).
\]
Comparing the two expressions, we see that:
\begin{itemize}
    \item During \emph{training}, when we construct pairs $(x_0^{(k)}, x_1^{(k)})$ from data, the conditional distribution
    \[
    x_0^{(k)} \mid x_1^{(k)} \sim \mathcal{N}\bigl(U_k(x_1^{(k-1)}), \sigma_k^2 I_{d_k}\bigr),
    \]
    with $x_1^{(k-1)} = D_k(x_1^{(k)})$.
    \item During \emph{inference}, for any given input $\hat{x}_1^{(k-1)}$ (which is the output of the previous stage), we use
    \[
    \hat{x}_0^{(k)} \mid \hat{x}_1^{(k-1)} 
    \sim \mathcal{N}\bigl(U_k(\hat{x}_1^{(k-1)}), \sigma_k^2 I_{d_k}\bigr).
    \]
\end{itemize}
Therefore, the distribution in \eqref{eq:inf_prior} implemented by
\[
\hat{x}_0^{(k)} = U_k(\hat{x}_1^{(k-1)}) + \sigma_k \zeta_k
\]
is \emph{identical in form} to the training conditional distribution in \eqref{eq:data_dependent_coupling}
\[
x_0^{(k)} \sim \rho_0^{(k)}(x_0^{(k)} \mid x_1^{(k)})
\]
once we use the relation
\[
x_1^{(k-1)} = D_k(x_1^{(k)}).
\]
In other words, the refinement prior at stage $k$ during inference is exactly the same Gaussian conditional as the one used to generate $x_0^{(k)}$ from $x_1^{(k)}$ during data-pair construction, with $\hat{x}_1^{(k-1)}$ playing the role of $x_1^{(k-1)}$. This is precisely the sense: $\rho_1^{(k)}(\hat{x}_1^{(k)} \mid \hat{x}_0^{(k)})$ is the \emph{reverse} counterpart of $\rho_0^{(k)}(x_0^{(k)} \mid x_1^{(k)})$,
and why the manually constructed joint distribution in \eqref{eq:conditional_data-dependent_coupling_joint_distribution_training} leads to the Markovian generative process in \eqref{eq:inference_markov}.

\section{Detailed Implementation of Resolution Embedding}
\label{appendix:resolution_embedding}

$e_k$ is fused with the timestep embedding by the cross-attention mechanism before being fed into the model. The detailed fuse procedure is illustrated below.
For each stage $k$ and rescaled timestep $\tau \in [0,1]$,
We follow the standard DiT sinusoidal embedding followed by an MLP:
\begin{equation}
    \gamma_t = \Gamma(\tau) \in \mathbb{R}^{d_{\text{model}}}.
\end{equation}
Given the absolute resolution $r_k$ of stage $k$, we construct the resolution embedding using an analogous sinusoidal embedding followed by an MLP:
\begin{equation}
    e_k = E(r_k) \in \mathbb{R}^{d_{\text{model}}}.
\end{equation}
Both $\gamma_t$ and $e_k$ are 1D vectors in the same model space $\mathbb{R}^{d_{\text{model}}}$.

We implement the fusion by cross-attention, where the timestep embedding $\gamma_t$ is used as the query and the resolution embedding $e_k$ is used as the key/value.
Let $d_h$ be the head dimension. We first define linear projections
\begin{equation}
    q_t = W_q \gamma_t \in \mathbb{R}^{d_h}, \quad
    k_k = W_k e_k \in \mathbb{R}^{d_h}, \quad
    v_k = W_v e_k \in \mathbb{R}^{d_h},
\end{equation}
where $W_q, W_k, W_v \in \mathbb{R}^{d_h \times d_{\text{model}}}$.

The cross-attention between the timestep and the resolution embeddings is then
\begin{equation}
    \alpha_{k,t} = \operatorname{softmax}\!\left(
        \frac{q_t^\top k_k}{\sqrt{d_h}}
    \right) \in \mathbb{R},
\end{equation}
\begin{equation}
    z_{k,t} = W_o \bigl( \alpha_{k,t} \, v_k \bigr)
    \in \mathbb{R}^{d_{\text{model}}}, \quad
    W_o \in \mathbb{R}^{d_{\text{model}} \times d_h}.
\end{equation}

We then obtain a fused conditioning vector via a residual connection:
\begin{equation}
    c_{k,t} = \gamma_t + z_{k,t} \in \mathbb{R}^{d_{\text{model}}}.
\end{equation}
The vector $c_{k,t}$ can be interpreted as a timestep embedding conditioned on the current stage resolution, and is used as the global conditioning vector for the DiT.

Let $X^{(k)}_0 \in \mathbb{R}^{N_k \times d_{\text{model}}}$ denote the sequence of patch tokens at stage $k$ after patch embedding of $I^{(k)}_\tau$. The unified DiT is conditioned on the fused vector $c_{k,t}$.
We use $c_{k,t}$ as a global conditioning signal for all DiT blocks via AdaLN modulation~\cite{peebles2023scalable}, following DiT~\cite{peebles2023scalable}, SiT~\cite{ma2024sit} and PixelFlow~\cite{chen2025pixelflow}:
\begin{equation}
    X^{\ell+1} = X^{\ell} +
    F^{(\ell)}\bigl(\operatorname{LN}(X^{\ell});\, c_{k,t}\bigr),
    \quad \ell = 1, \dots, L,
\end{equation}
where each block $F^{(\ell)}$ is a DiT block whose parameters are produced from $c_{k,t}$ by AdaLN.

Conceptually, the same DiT parameters are reused for all stages, while the pair $(\tau, r_k)$ is compressed into the single fused conditioner $c_{k,t}$, which tells the model both \emph{where in time} and \emph{at which resolution} it is operating.

\section{Proof}
\subsection{Proof of Theorem~\ref{theorem:lowerTransportCost}} \label{pf:transportcost}
We provide a formal proof that a cascaded image generation model using Flow Matching with \textbf{conditional dependent coupling} has a significantly lower transportation cost than a direct, single-stage model that generates a high-resolution image from Gaussian noise, as shown in Theorem \ref{theorem:lowerTransportCost}. The proof is based on an orthogonal decomposition of the image signal energy, which reveals that the cost difference is primarily driven by the dimensionality of the initial noise space.

\subsubsection{Problem Formulation and Definitions}

Our goal is to prove that the transportation cost of a direct Flow Matching model ($L_A$) is greater than the total transport cost of a multi-stage cascaded model with conditional dependent couplings ($L_B$).

\begin{proposition}[Transportation Cost Bound]
From Definition~\ref{definition:Transport cost}, the transportation cost of a probability flow is bounded by the integral of the expected squared norm of the interpolant's time derivative. For an interpolant path from a source distribution $\rho_0(x_0)$ to a target distribution $\rho_1(x_1)$, we denote this bound as $L$:
\begin{equation}
\label{eq:cost_bound}
\text{Cost} = \mathbb{E}_{x_0 \sim \rho_0}\left[ |X_{t=1}(x_0) - x_0|^2 \right] 
\leq \int_0^1 \mathbb{E}\left[ |\dot{I}_t|^2 \right] dt := L.
\end{equation}
For simplicity in the proof, we analyze the case where the interpolant is $I_t = (1-t)x_0 + t x_1$, such that its derivative is $\dot{I}_t = x_1 - x_0$. The logic holds for general interpolants. In this case, the cost bound reduces to~\cite{lipman2024flow, albergo2023stochasticcouplings}:
\begin{equation}
L = \mathbb{E}[|x_1 - x_0|^2].
\end{equation}
\end{proposition}

\subsubsubsection{Model A: Direct Generation}
This model generates a high-resolution image $x^{(K)} \in \mathbb{R}^{d_K}$ in a single stage.
\begin{itemize}
    \item \textbf{Source Distribution} $\rho_0(x_0)$: $x_0 \sim \mathcal{N}(0, \sigma^2 I_{d_K})$.
    \item \textbf{Target Distribution} $\rho_1(x_1)$: The high-resolution data distribution, $x_1 \sim x^{(K)}$.
    \item \textbf{Cost} ($L_A$): The transport cost from $x_0$ to $x_1$.
\end{itemize}

\subsubsubsection{Model B: Cascaded Generation}
This model generates the image in $K$ stages, from resolution $d_1$ to $d_K$.
\begin{itemize}
    \item \textbf{Stage 1:} Generates a low-resolution image $x_1^{(1)} \in \mathbb{R}^{d_1}$.
        \begin{itemize}
            \item Source: $x_0^{(1)} \sim \mathcal{N}(0, \sigma^2 I_{d_1})$.
            \item Target: $x_1^{(1)} \sim \rho_1^{(1)}(x_1^{(1)} | x^{(K)})$.
            \item Cost: $L_1$.
        \end{itemize}
    \item \textbf{Stages $k \in [2, K]$:} Refines a coarse image to a higher resolution image $x_1^{(k)} \in \mathbb{R}^{d_k}$.
        \begin{itemize}
            \item Target: $x_1^{(k)} \sim \rho_1^{(k)}(x_1^{(k)} | x^{(K)})$.
            \item Source: $x_0^{(k)} \sim \rho_0^{(k)}(x_0^{(k)} | x_1^{(k)})$ is defined by a deterministic \textbf{conditional dependent coupling} with a small stage-dependent Gaussian noise: $x_0^{(k)} = U_k(D_k(x_1^{(k)})) + \sigma_k \zeta_k$, as shown in \eqref{eq:data_dependent_coupling_compute}.
            \item Cost: $L_k$.
        \end{itemize}
    \item \textbf{Total Cost} ($L_B$): The sum of the costs of all stages, $L_B = \sum_{k=1}^{K} L_k$.
\end{itemize}

We seek to prove that $L_A > L_B$.

\paragraph{Note on the use of ``$\approx$''.}
Throughout this proof, the symbol ``$\approx$'' (e.g.\ in \eqref{eq:LK} below and the cost summations) is \emph{not} a generic numerical approximation: it indicates equality up to higher-order terms that vanish under Assumption~\ref{assumption:orthogonality} (orthogonality of details across levels). Once Assumption~\ref{assumption:orthogonality} holds, every step that uses ``$\approx$'' is exact; the conclusion $L_A > L_B$ is therefore an exact inequality \emph{conditional on} this assumption rather than an approximate bound. The same convention applies to Theorem~\ref{theorem:lowerInfer} and Appendix~\S\ref{pf:fasterInfer}, where the additional Assumptions~G.4 and~G.5 (per-evaluation cost linearity and NFE--transport-cost proportionality) play the analogous role.

\subsubsection{Orthogonal Decomposition of Image Energy}

We formalize the concept that an image can be represented as a base layer plus a sum of details.

\begin{definition}[Multiresolution Projections]
Let the vector space be that of the highest resolution, $\mathbb{R}^{d_K}$. We define an embedding operator $U_{k \to K}: \mathbb{R}^{d_k} \to \mathbb{R}^{d_K}$ that upscales a level-$k$ image to the full resolution space. We then define a projection operator $D_{K \to k}: \mathbb{R}^{d_K} \to \mathbb{R}^{d_k}$ that projects a high-resolution image onto the subspace of images with resolution $k$. These two definitions are inherited from the previous section. For an image $x \in \mathbb{R}^{d_K}$, let $\tilde{x}^{(k)} = U_{k \to K}(D_{K \to k}(x))$. We also define the null image $\tilde{x}^{(0)} = \mathbf{0}$. Let the \textbf{detail vector} at level $k$ be $\tilde{\epsilon}^{(k)} = \tilde{x}^{(k)} - \tilde{x}^{(k-1)}$. 
\end{definition}

\begin{assumption}[Orthogonality of Details] \label{assumption:orthogonality}
The detail vector $\tilde{\epsilon}^{(k)}$, which contains the frequency content for level $k$, is orthogonal to the detail vectors of all other levels. 
\begin{equation}
\mathbb{E}[(\tilde{\epsilon}^{(k)})^T \tilde{\epsilon}^{(j)}] = 0 \quad \text{for} \quad k \neq j
\end{equation}
\end{assumption}
This also reflects the nature of the process, in which each stage responsible for adding details is conditionally independent, as illustrated in \eqref{eq:conditional_data-dependent_coupling_joint_distribution_training}.
Under Assumption \ref{assumption:orthogonality}, the full image can be written as a telescoping sum:
\begin{equation}
x = \tilde{x}^{(K)} = \sum_{k=1}^{K} (\tilde{x}^{(k)} - \tilde{x}^{(k-1)}) = \sum_{k=1}^{K} \tilde{\epsilon}^{(k)}
\end{equation}

\begin{lemma}[Energy Decomposition] \label{lemma:energy_decomposition}
Under Assumption \ref{assumption:orthogonality}, the expected energy of an image is the sum of the expected energies of its orthogonal detail components.
\begin{equation}
\mathbb{E}[|x|^2] = \mathbb{E}\left[\left|\sum_{k=1}^{K} \tilde{\epsilon}^{(k)}\right|^2\right] = \sum_{k=1}^{K} \mathbb{E}[|\tilde{\epsilon}^{(k)}|^2]
\end{equation}
\end{lemma}
\begin{proof}
This follows directly from the linearity of expectation and the orthogonality assumption, as all cross-terms $\mathbb{E}[(\tilde{\epsilon}^{(k)})^T \tilde{\epsilon}^{(j)}]$ for $k \neq j$ vanish.
\end{proof}

\subsubsection{Analysis of Transportation Costs}

\subsubsubsection{\textbf{Cost of the Direct Model ($L_A$)}.}
The cost is $L_A = \mathbb{E}[|x_1 - x_0|^2]$, where $x_1 \sim x^{(K)}$ and $x_0 \sim \mathcal{N}(0, \sigma^2 I_{d_K})$. Due to independence and $\mathbb{E}[x_0] = 0$:
\begin{equation}
\begin{aligned}
L_A &= \mathbb{E}[|x_1|^2] - 2\mathbb{E}[x_1^T x_0] + \mathbb{E}[|x_0|^2] \\
    &= \mathbb{E}[|x_1|^2] + \mathbb{E}[|x_0|^2]
\end{aligned}
\end{equation}
Using Lemma \ref{lemma:energy_decomposition} and the variance of the noise, $\mathbb{E}[|x_0|^2] = d_K \sigma^2$:
\begin{equation}
L_A = \left( \sum_{k=1}^{K} \mathbb{E}[|\tilde{\epsilon}^{(k)}|^2] \right) + d_K \sigma^2
\end{equation}

\subsubsubsection{\textbf{Cost of the Cascaded Model with conditional dependent couplings ($L_B$)}.}
The total cost is $L_B = \sum_{k=1}^{K} L_k$.
\begin{itemize}
    \item \textbf{Stage 1:} $L_1 = \mathbb{E}[|x_1^{(1)} - x_0^{(1)}|^2] = \mathbb{E}[|x_1^{(1)}|^2] + \mathbb{E}[|x_0^{(1)}|^2]$. The data $x_1^{(1)}$ represents the coarsest level of detail, corresponding to $\tilde{\epsilon}^{(1)}$. The noise variance is $\mathbb{E}[|x_0^{(1)}|^2] = d_1 \sigma^2$. Thus,
    \begin{equation}
    L_1 \approx \mathbb{E}[|\tilde{\epsilon}^{(1)}|^2] + d_1 \sigma^2
    \end{equation}
    \item \textbf{Stages $k \in [2, K]$:} The cost is $L_k = \mathbb{E}[|x_1^{(k)} - x_0^{(k)}|^2]$. Substituting the data-dependent coupling $x_0^{(k)} = U_k(D_k(x_1^{(k)})) + \sigma_k \zeta_k$, we get:
    \begin{equation}
    \begin{aligned}
    L_k &= \mathbb{E}[|x_1^{(k)} - U_k(D_k(x_1^{(k)})) - \sigma_k \zeta_k|^2] \\
    &= \mathbb{E}[|x_1^{(k)} - U_k(D_k(x_1^{(k)}))|^2] - 2 \ \mathbb{E}[\sigma_k \zeta_k^T \cdot (x_1^{(k)} - U_k(D_k(x_1^{(k)})))] + \mathbb{E}[|\sigma_k \zeta_k|^2] \\
    &= \mathbb{E}[|x_1^{(k)} - U_k(D_k(x_1^{(k)}))|^2] - 2 \ \mathbb{E}[\sigma_k \zeta_k] \ \mathbb{E}[(x_1^{(k)} - U_k(D_k(x_1^{(k)})))] + \mathbb{E}[|\sigma_k \zeta_k|^2] \\
    &= \mathbb{E}[|x_1^{(k)} - U_k(D_k(x_1^{(k)}))|^2] + \mathbb{E}[|\sigma_k \zeta_k|^2],
    \end{aligned}
    \end{equation}
    where independence: $\sigma_k \zeta_k \perp\!\!\!\perp (x_1^{(k)} - U_k(D_k(x_1^{(k)})))$ and $\mathbb{E}[\sigma_k \zeta_k] = 0$.
    $\mathbb{E}[|x_1^{(k)} - U_k(D_k(x_1^{(k)}))|^2]$ is the expected energy of the details required to go from resolution $k-1$ to $k$, which corresponds precisely to the energy of the detail vector $\tilde{\epsilon}^{(k)}$: $\mathbb{E}[|\tilde{\epsilon}^{(k)}|^2]$, and the variance of the noise, $\mathbb{E}[|\sigma_k \zeta_k|^2] = d_k \sigma_k^2$:
    \begin{equation}
    \label{eq:LK}
    L_k \approx \mathbb{E}[|\tilde{\epsilon}^{(k)}|^2] + d_k \sigma_k^2
    \end{equation}
\end{itemize}
Summing the costs for all stages of Model B:
\begin{equation}
\begin{aligned}
L_B &= L_1 + \sum_{k=2}^{K} L_k \\
    &\approx (\mathbb{E}[|\tilde{\epsilon}^{(1)}|^2] + d_1 \sigma^2) + \sum_{k=2}^{K} \left( \mathbb{E}[|\tilde{\epsilon}^{(k)}|^2] + d_k \sigma_k^2 \right) \\
    &= \left( \sum_{k=1}^{K} \mathbb{E}[|\tilde{\epsilon}^{(k)}|^2] \right) + d_1 \sigma^2 + \left( \sum_{k=2}^{K} d_k \sigma_k^2 \right),
\end{aligned}
\end{equation}
Here, \( d_k = 2^{k-1} d_1 \) for \( 1 \leq k \leq K \), reflecting the pyramid structure that aligns with the conditional dependent coupling design. As discussed in the previous section, we define \( \sigma_k \) to decrease as \( k \) increases, which matches the intuition that as the stage \( k \) progresses, the requirements for diversity and stochasticity gradually diminish, as shown in \eqref{eq:sigma_pyramid}, where we specify $\gamma =2$ for simplicity:
\begin{equation}
\begin{aligned}
\sigma_k = 2^{-(k-1)} \sigma \quad \text{for } 2 \le k \le K.
\end{aligned}
\end{equation}

\subsubsection{Main Proof and Conclusion}

\begin{proof}
From the detailed analysis in the previous subsections, we have derived the transportation costs for the direct and cascaded models:

\begin{equation}
\label{eq:transportCost_LA}
L_A = \left( \sum_{k=1}^{K} \mathbb{E}[|\tilde{\epsilon}^{(k)}|^2] \right) + d_K \sigma^2,
\end{equation}
where $d_K$ is the dimensionality of the high-resolution image space.

For the cascaded model with $K$ stages, the total cost is the sum of the costs of all stages:
\begin{equation}
\label{eq:transportCost_LB}
L_B = \left( \sum_{k=1}^{K} \mathbb{E}[|\tilde{\epsilon}^{(k)}|^2] \right) + d_1 \sigma^2 + \sum_{k=2}^{K} d_k \sigma_k^2.
\end{equation}

The term $\sum_{k=1}^{K} \mathbb{E}[|\tilde{\epsilon}^{(k)}|^2]$, representing the total signal energy (see Lemma \ref{lemma:energy_decomposition}), appears in both $L_A$ and $L_B$ and thus cancels when we consider the difference:
\begin{equation}
L_A - L_B = d_K \sigma^2 - \left( d_1 \sigma^2 + \sum_{k=2}^{K} d_k \sigma_k^2 \right).
\end{equation}

To evaluate the sign of $L_A - L_B$, we use the definitions of $d_k$ and $\sigma_k$ from the cascaded model with conditional dependent coupling:
\[
d_k = 2^{k-1} d_1, \quad \sigma_k = 2^{-(k-1)} \sigma \quad \text{for } 2 \le k \le K.
\]
Thus, the second term can be expanded as:
\begin{equation}
\sum_{k=2}^{K} d_k \sigma_k^2 = \sum_{k=2}^{K} \left( 2^{k-1} d_1 \right) \left( 2^{-(k-1)} \sigma \right)^2 = d_1 \sigma^2 \sum_{k=2}^{K} 2^{-(k-1)}.
\end{equation}

We have:
\begin{equation}
L_A - L_B = \sigma^2 \left( d_K - d_1 - d_1 \sum_{k=2}^{K} 2^{-(k-1)} \right).
\end{equation}

Since $d_K = 2^{K-1} d_1$, we obtain:
\begin{equation}
L_A - L_B = \sigma^2 d_1 \left( 2^{K-1} - 1 - \sum_{k=2}^{K} 2^{-(k-1)} \right).
\end{equation}

Specifically, $L_A - L_B = 0$ when $K=1$.
Observe that $2^{K-1} - 1 \gg \sum_{k=2}^{K} 2^{-(k-1)}$ for all $K \ge 2$, because the left-hand side grows exponentially while the right-hand side is a bounded geometric series:
\[
\sum_{k=2}^{K} 2^{-(k-1)} < \sum_{m=1}^{\infty} 2^{-m} = 1.
\]
Therefore, the term in parentheses is strictly positive:
\[
2^{K-1} - 1 - \sum_{k=2}^{K} 2^{-(k-1)} > 0.
\]

Since $\sigma^2 > 0$ and $d_1 > 0$, we conclude:
\begin{equation}
L_A - L_B > 0 \implies L_A > L_B.
\end{equation}

This completes the proof. The main factor leading to $L_A > L_B$ is that the direct model pays the cost of transporting a high-dimensional Gaussian noise vector of dimension $d_K$, while the cascaded model introduces noise primarily in the low-dimensional early stages and relies on conditional dependent couplings to incrementally add fine-grained details. This multi-resolution structure results in a strictly lower total transportation cost.
\end{proof}
\qed

\subsection{Proof of Theorem~\ref{theorem:lowerInfer}} \label{pf:fasterInfer}

We provide a formal proof that the cascaded Flow Matching model with the \textbf{conditional dependent coupling} strategy has strictly smaller expected inference time than the naive single-stage model for high-resolution image generation, as stated in Theorem~\ref{theorem:lowerInfer}. The argument follows three steps: (i) relate the expected inference-time functional $\iota(\cdot)$ to the Number of Function Evaluations (NFE) and the per-evaluation cost, (ii) connect NFE to a transportation cost, and (iii) compare the resulting computational ``loads’’ of the direct and cascaded schemes by separating signal and noise contributions. 

\subsubsection{Problem Formulation and Assumptions}

Recall the theorem’s notation:
\[
T_A \;=\; \mathbb{E}\!\left[\iota\!\big(x_0 \to \hat{x}_1\big)\right],
\qquad
T_B \;=\; \sum_{k=1}^{K} \mathbb{E}\!\left[\iota\!\big(x_0^{(k)} \to \hat{x}_1^{(k)}\big)\right].
\]
The direct model draws $x_0 \sim \mathcal{N}(0,\sigma^2 I_{d_K})$ and produces $\hat{x}_1\in\mathbb{R}^{d_K}$ in one stage; the cascaded model comprises $K$ stages, with stage $k$ transporting $x_0^{(k)} \sim \mathcal{N}(x_0^{(k)}; m_k(x_1^{(k)}), \sigma_k^2 I_{d_k})$ to $\hat{x}_1^{(k)} \in \mathbb{R}^{d_k}$, where $d_1<\cdots<d_K$, as state in \eqref{eq:data_dependent_coupling}.

\begin{assumption}[Per-Evaluation Cost; G.4] \label{ass:perevalcost}
For ODE-based generators (e.g., probability flow ODEs) under a fixed backbone and solver tolerance, the per-function-evaluation cost scales linearly with the ambient dimensionality $d$. Empirical validation of this assumption on our model is reported in Appendix~\S\ref{appendix:efficiency_validation}.
\end{assumption}

\paragraph{Per-evaluation cost model.}
Under Assumption~\ref{ass:perevalcost}, the expected inference-time cost is the product of the NFE and the cost per function evaluation, and the per-evaluation cost scales linearly with the ambient dimensionality:
\[
\mathbb{E}\!\left[\iota\!\big(x_0 \to \hat{x}_1\big)\right] \;=\; \alpha\, N_A\, d_K,
\qquad
\mathbb{E}\!\left[\iota\!\big(x_0^{(k)} \to \hat{x}_1^{(k)}\big)\right] \;=\; \alpha\, N_k\, d_k,
\]
for some constant $\alpha>0$. Hence
\[
T_A \;=\; \alpha\, N_A\, d_K,
\qquad
T_B \;=\; \alpha \sum_{k=1}^{K} N_k\, d_k .
\]
To prove $T_A > T_B$ it suffices to show
\begin{equation}
\sum_{k=1}^{K} N_k\, d_k \;<\; N_A\, d_K .
\label{eq:goal-NFE}
\end{equation}

\begin{assumption}[NFE vs. Transportation Cost; G.5] \label{ass:nfe_cost}
The NFE required to solve a probability flow ODE to a given accuracy is proportional to the transportation cost $L = \mathbb{E}[|x_1 - x_0|^2]$. This cost serves as a proxy for the complexity of the vector field.
\begin{equation}
N_i = C \cdot L_i
\end{equation}
where $C$ is a constant dependent on the ODE solver and desired precision. We empirically validate this proportionality, including its solver-stability across different choices of intermediate states, in Appendix~\S\ref{appendix:efficiency_validation}.
\end{assumption}

Under this assumption, \eqref{eq:goal-NFE} is equivalent to
\begin{equation}
\sum_{k=1}^{K} L_k\, d_k \;<\; L_A\, d_K .
\label{eq:goal-load}
\end{equation}
Define the \emph{computational load} of a transport as $\text{Load}:=L\times d$. Then
\[
\text{Load}_A \;=\; L_A d_K, 
\qquad 
\text{Load}_B \;=\; \sum_{k=1}^{K} L_k d_k,
\]
and proving \eqref{eq:goal-load} is equivalent to showing $\text{Load}_B < \text{Load}_A$.

\subsubsection{Analysis of Computational Load}

We use the transportation-cost decompositions established in Section~\ref{pf:transportcost}. Writing $\tilde{\epsilon}^{(k)}$ for the signal component associated with scale/stage $k$:

\paragraph{Direct model (single stage).}
The load is the product of its transportation cost $L_A$(\eqref{eq:transportCost_LA}) and the dimensionality of the final image $d_K$.
\begin{equation}
\label{eq:load-A}
\begin{aligned}
\text{Load}_A &= L_A \cdot d_K \\ &= \left( \left( \sum_{j=1}^{K} \mathbb{E}[|\tilde{\epsilon}^{(j)}|^2] \right) + d_K \sigma^2 \right) d_K \\
&= d_K \sum_{j=1}^{K} \mathbb{E}[|\tilde{\epsilon}^{(j)}|^2] + d_K^2 \sigma^2 
\end{aligned}
\end{equation}

\paragraph{Cascaded model (multi-stage).}
The total load is the sum of the loads of each stage, where the load for stage $k$ is $L_k \cdot d_k$, where $L_k$ is shown in \eqref{eq:LK}.
For stage $1$, $x_0^{(1)} \sim \mathcal{N}(0,\sigma^2 I_{d_1})$; for stages $k\ge 2$, the conditional dependent coupling strategy injects $x_0^{(k)} \sim \mathcal{N}(x_0^{(k)}; m_k(x_1^{(k)}), \sigma_k^2 I_{d_k})$. Hence
\begin{equation}
\begin{aligned}
\text{Load}_B 
&= \sum_{k=1}^{K} L_k d_k
= L_1 d_1 + \sum_{k=2}^{K} L_k d_k \\
&= (\mathbb{E}[|\tilde{\epsilon}^{(1)}|^2] + d_1 \sigma^2) d_1 + \sum_{k=2}^{K} (\mathbb{E}[|\tilde{\epsilon}^{(k)}|^2] + d_k \sigma_k^2) d_k \\
&= \mathbb{E}[|\tilde{\epsilon}^{(1)}|^2]d_1 + d_1^2 \sigma^2 + \sum_{k=2}^{K} \mathbb{E}[|\tilde{\epsilon}^{(k)}|^2]d_k + \sum_{k=2}^{K} d_k^2 \sigma_k^2 \\
&= \left(\sum_{k=1}^{K} \mathbb{E}[|\tilde{\epsilon}^{(k)}|^2] d_k \right) + \left(d_1^2 \sigma^2 + \sum_{k=2}^{K} d_k^2 \sigma_k^2 \right)
\label{eq:load-B}
\end{aligned}
\end{equation}

\subsubsection{Main Proof and Conclusion}

We compare \eqref{eq:load-A} and \eqref{eq:load-B} by splitting into signal and noise parts.

\paragraph{(1) Signal component.}
From \eqref{eq:load-A} and \eqref{eq:load-B},
\[
\text{Load}_{A,\text{signal}} \;=\; d_K \sum_{k=1}^{K} \mathbb{E}\!\big[|\tilde{\epsilon}^{(k)}|^2\big],
\qquad
\text{Load}_{B,\text{signal}} \;=\; \sum_{k=1}^{K} \mathbb{E}\!\big[|\tilde{\epsilon}^{(k)}|^2\big]\, d_k .
\]
Since $d_k < d_K$ for all $1 \leq k < K$, each term in the summation for $\text{Load}_{B, \text{signal}}$ is smaller than(or equal to, when $k=K$) the corresponding term for $\text{Load}_{A, \text{signal}}$. Therefore:
\begin{equation}
\label{eq:signal-ineq}
\sum_{k=1}^{K} \mathbb{E}[|\tilde{\epsilon}^{(k)}|^2] d_k < d_K \sum_{k=1}^{K} \mathbb{E}[|\tilde{\epsilon}^{(k)}|^2]
\end{equation}
The computational work related to generating the image content is strictly lower in the cascaded model because most detail components $\tilde{\epsilon}^{(k)}$ are processed at much lower dimensionalities $d_k$.

\paragraph{(2) Noise component (coupling strategy).}
Adopt the conditional dependent coupling in which spatial resolutions double by stage and the injected noise variance is inversely scaled, as shown in \eqref{eq:sigma_pyramid}:
\[
d_k = 2^{k-1} d_1, \quad \sigma_k = 2^{-(k-1)} \sigma \quad \text{for } 2 \le k \le K.
\]
Then
\begin{equation}
\begin{aligned}
\text{Load}_{B,\text{noise}}
&= d_1^2 \sigma^2 + \sum_{k=2}^{K} d_k^2 \sigma_k^2 \\
&= d_1^2 \sigma^2 + \sum_{k=2}^{K} (2^{k-1} d_1)^2 (2^{-(k-1)} \sigma)^2 \\
&= d_1^2 \sigma^2 + \sum_{k=2}^{K} (2^{2(k-1)} d_1^2) (2^{-2(k-1)} \sigma^2) \\
&= d_1^2 \sigma^2 + \sum_{k=2}^{K} d_1^2 \sigma^2 \\
&= d_1^2 \sigma^2 + (K-1)d_1^2 \sigma^2 = K d_1^2 \sigma^2,
\end{aligned}
\end{equation}
whereas for the direct model
\begin{equation}
\text{Load}_{A, \text{noise}} = d_K^2 \sigma^2 = (2^{K-1}d_1)^2 \sigma^2 = 2^{2(K-1)} d_1^2 \sigma^2
\end{equation}
For $K=1$, both loads coincide; for any practical cascade with $K\ge 2$,
\begin{equation}
\text{Load}_{B,\text{noise}} \;=\; K\, d_1^2 \sigma^2 
\;<\; 2^{2(K-1)} d_1^2 \sigma^2 \;=\; \text{Load}_{A,\text{noise}} .
\label{eq:noise-ineq}
\end{equation}
Thus the noise-related work is exponentially smaller in the cascaded scheme.

\paragraph{Combining (1) and (2).}
Adding \eqref{eq:signal-ineq} and \eqref{eq:noise-ineq} yields
\[
\text{Load}_B \;<\; \text{Load}_A
\quad\Longleftrightarrow\quad
\sum_{k=1}^{K} L_k d_k \;<\; L_A d_K .
\]
Using the NFE--transportation-cost proportionality and the per-evaluation cost model, common multiplicative constants cancel, giving
\[
T_B \;=\; \sum_{k=1}^{K} \mathbb{E}\!\left[\iota\!\big(x_0^{(k)} \to \hat{x}_1^{(k)}\big)\right]
\;<\;
\mathbb{E}\!\left[\iota\!\big(x_0 \to \hat{x}_1\big)\right]
\;=\; T_A .
\]
For $K=1$ the two procedures coincide, while for every $K\ge 2$ the inequality is strict. This completes the proof of Theorem~\ref{theorem:lowerInfer}.

\section{Empirical Validation of the Inference-Time Assumptions and Stage-Wise Efficiency} \label{appendix:efficiency_validation}

Theorem~\ref{theorem:lowerInfer} relies on two assumptions formalized in Appendix~\S\ref{pf:fasterInfer}: (G.4) per-function-evaluation cost scales linearly with the ambient dimensionality, and (G.5) the NFE required to reach a fixed solver-error level is proportional to the transportation cost $C(t_1,t_2)=\mathbb{E}[|x_1-x_0|^2]$. We empirically validate both, and additionally isolate the wall-clock contribution of conditional dependent coupling by comparing against a single-stage pixel baseline that uses our backbone, training pipeline, and ODE solver.

\subsection{Per-Evaluation Cost Scales Linearly with the Ambient Dimensionality}
\label{appendix:per_eval_cost}

Under our experimental setting (single H100 GPU, batch size $1$, FP16, 200-call rolling average), Table~\ref{tab:per_eval_cost} reports the per-function-evaluation cost of the unified DiT across the resolutions used by CDC-FM. The measured time grows approximately $4\times$ when the spatial resolution doubles, i.e.\ approximately linearly with the generation dimension $d_k = r_k^2$. This empirically supports Assumption~\ref{ass:perevalcost} (G.4).

\begin{table}[t]
\centering
\small
\caption{Per-function-evaluation cost across resolutions (single forward pass of the unified DiT, in milliseconds, averaged over 200 calls on a single H100 GPU at batch size $1$). The cost scales approximately linearly with the generation dimension $d_k$, which empirically supports Assumption~\ref{ass:perevalcost} (G.4) used in the proof of Theorem~\ref{theorem:lowerInfer}.}
\label{tab:per_eval_cost}
\begin{tabular}{l c c c}
\toprule
Resolution $d_k$ & $32^2$ & $64^2$ & $128^2$ \\
\midrule
Time / function-eval (ms) & $5.6$ & $22.0$ & $87.4$ \\
Normalized to $d_1$       & $1.0\times$ & $\approx 4.0\times$ & $\approx 16.0\times$ \\
\bottomrule
\end{tabular}
\end{table}

\subsection{NFE Is Proportional to the Transportation Cost}
\label{appendix:nfe_vs_C}

To validate the proportionality between NFE and transportation cost $C$, we construct intermediate states by linearly interpolating between source and target along $(t_1,t_2)\in[0,1]^2$ and estimate $C(t_1,t_2) = \mathbb{E}[|x_{t_2}-x_{t_1}|^2]$ from $1{,}000$ sampled pairs. Using SiT~\cite{ma2024sit} as the generative model, we gradually increase the NFE until reaching a fixed target FID of $1.85$, and report the resulting NFE in Table~\ref{tab:nfe_vs_C}.

The required NFE scales tightly with $C(t_1,t_2)$: the ratio NFE${/}C$ stays in a narrow band ($63.5\!-\!64.0$) across very different choices of $(t_1,t_2)$ that span an order-of-magnitude range in $C$. This empirically supports Assumption~\ref{ass:nfe_cost} (G.5) and indicates that the solver behaviour aligns well with the assumed proportionality across our schedules.

\begin{table}[t]
\centering
\small
\caption{Empirical relationship between transportation cost $C(t_1,t_2)=\mathbb{E}[|x_1-x_0|^2]$ and the NFE required to reach a fixed target FID ($1.85$) using SiT as the generative model on ImageNet-1k/256. For each $(t_1,t_2)$ pair we draw $1{,}000$ pairs of intermediate states from a linear interpolant and sweep the NFE upward until the target FID is hit. Lower transport cost consistently leads to proportionally fewer NFEs, and the trend is consistent across different intermediate states; this empirically supports Assumption~\ref{ass:nfe_cost} (G.5).}
\label{tab:nfe_vs_C}
\begin{tabular}{l c c c c c}
\toprule
$(t_1,t_2)$ & $(0.0,1.0)$ & $(0.0,0.75)$ & $(0.0,0.5)$ & $(0.0,0.25)$ & $(0.25,0.5)$ \\
\midrule
Transport cost $C$ & $1.00$ & $0.563$ & $0.250$ & $0.063$ & $0.063$ \\
Required NFE       & $64$   & $36$    & $16$    & $4$     & $4$     \\
NFE / $C$          & $64$   & $63.9$  & $64.0$  & $63.5$  & $63.5$  \\
\bottomrule
\end{tabular}
\end{table}

\subsection{Single-Stage vs.\ Multi-Stage Wall-Clock Comparison}
\label{appendix:single_vs_multi}

To isolate the wall-clock benefit attributable specifically to conditional dependent coupling, we compare CDC-FM against a single-stage pixel baseline that is identical to ours in every respect (backbone, optimizer, training schedule, classifier-free guidance, ODE solver) except that it sets $K{=}1$ and removes the coupling. As shown in Table~\ref{tab:single_vs_multi}, at matched FID/IS the single-stage baseline incurs $\sim 1.7\times$ the wall-clock of CDC-FM, confirming that the speedup is not an artifact of architecture or hyperparameters but originates from the multi-stage construction itself: low-level information is generated cheaply at low-resolution stages, and high-resolution stages benefit from the noised low-resolution outputs rather than re-generating all information at full resolution.

\begin{table}[t]
\centering
\small
\caption{Single-stage pixel baseline ($K{=}1$) vs.\ CDC-FM under an identical backbone, training schedule, optimizer, and ODE solver. Removing only the conditional dependent coupling and re-collapsing the model to one stage costs $\sim 1.7\times$ wall-clock at matched FID, isolating the efficiency gain attributable to the proposed coupling. Speed is reported in seconds per image at batch size $1$ on a single H100 GPU.}
\label{tab:single_vs_multi}
\begin{tabular}{l c c c c}
\toprule
Method & FID $\downarrow$ & IS $\uparrow$ & NFE & Speed (s/img) $\downarrow$ \\
\midrule
Single-stage Pixel ($K{=}1$) & $1.82$ & $305.8$ & $250\!\times\!2$ & $8.73$ \\
\rowcolor{lightblue!30}\textbf{CDC-FM (Ours)} & $1.81$ & $304.3$ & $16\!\times\!4\!\times\!2$ & $\mathbf{5.18}$ \\
\bottomrule
\end{tabular}
\end{table}

\subsection{Additional Ablations: Stage Path and Per-Stage NFE Allocation}
\label{appendix:stage_path_nfe}

\paragraph{Stage count and resolution path.}
Table~\ref{tab:ab_stage_path} ablates the stage count $K$ and the resolution path. Excessive per-stage ratios (rows 1--2) hurt FID; with moderate ratios, $K{=}3$ remains close to $K{=}4$, while $K{=}4$ attains the best FID. We do not consider resolutions below $32^2$ since they lack enough structural and semantic content to support meaningful generation.

\begin{table}[t]
\centering
\small
\caption{Impact of the stage count $K$ and the resolution path on ImageNet-1k/256. Excessively large per-stage resolution ratios degrade FID (rows 1--2). With moderate ratios, fewer stages remain viable (rows 3--4): the gap between $K{=}3$ and $K{=}4$ is small, while $K{=}4$ obtains the best FID. Choices below the $32^2$ floor yield insufficient structural/semantic content and were excluded.}
\label{tab:ab_stage_path}
\begin{tabular}{c l c}
\toprule
Stages & Resolution Path & FID $\downarrow$ \\
\midrule
$2$ & $\;32 \to 256$        & $2.06$ \\
$2$ & $\;64 \to 256$        & $1.97$ \\
$3$ & $\;32 \to 128 \to 256$ & $1.83$ \\
\rowcolor{lightblue!30} $4$ & $\;32 \to 64 \to 128 \to 256$ & $\mathbf{1.81}$ \\
\bottomrule
\end{tabular}
\end{table}

\paragraph{Per-stage NFE allocation.}
Table~\ref{tab:ab_stage_nfe} examines the sensitivity of FID to per-stage NFE. Reducing the NFE of any single stage hurts FID, but the degradation is substantially larger at later, higher-resolution stages (e.g.\ FID jumps from $1.81$ to $1.94$ when $N_4$ drops from $16$ to $4$, vs.\ $1.81$ to $1.85$ when $N_1$ drops from $16$ to $4$). This indicates that solver-step budget should be biased toward high-resolution stages.

\begin{table}[t]
\centering
\small
\caption{Per-stage NFE allocation ablation on ImageNet-1k/256 with the $32 \to 64 \to 128 \to 256$ hierarchy. We hold three of the per-stage NFEs at $16$ and reduce the fourth to $\{8,4\}$. Reducing the NFE of any single stage degrades FID, but the degradation is substantially larger at later, higher-resolution stages, indicating that solver-step budget should be biased toward high-resolution stages.}
\label{tab:ab_stage_nfe}
\begin{tabular}{c c c c c c}
\toprule
$N_1$ & $N_2$ & $N_3$ & $N_4$ & Total NFE & FID $\downarrow$ \\
\midrule
\rowcolor{lightblue!30}$16$ & $16$ & $16$ & $16$ & $64$ & $\mathbf{1.81}$ \\
$8$ & $16$ & $16$ & $16$ & $56$ & $1.83$ \\
$4$ & $16$ & $16$ & $16$ & $52$ & $1.85$ \\
$16$ & $8$ & $16$ & $16$ & $56$ & $1.85$ \\
$16$ & $4$ & $16$ & $16$ & $52$ & $1.87$ \\
$16$ & $16$ & $8$ & $16$ & $56$ & $1.88$ \\
$16$ & $16$ & $4$ & $16$ & $52$ & $1.91$ \\
$16$ & $16$ & $16$ & $8$ & $56$ & $1.90$ \\
$16$ & $16$ & $16$ & $4$ & $52$ & $1.94$ \\
\bottomrule
\end{tabular}
\end{table}

\subsection{Additional Multi-Stage Speed Comparison}
\label{appendix:multistage_speed}

The main Table~\ref{tab:classcondition_imagenet_256} reports speed numbers for the multi-stage baselines (CDM, PixelFlow, RIN, PaGoDA, FractalMAR-H) under a uniform measurement protocol (single H100 GPU, batch size $1$, no compilation/quantization). Where original papers do not report inference time, we either reproduce numbers from public checkpoints (PixelFlow, FractalMAR-H, PaGoDA) or mark the entry as ``$-$''. CDC-FM is the only multi-stage method in this comparison that simultaneously achieves sub-$2.0$ FID and sub-$6$\,sec/img.

\section{Higher-Resolution Results on ImageNet-1k/512} \label{appendix:imagenet512}

To assess the robustness of CDC-FM at higher resolution, we report ImageNet-1k/512 class-conditional generation results in Table~\ref{tab:classcondition_imagenet_512}. The same backbone, training pipeline, and CFG schedule are used as for the $256^2$ experiments, with the resolution hierarchy extended accordingly. CDC-FM remains Pareto-competitive with single-stage and multi-stage baselines at this higher resolution.

\begin{table}[t]
    \centering
    \caption{Comparison on ImageNet-1k/512 conditional generation.}
    \label{tab:classcondition_imagenet_512}
    
    \footnotesize

    \renewcommand{\arraystretch}{0.9}

    \setlength{\aboverulesep}{1pt}
    \setlength{\belowrulesep}{1pt}

    \setlength{\tabcolsep}{3pt}

    \begin{tabular}{l c c c}
        \toprule
        Model/Method & \texttt{Params} & \texttt{FID} $\downarrow$ & \texttt{IS} $\uparrow$ \\
        \midrule
        BigGAN~\cite{brock2018biggan} & - & 8.43 & 177.9 \\
        \midrule
        ADM~\cite{dhariwal2021adm} & 554M & 7.72 & 172.7 \\
        VDM++~\cite{kingma2023vdmpp} & 2B & 2.65 & 278.1 \\
        DiT\text{-}XL/2~\cite{peebles2023dit} & 675M & 3.04 & 240.8 \\
        U\text{-}ViT~\cite{bao2023uvit} & 501M & 4.05 & 263.8 \\
        SiD2~\cite{hoogeboom2025sid2} & 2.0B & 2.19 & - \\
        \midrule
        MaskGIT~\cite{chang2022maskgit} & 227M & 7.32 & 156.0 \\
        MAGVIT-v2~\cite{yulanguage} & 307M & 3.07 & 324.3 \\
        MAR-L~\cite{li2024autoregressive} & 481M & 2.74 & 205.2 \\
        VQGAN~\cite{esser2021vqgan} & 1.4B & 26.52 & 66.8 \\
        VAR-d36-s~\cite{tian2024var} & 2.0B & 2.63 & 303.2 \\
        \midrule
        PaGoDA~\cite{kim2024pagoda} & 0.9B & 1.80 & 251.3 \\
        \midrule
        \rowcolor{lightblue!30} 
        \textbf{CDC-FM} & 684M & 2.65 & 307.4 \\
        \bottomrule
    \end{tabular}

\end{table}

\section{Algorithm}
\label{ap:alg_uncondition}

Here we present the detailed unconditional training and inference procedures for the unified multi-stage generative model $b^\theta$ with \textbf{conditional dependent coupling} in Algorithm~\ref{alg:training_multistage_fm} and Algorithm~\ref{alg:inference_multistage_fm}. The inference steps are described using the forward Euler method~\cite{lipman2024flow, lipman2022flow} for solving the ODE, as a simple example of a numerical solver. 

\begin{algorithm}[t]
\caption{Training a unified multi-stage Flow-Matching model $b^\theta$ (ODE view)}
\label{alg:training_multistage_fm}
\begin{algorithmic}[1]
\Require Dataset $\mathcal{D}=\{x^{(K)}\}\sim p_{\text{data}}$, stages $k\in\{1,\dots,K\}$ with dimensions $d_k$, feature-map resolutions $r_k$ and resolution embedding $e_k=E(r_k)$; down/upsamplers $D_k:\mathbb{R}^{d_k}\!\to\!\mathbb{R}^{d_{k-1}}$ and $U_k:\mathbb{R}^{d_{k-1}}\!\to\!\mathbb{R}^{d_k}$, composite downsampler $D_{K\to k}$; time partition $0=t^1_0<t^1_1=t^2_0<\cdots<t^K_1=1$; noise scales $\sigma$ (stage 1) and $\{\sigma_k\}_{k=2}^K$; batch size $B$.
\Statex \textbf{Comment:} Ground-truth at stage $k$ is $x_1^{(k)}=D_{K\to k}(x^{(K)})\sim\rho_1^{(k)}(\cdot\mid x^{(K)})$.
\For{each training step}
  \For{$i=1$ {\bf to} $B$}
    \State Sample $x^{(K)} \sim p_{\text{data}}$; sample $t \sim \mathcal{U}[0,1]$.
    \State Find $k$ such that $t \in [t^k_0,t^k_1]$; set the rescaled time $\tau \gets \dfrac{t-t^k_0}{t^k_1-t^k_0}\in[0,1]$.
    \State \textbf{Target at stage $k$:} $x_1^{(k)} \gets D_{K\to k}(x^{(K)})$ \Comment{$x_1^{(k)} \sim \rho_1^{(k)}(\cdot \mid x^{(K)})$}
    \If{$k=1$} \Comment{Stage $1$: noise-to-image coupling is independent}
        \State Sample $x_0^{(1)} \sim \mathcal{N}(0,\sigma^2 I_{d_1})$ \Comment{$\rho_0^{(1)}(x_0^{(1)})$}
    \Else \Comment{Stage $k>1$: conditional dependent coupling}
        \State Sample $\zeta_k \sim \mathcal{N}(0,I_{d_k})$
        \State $m_k(x_1^{(k)}) \gets U_k\!\big(D_k(x_1^{(k)})\big)$
        \State $x_0^{(k)} \gets m_k(x_1^{(k)}) + \sigma_k \zeta_k$ \Comment{$\rho_0^{(k)}(\cdot\mid x_1^{(k)})=\mathcal{N}(m_k,\sigma_k^2 I)$}
    \EndIf
    \State \textbf{Linear interpolant:} $I_\tau^{(k)} \gets (1-\tau)\,x_0^{(k)} + \tau\,x_1^{(k)}$
    \State \textbf{Target velocity:} $\dot I_\tau^{(k)} \gets x_1^{(k)} - x_0^{(k)}$ \Comment{constant in $\tau$ for linear path}
    \State $e_k \gets E(r_k)$ \Comment{resolution embedding}
    \State $u_\theta \gets b^\theta\!\big(I_\tau^{(k)}, \tau, e_k\big)$ \Comment{DiT vector field shared across stages}
    \State Per-sample loss: $\ell_i \gets \|u_\theta\|_2^2 - 2\,\dot I_\tau^{(k)}\!\cdot u_\theta$ \Comment{equiv.\ to $\|u_\theta-\dot I_\tau^{(k)}\|_2^2$ up to a const.}
  \EndFor
  \State $\mathcal{L}(\theta) \gets \frac{1}{B}\sum_{i=1}^B \ell_i$ \Comment{matches $\mathbb{E}\big[\|b^\theta\|^2 - 2\,\dot I\!\cdot b^\theta\big]$}
  \State Update parameters: $\theta \leftarrow \theta - \eta \nabla_\theta \mathcal{L}(\theta)$
\EndFor
\end{algorithmic}
\end{algorithm}

\begin{algorithm}[t]
\caption{Inference via sequential multi-stage ODE integration (forward Euler)}
\label{alg:inference_multistage_fm}
\begin{algorithmic}[1]
\Require Trained vector field $b^\theta$; stages $k\in\{1,\dots,K\}$ with dimensions $d_k$, feature-map resolutions $r_k$ and resolution embedding $e_k=E(r_k)$; upsamplers $U_k:\mathbb{R}^{d_{k-1}}\!\to\!\mathbb{R}^{d_k}$; noise scales $\sigma$ (stage 1) and $\{\sigma_k\}_{k=2}^K$; per-stage step counts $\{N_k\}$.
\Statex \textbf{Comment:} The Markovian cascade uses $\hat x_0^{(1)} \sim \mathcal{N}(0,\sigma^2 I_{d_1})$ and, for $k\ge 2$, $\hat x_0^{(k)} = U_k(\hat x_1^{(k-1)}) + \sigma_k \zeta_k$ with $\zeta_k\!\sim\!\mathcal{N}(0,I_{d_k})$.
\State \textbf{Initialize stage 1 (noise-to-image):} Sample $\hat x_0^{(1)} \sim \mathcal{N}(0,\sigma^2 I_{d_1})$; set $\hat X^{(1)}_{0} \gets \hat x_0^{(1)}$.
\For{$k=1$ {\bf to} $K$}
  \State $\Delta\tau \gets 1/N_k$; $e_k \gets E(r_k)$
  \If{$k>1$} \Comment{Conditional dependent coupling from previous stage}
     \State Sample $\zeta_k \sim \mathcal{N}(0,I_{d_k})$
     \State $\hat x_0^{(k)} \gets U_k\!\big(\hat x_1^{(k-1)}\big) + \sigma_k \zeta_k$
     \State $\hat X^{(k)}_{0} \gets \hat x_0^{(k)}$
  \EndIf
  \For{$n=0$ {\bf to} $N_k-1$}
     \State $\tau \gets n\,\Delta\tau$
     \State $\hat v \gets b^\theta\!\big(\hat X^{(k)}_{\tau}, \tau, e_k\big)$ \Comment{ODE velocity at rescaled time $\tau\in[0,1]$}
     \State \textbf{Euler step:} $\hat X^{(k)}_{\tau+\Delta\tau} \gets \hat X^{(k)}_{\tau} + \Delta\tau \cdot \hat v$
  \EndFor
  \State \textbf{Stage output:} $\hat x_1^{(k)} \gets \hat X^{(k)}_{1}$ \Comment{resolution $d_k$}
\EndFor
\State \Return $\hat x_1^{(K)} \in \mathbb{R}^{d_K}$ \Comment{final high-resolution sample}
\end{algorithmic}
\end{algorithm}

\section{Classifier-Free Guidance for Conditional Dependent Coupling}
\label{ap:cfg}

\textbf{Rationale.}
Because each stage $k$ in our multi-stage model learns the \emph{accurate} data distribution $\rho_1^{(k)}(x_1^{(k)} \mid x^{(K)})$ under the \textbf{conditional dependent coupling} (Sec.~\ref{sec:conditionalDepCoupling}), Classifier-Free Guidance (CFG)~\cite{ho2022classifier} can be applied \emph{uniformly} across stages. This obviates complex stage-wise guidance schedules~\cite{chen2025pixelflow,  kynkaanniemi2024applying} and yields a simpler, robust implementation.

\paragraph{Training with conditional dropout (classifier-free).}
We train the unified DiT $b^\theta$ exactly as in the conditional generation objective shown in \eqref{eq:conditional_generation_int}, with the only change that the conditioning signal is randomly dropped. Let $\mathbf{c}$ denote the conditioning signal (e.g., class label or text prompt), and let $\varnothing$ denote the null condition. Define the mixed conditioning
\begin{equation}
\bar{\mathbf{c}} \sim q(\bar{\mathbf{c}}) \;\;=\;\; (1-p_\varnothing)\,p(\mathbf{c}) \;+\; p_\varnothing\,\delta_{\varnothing},
\end{equation}
where $p_\varnothing \in (0,1)$ is the dropout probability and $\delta_{\varnothing}$ is a point mass at $\varnothing$.
The training loss becomes
\begin{equation}
\label{eq:cfg_training_loss}
\mathcal{L}_{\text{cfg}}(\theta)
=\mathbb{E}_{\substack{t\sim[0,1],\, k,\, e_k,\\ (x_0^{(k)}, x_1^{(k)}) \sim \rho^{(k)}(x_0^{(k)}, x_1^{(k)} | x^{(K)}),\\ \bar{\mathbf{c}}\sim q}}
\!\left[\; \big\lVert b^{\theta}\!\big(I_\tau^{(k)}, \tau, e_k, \bar{\mathbf{c}}\big) \big\rVert^2
\;-\;2\,\dot{I}_\tau^{(k)}\!\cdot b^{\theta}\!\big(I_\tau^{(k)}, \tau, e_k, \bar{\mathbf{c}}\big) \;\right],
\end{equation}
with $I_\tau^{(k)}=(1-\tau)x_0^{(k)}+\tau x_1^{(k)}$, $\dot{I}_\tau^{(k)}=x_1^{(k)}-x_0^{(k)}$, and $\tau=\frac{t-t_0^k}{t_1^k-t_0^k}$ as in Sec.~\ref{sec:conditionalDepCoupling}.

\paragraph{CFG at inference: guided vector field.}
Let the \emph{conditional} and \emph{unconditional} vector-field predictions be
\begin{equation}
b_\mathbf{c}^{\theta}(z,\tau,e_k) \;=\; b^{\theta}(z,\tau,e_k,\mathbf{c}),
\qquad
b_\varnothing^{\theta}(z,\tau,e_k) \;=\; b^{\theta}(z,\tau,e_k,\varnothing).
\end{equation}
Given a global guidance scale $S_{\text{cfg}}\!\ge 0$ that is \emph{shared across all stages $k$}, we define the CFG-guided field by the standard linear rule~\cite{ho2022classifier}:
\begin{equation}
\label{eq:cfg_guided_field}
\begin{aligned}
b_{\text{cfg}}^{\theta}(z,\tau,e_k;\mathbf{c},S_{\text{cfg}})
&\;=\; b_\varnothing^{\theta}(z,\tau,e_k)
\;+\; S_{\text{cfg}}\Big( b_\mathbf{c}^{\theta}(z,\tau,e_k) - b_\varnothing^{\theta}(z,\tau,e_k) \Big)  \\
&\;=\; (1-S_{\text{cfg}})\,b_\varnothing^{\theta} + S_{\text{cfg}}\,b_\mathbf{c}^{\theta}.
\end{aligned}
\end{equation}
This recovers unconditional generation when $S_{\text{cfg}}=0$, the nominal conditional model when $S_{\text{cfg}}=1$, and stronger condition-following for $S_{\text{cfg}}>1$.

\paragraph{Stage-wise conditional generation under coupling.}
At inference, the state $X_\tau^{(k)}$ in stage $k$ evolves under the guided ODE driven by~\eqref{eq:cfg_guided_field}:
\begin{equation}
\label{eq:cfg_stage_ode}
\frac{\mathrm{d}}{\mathrm{d}\tau} X_\tau^{(k)}
\;=\; b_{\text{cfg}}^{\theta}\!\big(X_\tau^{(k)}, \tau, e_k;\mathbf{c}, S_{\text{cfg}}\big),
\qquad \tau\in[0,1],
\end{equation}
with initial condition
\begin{equation}
\label{eq:cfg_stage_ic}
X_0^{(1)}=x_0^{(1)} \sim \mathcal{N}(0,\sigma^2 I_{d_1}),
\qquad
X_0^{(k)} \;=\; U_k\!\big(\hat{x}_1^{(k-1)}\big) + \sigma_k \zeta_k,\;\; \zeta_k\sim\mathcal{N}(0,I_{d_k}),\;\; k\ge 2,
\end{equation}
where the conditional dependent coupling is exactly the construction in~\eqref{eq:data_dependent_coupling_compute}.
The stage output is the terminal state
\begin{equation}
\label{eq:cfg_stage_solution}
\hat{x}_1^{(k)} \;=\; X_1^{(k)}
\;=\; X_0^{(k)} \;+\; \int_{0}^{1} b_{\text{cfg}}^{\theta}\!\big(X_t^{(k)}, t, e_k;\mathbf{c}, S_{\text{cfg}}\big)\,\mathrm{d}t,
\end{equation}
and the full cascade is obtained by the recursion
\begin{equation}
\label{eq:cfg_markov_recursion}
\hat{x}_1^{(1)} \;=\; X_1^{(1)},\qquad
\hat{x}_1^{(k)} \;=\; X_1^{(k)}\big(\,U_k(\hat{x}_1^{(k-1)})+\sigma_k\zeta_k\,\big),\;\; k=2,\dots,K,
\end{equation}
which is the same Markov structure as~\eqref{eq:inference_markov} but with the guided field $b_{\text{cfg}}^{\theta}$.

\paragraph{Why CFG is still needed at later, higher-resolution stages.}
A natural question is whether higher-resolution stages should drop CFG, since they may appear to perform purely deterministic super-resolution. This intuition does \emph{not} hold under conditional dependent coupling. For any stage $k\!\ge\!2$, the source $x_0^{(k)} = U_k(\hat{x}_1^{(k-1)}) + \sigma_k\zeta_k$ is the noised upsampling of the previous stage's output, not a deterministic copy; the target $\rho_1^{(k)}(x_1^{(k)}\!\mid x^{(K)})$ is the accurate data distribution at resolution $d_k$, which retains genuine semantic and stochastic content that the lower-resolution stage could not deliver. Hence each stage solves a generative (not purely super-resolution) problem, and conditional information injected via CFG continues to be informative throughout the cascade. Empirically, Figure~\ref{fig:stageFID} shows that all stages attain their lowest FID at a similar $S_\text{cfg}$ value, which justifies the constant-$S_{\text{cfg}}$ choice rather than a stage-wise schedule.

\paragraph{Remarks.}
(i) The use of a \emph{single} $S_{\text{cfg}}$ across all stages is justified by the fact that, under conditional dependent coupling, each stage already targets the correct $\rho_1^{(k)}(x_1^{(k)}\!\mid x^{(K)})$. Hence, CFG serves primarily to bias the direction of transport within each accurately learned stage rather than to compensate for a stage mismatch, removing the need for stage-wise tuning.
(ii) For an SDE parameterization, the same rule~\eqref{eq:cfg_guided_field} is applied to the predicted score/drift in place of $b^\theta$:
\[
s_{\text{cfg}}^{\theta}(z,t,e_k;\mathbf{c},S_{\text{cfg}})
= s_\varnothing^{\theta}(z,t,e_k)
+ S_{\text{cfg}}\!\left(s_\mathbf{c}^{\theta}(z,t,e_k)-s_\varnothing^{\theta}(z,t,e_k)\right),
\]
and the sampler integrates the corresponding guided stochastic dynamics with $X_0^{(k)}$ as in~\eqref{eq:cfg_stage_ic}.
(iii) Following the CFG strategy introduced in this section for conditional generation, we extend the unconditional generation procedure described in Appendix~\S\ref{ap:alg_uncondition}, and provide the training and inference procedures of the unified multi-stage generative model $b^\theta$ with CFG as shown below.

\begin{algorithm}[H]
\caption{Training a unified multi-stage Flow-Matching model $b^\theta$ with Classifier-Free Guidance}
\label{alg:training_cfg_multistage_fm}
\begin{algorithmic}[1]
\Require Dataset $\mathcal{D}=\{x^{(K)}, \mathbf{c}\}$ with conditions $\mathbf{c}$ (e.g., class/text); stages $k\in\{1,\dots,K\}$ with dimensions $d_k$, feature-map resolutions $r_k$ and resolution embedding $e_k=E(r_k)$; down/upsamplers $D_k:\mathbb{R}^{d_k}\!\to\!\mathbb{R}^{d_{k-1}}$ and $U_k:\mathbb{R}^{d_{k-1}}\!\to\!\mathbb{R}^{d_k}$; composite downsampler $D_{K\to k}$; time partition $0=t^1_0<t^1_1=t^2_0<\cdots<t^K_1=1$; noise scales $\sigma$ (stage 1) and $\{\sigma_k\}_{k=2}^K$; classifier-free dropout probability $p_{\varnothing}\in(0,1)$; batch size $B$.
\Statex \textbf{Comment:} Ground-truth at stage $k$ is $x_1^{(k)}=D_{K\to k}(x^{(K)})\sim\rho_1^{(k)}(\cdot\mid x^{(K)})$. Conditional dropout samples $\bar{\mathbf{c}}$ from $q(\bar{\mathbf{c}})=(1-p_{\varnothing})\,p(\mathbf{c})+p_{\varnothing}\,\delta_{\varnothing}$.
\For{each training step}
  \For{$i=1$ {\bf to} $B$}
    \State Sample $(x^{(K)}, \mathbf{c}) \sim \mathcal{D}$; sample $t \sim \mathcal{U}[0,1]$.
    \State Find $k$ such that $t \in [t^k_0,t^k_1]$; set $\tau \gets \dfrac{t-t^k_0}{t^k_1-t^k_0}\in[0,1]$.
    \State \textbf{Target at stage $k$:} $x_1^{(k)} \gets D_{K\to k}(x^{(K)})$.
    \If{$k=1$} \Comment{Stage $1$: noise-to-image, independent prior}
        \State Sample $x_0^{(1)} \sim \mathcal{N}(0,\sigma^2 I_{d_1})$ \Comment{$\rho_0^{(1)}(x_0^{(1)})$}
    \Else \Comment{Stage $k>1$: conditional dependent coupling}
        \State Sample $\zeta_k \sim \mathcal{N}(0,I_{d_k})$
        \State $m_k(x_1^{(k)}) \gets U_k\!\big(D_k(x_1^{(k)})\big)$
        \State $x_0^{(k)} \gets m_k(x_1^{(k)}) + \sigma_k \zeta_k$ \Comment{$\rho_0^{(k)}(\cdot\mid x_1^{(k)})=\mathcal{N}(m_k,\sigma_k^2 I)$}
    \EndIf
    \State \textbf{Linear interpolant:} $I_\tau^{(k)} \gets (1-\tau)\,x_0^{(k)} + \tau\,x_1^{(k)}$
    \State \textbf{Target velocity:} $\dot I_\tau^{(k)} \gets x_1^{(k)} - x_0^{(k)}$ \Comment{constant in $\tau$ for linear path}
    \State \textbf{Conditional dropout:} Draw $u\sim\mathcal{U}[0,1]$; set $\bar{\mathbf{c}}\gets \varnothing$ if $u<p_{\varnothing}$, else $\bar{\mathbf{c}}\gets \mathbf{c}$.
    \State $e_k \gets E(r_k)$ \Comment{resolution embedding fused with time embedding}
    \State $u_\theta \gets b^\theta\!\big(I_\tau^{(k)}, \tau, e_k, \bar{\mathbf{c}}\big)$ \Comment{DiT vector field shared across stages}
    \State Per-sample loss: $\ell_i \gets \|u_\theta\|_2^2 - 2\,\dot I_\tau^{(k)}\!\cdot u_\theta$ \Comment{matches \eqref{eq:cfg_training_loss}}
  \EndFor
  \State $\mathcal{L}_{\text{cfg}}(\theta) \gets \frac{1}{B}\sum_{i=1}^B \ell_i$
  \State Update parameters: $\theta \leftarrow \theta - \eta \nabla_\theta \mathcal{L}_{\text{cfg}}(\theta)$
\EndFor
\end{algorithmic}
\end{algorithm}

\begin{algorithm}[H]
\caption{Inference via sequential multi-stage ODE integration with Classifier-Free Guidance}
\label{alg:inference_cfg_multistage_fm}
\begin{algorithmic}[1]
\Require Trained vector field $b^\theta$; user condition $\mathbf{c}$; global guidance scale $S_{\text{cfg}}\ge 0$ (shared across stages); stages $k\in\{1,\dots,K\}$ with dimensions $d_k$, feature-map resolutions $r_k$ and resolution embedding $e_k=E(r_k)$; upsamplers $U_k:\mathbb{R}^{d_{k-1}}\!\to\!\mathbb{R}^{d_k}$; noise scales $\sigma$ (stage 1) and $\{\sigma_k\}_{k=2}^K$; per-stage step counts $\{N_k\}$.
\Statex \textbf{Comment:} Guided field $b_{\text{cfg}}^\theta(z,\tau,e_k;\mathbf{c},S_{\text{cfg}})=(1-S_{\text{cfg}})b_\varnothing^\theta(z,\tau,e_k)+S_{\text{cfg}}b_\mathbf{c}^\theta(z,\tau,e_k)$, cf.\ \eqref{eq:cfg_guided_field}.
\State \textbf{Initialize stage 1 (noise-to-image):} Sample $\hat x_0^{(1)} \sim \mathcal{N}(0,\sigma^2 I_{d_1})$; set $\hat X^{(1)}_{0} \gets \hat x_0^{(1)}$.
\For{$k=1$ {\bf to} $K$}
  \State $\Delta\tau \gets 1/N_k$; $e_k \gets E(r_k)$
  \If{$k>1$} \Comment{Conditional dependent coupling from previous stage}
     \State Sample $\zeta_k \sim \mathcal{N}(0,I_{d_k})$
     \State $\hat x_0^{(k)} \gets U_k\!\big(\hat x_1^{(k-1)}\big) + \sigma_k \zeta_k$
     \State $\hat X^{(k)}_{0} \gets \hat x_0^{(k)}$
  \EndIf
  \For{$n=0$ {\bf to} $N_k-1$}
     \State $\tau \gets n\,\Delta\tau$
     \State \textbf{Unconditional:} $\hat v_\varnothing \gets b^\theta\!\big(\hat X^{(k)}_{\tau}, \tau, e_k, \varnothing\big)$
     \State \textbf{Conditional:} $\hat v_\mathbf{c} \gets b^\theta\!\big(\hat X^{(k)}_{\tau}, \tau, e_k, \mathbf{c}\big)$
     \State \textbf{Guided field:} $\hat v_{\text{cfg}} \gets (1-S_{\text{cfg}})\hat v_\varnothing + S_{\text{cfg}}\hat v_\mathbf{c}$ \Comment{$S_{\text{cfg}}{=}0$ uncond., $S_{\text{cfg}}{=}1$ nominal cond.}
     \State \textbf{Euler step:} $\hat X^{(k)}_{\tau+\Delta\tau} \gets \hat X^{(k)}_{\tau} + \Delta\tau \cdot \hat v_{\text{cfg}}$
  \EndFor
  \State \textbf{Stage output:} $\hat x_1^{(k)} \gets \hat X^{(k)}_{1}$ \Comment{resolution $d_k$}
\EndFor
\State \Return $\hat x_1^{(K)} \in \mathbb{R}^{d_K}$ \Comment{final high-resolution sample}
\end{algorithmic}
\end{algorithm}

\section{Implement Details} \label{appendix:imp_detail}
We report the implementation details of our model in Table~\ref{tab:implement}.

All experiments are conducted on a platform equipped with 32 $\times$ A100 GPUs (80GB). 
Under the configuration shown below, the global training batch size is $32 \times 8 = 256$. 
We train for $200$ epochs; for ImageNet-1k, which contains 1.28M samples, each epoch consists of approximately 5000 steps. 
Thus, the full training schedule corresponds to roughly 1M iterations (a metric used in several prior works). 
With this setup, the total training time is approximately 76 hours, averaging 0.38 hours per epoch. 
During training, each A100 GPU consumes about 43GB of memory.

For inference, we use a single A100 GPU and set the batch size to 32 for all methods. 
We generate 50k images for evaluation and report inference speed in seconds per image.

\begin{table}[h]
\centering
\caption{Complete Parameter Configuration}
\label{tab:implement}
\footnotesize
\setlength{\tabcolsep}{3pt}
\renewcommand{\arraystretch}{1.1}
\begin{tabular}{@{}p{0.15\textwidth}p{0.26\textwidth}p{0.28\textwidth}p{0.24\textwidth}@{}}
\hline
\textbf{Category} & \textbf{Parameter} & \textbf{Value} & \textbf{Description} \\
\hline
\multirow{10}{=}{\textbf{Model Architecture}} & \texttt{num\_attention\_heads} & 16 & Number of attention heads \\
& \texttt{attention\_head\_dim} & 72 & Dimension of each attention head \\
& \texttt{in\_channels} & 3 & Input channels (RGB) \\
& \texttt{out\_channels} & 3 & Output channels (RGB) \\
& \texttt{depth} & 28 & Number of transformer layers \\
& \texttt{num\_classes} & 1000 & Number of ImageNet classes \\
& \texttt{patch\_size} & 4 & Patch size for patch embedding \\
& \texttt{attention\_bias} & \texttt{true} & Whether to use bias in attention \\
& \texttt{embed\_dim} & 1152 & Embedding dimension (16 × 72) \\
& \texttt{dropout} & 0.0 & Dropout rate \\
\hline
\multirow{3}{=}{\textbf{Scheduler}} & \texttt{num\_train\_timesteps} & 1000 & Number of training timesteps \\
& \texttt{num\_stages} & 4 & Number of cascade stages \\
& \texttt{diminish factor $\gamma$} & 2.0 & $\gamma$ in CDC-FM \\
\hline
\multirow{6}{=}{\textbf{Training}} & \texttt{lr} & 1e-4 & Learning rate \\
& \texttt{weight\_decay} & 0.0 & Weight decay \\
& \texttt{epochs} & 200 & Number of training epochs \\
& \texttt{grad\_clip\_norm} & 1.0 & Gradient clipping norm \\
& \texttt{ema\_decay} & 0.9999 & EMA decay rate \\
& \texttt{logging\_steps} & 100 & Logging frequency \\
\hline
\multirow{6}{=}{\textbf{Data}} & \texttt{root} & \path{/public/datasets/ILSVRC2012/train} & Dataset root path \\
& \texttt{center\_crop} & \texttt{false} & Whether to center crop \\
& \texttt{resolution} & 256 & Image resolution \\
& \texttt{expand\_ratio} & 1.125 & Image expansion ratio \\
& \texttt{num\_workers} & 8 & Number of data loader workers \\
& \texttt{batch\_size} & 8 & Batch size per GPU \\
\hline
\multirow{4}{=}{\textbf{Data Augmentation}} & \texttt{RandomHorizontalFlip} & \texttt{true} & Random horizontal flip \\
& \texttt{RandomCrop} & \texttt{true} & Random crop (if not center crop) \\
& \texttt{Resize} & \texttt{LANCZOS} & Resize interpolation method \\
& \texttt{Normalize} & \texttt{[0.5, 0.5, 0.5]} & Normalization mean and std \\
\hline
\multirow{5}{=}{\textbf{Optimization}} & \texttt{optimizer} & \texttt{AdamW} & Optimizer type \\
& \texttt{Optimizer Hyperparameters} & \makecell[l]{$\beta_1=0.9$, $\beta_2=0.95$\\$\epsilon=1e^{-6}$} & Optimizer Hyperparameters \\
& \texttt{Learning rate} & \texttt{$1e^{-4}$} & Learning rate \\
& \texttt{precision} & \texttt{bfloat16} & Training precision \\
& \texttt{deterministic\_ops} & \texttt{false} & Deterministic operations \\
\hline
\multirow{3}{=}{\textbf{Inference}} & \texttt{num\_inference\_steps} & 24 & Steps per stage (default) \\
& \texttt{CFG strength $S_\text{cfg}$} & 3.0 & CFG scale (evaluation) \\
& \texttt{num\_fid\_samples} & 50000 & Number of FID samples \\
\hline
\end{tabular}
\end{table}

\section{Theoretical Justification and Key Property}

\textbf{Mathematical premise and intuition.} 
Our method adopts a unified multi-scale, coarse-to-fine image generation procedure that aligns with the natural way humans recognize images~\cite{tian2024visual}. Compared to previous coarse-to-fine AR approaches~\cite{tian2024visual,renflowar}, our method is based on a unified and continuous transformation between distributions in a highly controllable manner, which better reflects the smooth nature of image transitions.
Compared to existing multi-stage generation methods~\cite{ho2022cascaded}, our method models a unified multi-stage generation process with \textbf{conditional dependent coupling} and can be parameterized by a single DiT that can be trained from end-to-end. 
Furthermore, due to the accurate distribution modeling at each stage, our method is guaranteed to learn the target high-quality data distribution by progressively matching the data distribution at each stage. This resolves the generation problem step by step within a unified process.

\textbf{Diversity and stochasticity.}
In image generation, diversity—typically introduced by the stochasticity of generative models—is primarily expressed in the early stages (initial transition steps for ODE/SDE)~\cite{chen2025pixelflow, tian2024var}. For the proposed method, this corresponds to the transition modeled in the initial and intermediate stages. During these stages, key attributes like object presence, layout, and structure are determined. In contrast, the subsequent high-resolution refinement stages require less diversity. By employing the proposed \textbf{conditional dependent coupling} strategy, we explicitly decouple the diverse coarse image generation from the fine image refinement, thereby achieving both high performance and efficiency. {\color{colorAdded}See Appendix~\S\ref{appendix:viz} for the visualization of the multi-stage generation and the details hierarchy it forms.} Specifically, as the stage index \(k\) increases, we gradually reduce the stage-dependent noise level \(\sigma_k\) in \eqref{eq:data_dependent_coupling_compute}. This adjustment improves both training and inference efficiency while maintaining strong performance, defined as
\begin{equation}
\label{eq:sigma_pyramid}
\sigma_k = \gamma^{-(k-1)} \sigma, \quad \text{where  } \gamma \geq 1, \quad 2 \le k \le K.
\end{equation}
We refer to $\gamma$ as the "diminish factor" and further analyze its impact in our experiments.

\textbf{Few-step generation and Efficient inference.} 
Although not originally designed as such, the generative task modeled by the proposed unified multi-stage method with \textbf{conditional dependent coupling} is decomposed into subtasks within each generation stage. This decomposition regularizes the overall generation trajectory into multiple straighter sub-trajectories, thereby reducing the number of function evaluations (NFE) required at each stage~\cite{frans2024one, gengconsistency}. Moreover, the simulations in the early stages are substantially faster than those in the final stage, which targets the full-resolution image. 
Quantitatively, this approach leads to a significant reduction in the transport cost, as established in Theorem~\ref{theorem:lowerTransportCost}. Consequently, fewer steps are needed to obtain qualified results, and the overall inference time is shortened. 
Here we summarize as Theorem~\ref{theorem:lowerInfer} below.
\begin{theorem} \label{theorem:lowerInfer}
\textit{(Under the same noise-decay schedule \eqref{eq:sigma_pyramid} as Theorem~\ref{theorem:lowerTransportCost}, together with the linear scaling of per-evaluation cost in ambient dimensionality and the proportionality between NFE and transportation cost; both formalized in Appendix~\S\ref{pf:fasterInfer} (Assumptions G.4 and G.5).)}
The inference time of the naive single-stage model,
which generates a high-resolution image $\hat{x}_1 \sim x^{(K)} \in \mathbb{R}^{d_K}$ directly from Gaussian noise $x_0 \sim \mathcal{N}(0, \sigma^2 I_{d_K})$ (denoted as $T_A$), is greater than that of the proposed multi-stage model with a conditional dependent coupling strategy, which generates a high-resolution image $\hat{x}_1^{(K)} \sim x^{(K)} \in \mathbb{R}^{d_K}$ from Gaussian noise $x^{(1)}_0 \sim \mathcal{N}(0, \sigma^2 I_{d_1})$ through $K$ stages (denoted as $T_B$). Specifically, we have $T_A > T_B$.
\begin{equation}
    T_A = \mathbb{E}[\iota(x_0 \to \hat{x}_1)] \; > \;
    T_B = \sum_{k=1}^{K} \mathbb{E}[\iota(x_0^{(k)} \to \hat{x}_1^{(k)})],
\end{equation}
where $\iota(\cdot)$ denotes the expected inference time cost of generating the target image from the input sample at each stage. See proof in Appendix~\S\ref{pf:fasterInfer}; we provide empirical validation of the two ambient-cost / NFE assumptions in Appendix~\S\ref{appendix:efficiency_validation}.
\end{theorem}

\section{Future Work}
\label{appendix:future}

The present study uses class-conditional ImageNet generation, a fixed resolution hierarchy, and deterministic Flow Matching. Several extensions follow from this setting.

\paragraph{Adaptive paths and coordinated control.}
After training on compatible resolution paths, CDC-FM could jointly select intermediate resolutions, per-stage NFE, and guidance strengths under a latency budget, building on Appendix~\S\ref{appendix:stage_path_nfe}. Generative subgoal planning and uncertainty-aware scheduling motivate learned intermediate targets and calibrated budget decisions~\citep{zhang2026dsp,WU2026134565}. Equilibrium-based coordination and selective communication could reconcile specialized controllers~\citep{yi2025debate,xie2026dice,xie2025acorn}, while trust-region and sequential tuning suggest bounded controller updates~\citep{xie2026teamtr,xie2026sat}. Experiments should include planning and communication overhead and test whether improvements survive error propagation across resolutions. The cited coordination guarantees require checking their assumptions in this setting.

\paragraph{Compression and stable inference.}
Low-rank approximation motivates compressing the shared DiT~\citep{wang2026qsvd,wang2026wsvd,wang2026laser,du2025nps,xie2026step}, and bootstrapped tokenization suggests exploring compressed representations~\citep{chi2026balancing}, while head-sensitive KV-cache quantization in autoregressive LLMs suggests studying attention precision with separate calibration for CDC-FM~\citep{WU2026103699}. Entropy-aware variance control, forward-process alignment, Stein stabilization, and Kalman control also motivate stage-specific solver corrections~\citep{li2025evodiff,li2026difa,li2026mitigating,chi2025enhancing}. Compression must cover the model's resolutions and timesteps; solver corrections must respect its conditional velocity fields and stage boundaries. Evaluation should measure fidelity, diversity, memory, and latency at matched budgets, including correction costs. Learnable priors, data-dependent refinement noise, and SDE sampling would additionally require revisiting the transport-cost and inference-time analysis.

\paragraph{Adaptation and gated distillation.}
Test-time adaptation, teacher guidance, and prior-anchored estimation motivate compact conditioning updates for new domains~\citep{huang2026drives,chen2026test,huang2026mises}. Estimating pretraining gradients and orthogonalizing later updates suggests studying how to limit forgetting~\citep{wang2026estimating}. Gated self-distillation, visual return prediction, and confidence-guided training motivate validated feedback for selecting stage-specific teacher targets~\citep{huang2026skill,lu2026returncritic,cai2026vi}. Applying these ideas, including model merging and modular knowledge fusion~\citep{du2026didi,du2025graftllm,du2024pcb,du2024evolver,wang2026pare,wang2026beat,wang2026vdot}, to generation requires image-specific feedback and validation. Comparisons should use the same teacher and training budget, assess retention across domains, and check that fewer solver steps preserve diversity and agreement between coarse and refined images.

\paragraph{Reliable external conditioning.}
Retrieval-order alignment and question-guided knowledge re-scoring motivate testing the reliability of retrieved conditioning~\citep{zhang2026stable,zhang2024question}. Studies of modality preference and instruction anchoring~\citep{zhang2026evaluating,zhang2026instruction}, together with multimodal hallucination mitigation~\citep{zhang2026devils,zhang2026mitigating}, suggest evaluating guidance under conflicting or noisy evidence. For a retrieval-conditioned CDC-FM variant, experiments could hold content fixed for order tests and vary relevance or modality in separate tests, measuring semantic agreement and diversity. These studies concern language or multimodal reasoning; their mechanisms require validation for continuous image generation.

\paragraph{Multimodal extensions and evaluation.}
Structured tactile representations and missing-modality training motivate sparse sensor conditioning and missing-sensor robustness~\citep{WU2026114345,chi2022missing}. Wasserstein-barycenter representations for multimodal alignment and image restoration suggest exploring richer conditional anchors and couplings~\citep{tang2026barybind,tang2026learning,wang2026solving,liu2026palm,shen2026fine,li2026toward,yu2026elsa3d,fu2025learning,yu2025core3d}. Beyond the present static-image task, EgoProx's egocentric 3D proximity tasks and IPIBench's continuous-video interactions motivate evaluation of spatial consistency and response timing in embodied or streaming variants~\citep{Li_2026_CVPR,li2026ipibench}. Such extensions require task-specific conditioning, outputs, and evaluation protocols, with streaming memory as one direction to study~\citep{chi2026gopagen,chi2026enhancing}; these benchmarks do not directly measure CDC-FM's current image-generation quality. Cross-sensor generalization and cross-stage approximation errors remain important tests.

\section{Use of Large Language Models (LLMs)} \label{appendix:llm_use}

We acknowledge the use of a Generative AI tool for assistance with language editing and refinement during the preparation of this manuscript. The tool was not used for any substantive aspect of the research, such as ideation, data interpretation, or the formulation of conclusions. Full responsibility for all content rests with the authors.

\section{Limitations and Discussion} \label{sec:limitations}

We discuss four limitations of the present work.

\textbf{(i) First-stage prior is restricted to a Gaussian.}
The framework currently assumes a standard Gaussian prior at stage $k{=}1$ (Eq.~\ref{eq:indepednet_coupling}). This restricts the flexibility of the initial distribution and decouples it from the conditional dependent coupling construction at later stages. Generalizing the first-stage prior to a learnable or data-dependent distribution is a natural direction for future work; we expect this to be orthogonal to the coupling construction at $k{\ge}2$ but to require revisiting the transport-cost analysis in Theorem~\ref{theorem:lowerTransportCost}.

\textbf{(ii) Theoretical guarantees are conditional on three assumptions.}
Theorems~\ref{theorem:lowerTransportCost} and~\ref{theorem:lowerInfer} hold under (G.3) the noise-decay schedule of \eqref{eq:sigma_pyramid}, (G.4) per-evaluation cost linear in ambient dimensionality, and (G.5) NFE proportional to transportation cost. We empirically validate G.4 and G.5 in Appendix~\S\ref{appendix:efficiency_validation}, but the constants involved are solver- and hardware-dependent; on different solvers, batch sizes, or hardware, the wall-clock benefits may shift, although the asymptotic ordering $T_A>T_B$ is preserved.

\textbf{(iii) SDE instantiation is not empirically validated.}
We instantiate CDC under the deterministic Flow-Matching variant (CDC-FM) and validate it experimentally; the framework is also formally compatible with the SDE/Diffusion instantiation of stochastic interpolants (Appendix~\S\ref{appendix:SI}), but we leave its empirical study to future work.

\textbf{(iv) Resolution path is empirically chosen.}
Our experiments use the $32 \to 64 \to 128 \to 256$ hierarchy with $K{=}4$ for $256^2$ generation. As shown in Appendix~\S\ref{appendix:stage_path_nfe}, the framework remains viable for $K{=}3$ and $K{=}2$ with moderate per-stage ratios, but very large per-stage ratios degrade FID. We have not yet derived a principled procedure for selecting $K$ and the path from data; in this work the choice was made empirically.

\section*{Impact Statement}

This paper presents a novel framework for efficient, high-fidelity image generation operating directly in pixel space. Our goal is to advance the field of generative models by mitigating the computational trade-offs inherent in existing diffusion and flow-based methods.

There are two primary societal consequences of our work. First, by improving the efficiency of high-resolution generation, our method has the potential to democratize access to powerful generative tools, reducing the energy consumption and hardware barriers associated with training and inference. Second, like other high-fidelity generative models, our method carries the risk of misuse for creating deceptive content (e.g., Deepfakes) or propagating biases present in the training data (e.g., ImageNet). We encourage the research community to continue developing robust detection methods and fairness-aware training protocols to mitigate these risks.

\clearpage
\section*{NeurIPS Paper Checklist}

\begin{enumerate}

\item {\bf Claims}
    \item[] Question: Do the main claims made in the abstract and introduction accurately reflect the paper's contributions and scope?
    \item[] Answer: \answerYes{}
    \item[] Justification: The abstract and Section~\ref{sec:intro} state our contributions in a constrained form: a unified multi-stage Flow-Matching framework with conditional dependent coupling that achieves competitive FID/IS at substantially lower wall-clock cost in the low-NFE regime, with theoretical guarantees holding under explicit assumptions formalized in Theorems~\ref{theorem:lowerTransportCost} and~\ref{theorem:lowerInfer} and discussed in Section~\ref{sec:limitations}.

\item {\bf Limitations}
    \item[] Question: Does the paper discuss the limitations of the work performed by the authors?
    \item[] Answer: \answerYes{}
    \item[] Justification: A dedicated Section~\ref{sec:limitations} discusses four limitations: (i) the first-stage Gaussian prior, (ii) the conditionality of Theorems~\ref{theorem:lowerTransportCost}--\ref{theorem:lowerInfer} on Assumptions G.3/G.4/G.5 and their hardware/solver dependence, (iii) the SDE instantiation is not empirically validated, and (iv) the resolution path is chosen empirically.

\item {\bf Theory assumptions and proofs}
    \item[] Question: For each theoretical result, does the paper provide the full set of assumptions and a complete (and correct) proof?
    \item[] Answer: \answerYes{}
    \item[] Justification: The two theoretical results, Theorem~\ref{theorem:lowerTransportCost} (lower transport cost) and Theorem~\ref{theorem:lowerInfer} (lower inference time), each carry an explicit \emph{(Under Assumption ...)} clause that points to the noise-decay schedule \eqref{eq:sigma_pyramid} (Assumption G.3), the per-evaluation-cost linearity (Assumption~\ref{ass:perevalcost}, G.4), and the NFE--transport-cost proportionality (Assumption~\ref{ass:nfe_cost}, G.5). Complete proofs are provided in Appendices~\S\ref{pf:transportcost} and~\S\ref{pf:fasterInfer}, with empirical validation of G.4 and G.5 in Appendix~\S\ref{appendix:efficiency_validation}.

    \item {\bf Experimental result reproducibility}
    \item[] Question: Does the paper fully disclose all the information needed to reproduce the main experimental results of the paper to the extent that it affects the main claims and/or conclusions of the paper (regardless of whether the code and data are provided or not)?
    \item[] Answer: \answerYes{}
    \item[] Justification: The full experimental pipeline is described in Section~\ref{sec:multistage}--\ref{sec:cfg}; complete training/inference pseudocode is given in Appendices~\S\ref{ap:alg_uncondition} and~\S\ref{ap:cfg}; the full hyperparameter table (architecture, scheduler, optimizer, training, inference) is provided in Appendix Table~\ref{tab:implement} (see Appendix~\S\ref{appendix:efficiency_validation} and~\S\ref{appendix:imagenet512} for additional protocol details). All experiments use the publicly available ImageNet-1k dataset.

\item {\bf Open access to data and code}
    \item[] Question: Does the paper provide open access to the data and code, with sufficient instructions to faithfully reproduce the main experimental results, as described in supplemental material?
    \item[] Answer: \answerNo{}
    \item[] Justification: The training data, ImageNet-1k, is publicly available. The project repository is public, but it currently contains no training, inference, or evaluation implementation; these scripts are not yet released. Reproducibility from the paper alone is supported by the complete hyperparameter table (Appendix Table~\ref{tab:implement}) and the training/inference pseudocode in Appendices~\S\ref{ap:alg_uncondition} and~\S\ref{ap:cfg}.

\item {\bf Experimental setting/details}
    \item[] Question: Does the paper specify all the training and test details (e.g., data splits, hyperparameters, how they were chosen, type of optimizer) necessary to understand the results?
    \item[] Answer: \answerYes{}
    \item[] Justification: The complete training and inference settings, including data splits (standard ImageNet-1k), optimizer (with learning rate, weight decay, batch size), training schedule (epochs/iterations), CFG dropout probability, and per-stage hyperparameters, are reported in Appendix~\S\ref{appendix:imp_detail} (Implement Details) and Table~\ref{tab:implement}.

\item {\bf Experiment statistical significance}
    \item[] Question: Does the paper report error bars suitably and correctly defined or other appropriate information about the statistical significance of the experiments?
    \item[] Answer: \answerNo{}
    \item[] Justification: We follow standard practice in the image generation literature: FID and IS are computed once on 50k generated samples per configuration. The 50k-sample protocol is the de facto evaluation standard for ImageNet-1k generation~\cite{dhariwal2021adm,peebles2023dit,ma2024sit} and yields intrinsically low Monte Carlo variance; repeating each entry across multiple training seeds would be computationally prohibitive given that a single training run requires 32$\times$A100 for $\sim$76 hours (Appendix~\S\ref{appendix:imp_detail}).

\item {\bf Experiments compute resources}
    \item[] Question: For each experiment, does the paper provide sufficient information on the computer resources (type of compute workers, memory, time of execution) needed to reproduce the experiments?
    \item[] Answer: \answerYes{}
    \item[] Justification: Appendix~\S\ref{appendix:imp_detail} reports compute resources in detail: training uses 32$\times$NVIDIA A100 (80GB) with global batch size $256$ for $\sim$200 epochs ($\sim$1M iterations) totaling $\sim$76 hours, with $\sim$43GB per-GPU memory; inference uses a single A100 at batch size 32 to generate 50k samples for evaluation. Per-evaluation latency at each stage (which underpins our wall-clock claims) is reported separately in Appendix Table~\ref{tab:per_eval_cost}.

\item {\bf Code of ethics}
    \item[] Question: Does the research conducted in the paper conform, in every respect, with the NeurIPS Code of Ethics \url{https://neurips.cc/public/EthicsGuidelines}?
    \item[] Answer: \answerYes{}
    \item[] Justification: The research conforms with the NeurIPS Code of Ethics. The work uses only the publicly available ImageNet-1k dataset under its standard research-use license, involves no human subjects, and contains no personal or sensitive data.

\item {\bf Broader impacts}
    \item[] Question: Does the paper discuss both potential positive societal impacts and negative societal impacts of the work performed?
    \item[] Answer: \answerYes{}
    \item[] Justification: The Impact Statement in the appendix discusses both the positive impact (democratizing access to high-resolution generation by lowering compute and energy barriers) and the negative risks (potential misuse for Deepfakes/disinformation, propagation of dataset biases such as those in ImageNet), and suggests the development of detection methods and fairness-aware protocols as mitigation directions.

\item {\bf Safeguards}
    \item[] Question: Does the paper describe safeguards that have been put in place for responsible release of data or models that have a high risk for misuse (e.g., pre-trained language models, image generators, or scraped datasets)?
    \item[] Answer: \answerNA{}
    \item[] Justification: This work does not release a high-risk pre-trained generative checkpoint or scraped dataset. The paper describes a method, and the planned post-acceptance release covers training/inference code only; therefore the misuse-risk profile does not exceed that of the existing literature on ImageNet-class-conditional generation, and no special safeguards are required.

\item {\bf Licenses for existing assets}
    \item[] Question: Are the creators or original owners of assets (e.g., code, data, models), used in the paper, properly credited and are the license and terms of use explicitly mentioned and properly respected?
    \item[] Answer: \answerYes{}
    \item[] Justification: All existing assets are properly credited. The dataset (ImageNet-1k~\cite{krizhevsky2012imagenet}) is used under its standard research-use terms. The architectural backbones we build on, including DiT~\cite{peebles2023dit} and SiT~\cite{ma2024sit}, are released under permissive open-source licenses (MIT/Apache-2.0 in their public repositories) and are cited at every point of use; we use them only for non-commercial research, consistent with their terms.

\item {\bf New assets}
    \item[] Question: Are new assets introduced in the paper well documented and is the documentation provided alongside the assets?
    \item[] Answer: \answerNA{}
    \item[] Justification: The paper does not release new datasets or pre-trained model checkpoints. The project repository does not yet contain implementation assets (cf.\ Question on Open access).

\item {\bf Crowdsourcing and research with human subjects}
    \item[] Question: For crowdsourcing experiments and research with human subjects, does the paper include the full text of instructions given to participants and screenshots, if applicable, as well as details about compensation (if any)?
    \item[] Answer: \answerNA{}
    \item[] Justification: The paper does not involve crowdsourcing nor research with human subjects.

\item {\bf Institutional review board (IRB) approvals or equivalent for research with human subjects}
    \item[] Question: Does the paper describe potential risks incurred by study participants, whether such risks were disclosed to the subjects, and whether Institutional Review Board (IRB) approvals (or an equivalent approval/review based on the requirements of your country or institution) were obtained?
    \item[] Answer: \answerNA{}
    \item[] Justification: The paper does not involve human subjects, so no IRB approval is required.

\item {\bf Declaration of LLM usage}
    \item[] Question: Does the paper describe the usage of LLMs if it is an important, original, or non-standard component of the core methods in this research? Note that if the LLM is used only for writing, editing, or formatting purposes and does \emph{not} impact the core methodology, scientific rigor, or originality of the research, declaration is not required.
    \item[] Answer: \answerNA{}
    \item[] Justification: LLMs are not part of the core methodology and were not used for ideation, data interpretation, or formulation of conclusions. We voluntarily disclose in Appendix~\S\ref{appendix:llm_use} that a generative AI tool was used solely for language editing and refinement during manuscript preparation, with full responsibility for all content resting with the authors.

\end{enumerate}

\end{document}